\documentclass[12pt]{article}
\usepackage[left=1in,top=1in,right=1in,nohead,bottom=1in]{geometry}
\usepackage{graphicx}
\usepackage{amsmath, amssymb, enumerate}
\usepackage[applemac]{inputenc}
\usepackage[T1]{fontenc}
\usepackage{amsfonts}
\usepackage{tikz}
\usetikzlibrary{shapes.geometric, arrows}
\usepackage{subcaption}
\usepackage{float}
\usepackage{parskip}
\usepackage{amsthm}
\usepackage{natbib}
\newtheorem{theorem}{Theorem}
\newtheorem*{definition}{Definition}
\newtheorem{assumption}{Assumption}[]
\newtheorem{lemma}{Lemma}[]
\newcommand{\ds}{\displaystyle}
\newcommand{\E}{\text{E}}
\newcommand{\Var}{\text{Var}}

\newcommand{\ts}{\textsuperscript}
\newcommand{\htaun}{h_{\tau_n}}
\newcommand{\pitaun}{\pi_{\tau_n}}
\newcommand{\hne}{h_{n_e}}
\newcommand{\pine}{\pi_{n_e}}
\newcommand{\hgn}{h_{q(n)}}

\newcommand{\pign}{\pi_{q(n)}}

\title{To update or not to update? Delayed Nonparametric Bandits with Randomized Allocation} 

 \author{Sakshi Arya and Yuhong Yang\\
 \small{School of Statistics, University of Minnesota}}
  \date{}  

\begin{document}
\maketitle







\abstract{Delayed rewards problem in contextual bandits has been of interest in various practical settings.
 We study randomized allocation strategies and provide an understanding on how the exploration-exploitation tradeoff  is affected by delays in observing the rewards. In randomized strategies, the extent of exploration-exploitation is controlled by a user-determined exploration probability sequence. In the presence of delayed rewards, one may choose between using the original exploration sequence that updates at every time point or update the sequence only when a new reward is observed, leading to two competing strategies. In this work, we show that while both strategies may lead to strong consistency in allocation, the property holds for a wider scope of situations for the latter. However, for finite sample performance, we illustrate that both strategies have their own advantages and disadvantages, depending on the severity of the delay and underlying reward generating mechanisms.} 

\section{Introduction}\label{Intro}
Contextual bandits provide a natural framework to model a lot of practical sequential decision making problems in various fields.
 \cite{woodroofe1979one} started studying multi-armed bandit problems with side information in a parametric framework, and \cite{yang2002randomized} initiated an investigation from a nonparametric perspective. See \cite{lai2001sequential};\cite{bartroff2008modern} for reviews on general sequential problems and \cite{bubeck2012regret} for bandits exclusively.
 In recent years,  bandit problems have gained popularity and have been studied extensively under different names, such as contextual bandits, multi-armed bandits with covariates (MABC), associative bandit problems and  multi-armed bandits with side information. For example, when treating patients of a disease, the doctor needs to decide which treatment amongst several competing treatments would be the best for the current patient, given the patient's covariate information and data available from previous patients. Most of the bandit algorithms assume instantaneous observance of rewards, but in most practical situations, rewards are only obtained at some delayed time. For example, it is often the case that  several other patients have to be treated before the outcome for the current patient is observed. One way to tackle this problem is to adopt black-box procedures incorporating delayed rewards using the already existing no-delay policies in the stochastic bandits setting. However, we present a case of why it is important to study delays more carefully for contextual bandit strategies based on the context of the problem, rather than always using the already existing no delay bandit strategies in black-box procedures to incorporating delayed rewards.   Delays in observing the rewards could affect the performance of bandit algorithms in different ways, depending on the nature of underlying data generating mechanisms and severity of the delays. Thus, it is important to balance the exploration-exploitation trade-off taking these aspects into consideration, in order to utilize most of the available information. We propose two different $\epsilon$-greedy like strategies incorporating delayed reward, which differ in how the exploration probability gets updated with the available information. We illustrate that both strategies can be advantageous in different situations, based on the complexity of the underlying data generating mechanism and the severity of the delays.

\section{Setup and related literature}\label{setup_relatedLitSec}
The setup of stochastic contextual bandits is as follows. Suppose there are $\ell > 1$ competing arms. The covariates are assumed to be random variables generated according to an unknown underlying probability distribution $P_X$ supported in $[0,1]^d$. A bandit strategy or policy is a random function from $[0,1]^d$ to $\{1,2,\hdots,\ell\}$ that decides which arm gets pulled for a given covariate. At time $j\geq 1$, let $I_j$ be the arm allocation made by the bandit strategy based on previous information and present context $X_j$.  We denote $Y_{i,j}$ to be the reward obtained for arm $i = I_j$. Let $f_i(x)$ denote the mean reward for the $i$\ts{th} arm with covariate $x$. We adopt a regression perspective to model the relationship between covariates and rewards,
\begin{align*}
Y_{i,j} = f_i(X_j) + \epsilon_{j}
\end{align*}
where $\epsilon_{j}$'s are independent errors with $\E(\epsilon_{j}) = 0$ and $\Var(\epsilon_{j}) < \infty$ for $j\geq 1$.

Now, the problem can be viewed as one of estimating the mean reward functions $f_i(x)$ for $i \in \{1,\hdots,\ell\}$ and allocating arms based on the estimators $\hat{f}_{i}$. Both parametric and non-parametric approaches for estimating $f_i$ have been well studied, see \cite{tewari2017ads}, \cite{lattimore2018bandit} for reference. In this work we follow a nonparametric approach with delayed rewards as in \cite{arya2020randomized} adopting modeling techiques similar to the earlier work of \cite{yang2002randomized}, \cite{qian2016kernel,qian2016randomized}. 

In our setup, the rewards can be obtained at some delayed time, which we denote by $\{t_j \in \mathbb{R}^+, j \geq 1\}$. The delay in the reward for pulling arm $I_j$ is given by the random variable, $d_j:= t_j - j$. We assume that $\{d_j: d_j \geq 0, j\geq 1\}$ is a sequence of independent random variables.
Let the number of rewards obtained at time $n$ be denoted by $\tau_n = \sum_{j=1}^n I(t_j \leq n)$, also a random variable.

We devise two sequential allocation strategies $\eta_1$ and $\eta_2$ in Section \ref{algorithm1}, incorporating delayed rewards, such that  they choose arms sequentially based on previous observations and present covariates. As a measure of performance of each of the strategies, we consider the following ratio,
\begin{align}
 R_n(\cdot) = \dfrac{\sum_{j=1}^n f_{I_j}(X_j)}{\sum_{j=1}^n f^*(X_j)}, \label{relative_reward}
 \end{align} 
 where $(\cdot)$ is used to denote the strategy being considered. Here, $f^*(x) = \max_{1\leq i \leq \ell} f_i(x)$ is the theoretical best mean reward functional value at $x$, and $i^*(x)$ is the corresponding arm. Then, we establish strong consistency for both strategies for the histogram method in Section \ref{histogram}, that is, we  show that $R_n(\eta_1) \rightarrow 1$ and $R_n(\eta_2) \rightarrow 1$ with probability 1, as $n\rightarrow \infty$.  In addition, from a finite-sample performance perspective, we compare the two allocation strategies and illustrate how both can be advantageous in different situations in Sections \ref{compareSPlvsThis} and \ref{sec:simulation}.

In the stochastic setting, delayed rewards have been studied previously by \cite{Dudik:2011:EOL:3020548.3020569}, \cite{joulani2013online} where the former considers constant known delay for contextual bandits while the latter provides a more systemic study of online learning problems with random delayed rewards (without covariates). \cite{joulani2013online} develop meta-algorithms which in a black-box fashion use algorithms developed for the non-delayed case into the ones that can handle delays in a feedback loop. Then, \cite{mandel2015queue} devise a method that guarantees good black-box algorithms when leveraging a prior dataset and incorporating heuristics to help improve empirical performance of the algorithms. \cite{desautels2014parallelizing} use Gaussian process bandits and develop algorithms for parallelizing exploration-exploitation trade-offs. Motivated by delayed conversions in advertising, \cite{vernade2017stochastic,vernade2018contextual} consider potentially infinite stochastic delays, where the latter deals with the contextual case with a linear regression model and does not assume prior knowledge of delay distribution unlike the former. Recently, \cite{zhou2019learning} design a  delay-adaptive algorithm for generalized linear contextual bandits using UCB-style exploration. \cite{arya2020randomized} consider potentially infinite delays in nonparametric bandits and provide strong consistency results for a proposed algorithm. Other works include \cite{eick1988gittins}, \cite{cella2019stochastic} where the former considers Gittins procedures for bandits with delayed rewards, while the latter is motivated by applications in music streaming. Apart from the stochastic setting,  \cite{cesa2016delay,li2019bandit,thune2019nonstochastic,zimmert2019optimal} study delayed rewards in the adversarial setting, while \cite{pike2017bandits,pike2018bandits,cesa2018nonstochastic} study the delayed anonymous composite feedback setting. 

\section{The proposed strategies}\label{algorithm1}
Define $Z^{n,i}$ to be the set of observations for arm $i$ whose rewards have been observed by time $n-1$, that is, $Z^{n,i}:= \{(X_j,Y_{i,j}): 1\leq t_j \leq n-1 \ \text{and} \ I_j = i\}$. Let $\hat{f}_{i,n}$ denote the regression estimator of $f_i$ based on the data $Z^{n,i}$. Let $\{\pi_j, j \geq 1 \}$ be a sequence of positive numbers in  $[0,1]$ decreasing to zero, such that $(\ell - 1) \pi_j < 1$ for all $j \geq 1$. We propose two strategies $\eta_1$ and $\eta_2$ with a subtle difference in the arm selection      step but same structure of the algorithm.
\subsection{Algorithms} \label{algorithm}
\begin{enumerate}[Step 1.]
  \item \textbf{Initialize.} Allocate each arm once, $I_1 = 1, I_2 = 2, \hdots, I_\ell = \ell$. Since the rewards are not immediately obtained for each of these $\ell$ arms, we continue these forced allocations until we have at least one reward observed for each arm. Suppose, that happens at time $m_0$.
  \item \textbf{Estimate the individual functions $\boldsymbol{f_i}$.} For $n = m_0 + 1$, based on $Z^{n,i}$, estimate $f_i$ by $\hat{f}_{i,n}$ for $1 \leq i \leq \ell$ using the chosen regression procedure.
  \item \textbf{Estimate the best arm.} For $X_{n}$, let 
  $\hat{i}_{n}(X_{n}) = \arg\max_{1\leq i \leq \ell} \hat{f}_{i,n}(X_{n})$.
   \item \label{randomization_step}  \textbf{Select and pull.} Recall, $\tau_n = \sum_{j=1}^n I(t_j \leq n)$ is the number of rewards observed by time $n$.
   \begin{enumerate}
     \item    Strategy $\eta_1$:      $I_n = \begin{cases}
     \hat{i}_n, & \ \text{with probability}\  1-(\ell -1)\pi_{n}\\
     i, & \ \text{with probability}\ \pi_{n}, \ i \neq \hat{i}_n, \ 1\leq i \leq \ell.
     \end{cases}$
     \item Strategy $\eta_2$: $I_n = \begin{cases}
     \hat{i}_n, & \ \text{with probability} 1-(\ell -1)\pi_{\tau_n}\\
     i, & \ \text{with probability}\ \pi_{\tau_n}, \ i \neq \hat{i}_n,\  1\leq i \leq \ell.
     \end{cases}$
   \end{enumerate}
\item \textbf{Update the estimates.}
  \begin{enumerate}[Step 5a.]
    \item If a reward is obtained at the $n$\ts{th} time (could be one or more rewards corresponding to one or more arms $I_j, 1\leq j \leq n$), update the function estimates of $f_i$ for the respective arm (or arms) for which the reward (or rewards) is obtained at $n\ts{th}$ time.
    \item If no reward is obtained at the $n$\ts{th} time, use the previous function estimators, i.e. $\hat{f}_{i,n+1} = \hat{f}_{i,n} \ \forall \ i \in \{1,\hdots,\ell\}$. 
  \end{enumerate}
  \item \textbf{Repeat.} Repeat steps 3-5 when the next covariate $X_{n+1}$ surfaces and so on.
\end{enumerate}
In the algorithms above, Step 1 initializes the allocations by pulling each arm alternatively until we observe at least one reward for each arm. Step 2 estimates the mean reward function for each arm. This could be done using several regression methods, we use kernel regression and histogram method in this work. Steps 3 and 4 enforce an $\epsilon$-greedy type of randomization scheme which prefers the best performing arm so far with some probability and explores with the remaining. The preference is determined by user determined sequence of exploration probability $\{\pi_n, n \geq 1\}$, which for strategy $\eta_2$ only gets updated when a new reward is observed, that is, $\pi_{\tau_n}$. While for strategy $\eta_1$, it is updated at every time point irrespective of a reward being observed or not, that is, $\pi_n$. Hence, the two strategies differ in the extent of exploration and exploitation that is allowed over time.  Finally, in Step 5, the mean reward function estimators are updated if new rewards are observed or they remain the same if no new rewards are observed. For notational convenience, we use $\{\cdot\}$ to denote a user-determined sequence, such as $\{\pi_n\}$, when we only want to refer to the original sequence selected by the user, without distinguishing between when it gets updated.

\section{Consistency of the proposed strategies} \label{consistency_section}
Let $\mathcal{A}_n := \{j: t_j \leq n\}$, denote the time points corresponding to the rewards observed by time $n$.

\begin{assumption} \label{Ass1}
   The regression procedure is strongly consistent in $L_\infty$ norm for all individual mean functions $f_i$ under the proposed allocation scheme. That is, $||\hat{f}_{i,n} - f_i||_\infty \overset{\text{a.s.}}{\rightarrow} 0$ as $n \rightarrow \infty$ for each $1\leq i \leq \ell$, where $\hat{f}_{i,n}$ is the estimator based on all previously observed rewards. 
   \end{assumption}
   Note that, due to the presence of delays, the mean reward function estimators $\hat{f}_{i,n}$ are only updated at the time points where a new reward is observed. Next, we make a mild assumption on the mean reward functions.

\begin{assumption} \label{Ass2}
  The mean reward functions are continuous and $f_i(x) \geq 0$ such that, 
  \begin{align*}
  A = \ds\sup_{1\leq i\leq \ell} \sup_{x \in [0,1]^d} (f^*(x) - f_i(x)) < \infty \ \text{and}\ \E (f^*(X_1)) > 0.
  \end{align*}
\end{assumption}

\begin{assumption} \label{assump_delay} Let the partial sums of delay distributions satisfy,
$\E(\tau_n) = \Omega(q(n))$ \footnote{$f(n) = \Omega{(g(n))}$ if for some positive constant $c$, $f(n) \geq cg(n)$ when $n$ is large enough}, where $q(n)$ is a sequence that acts as a lower bound to the expected number of observed rewards by time $n$, and $q(n) \rightarrow \infty$ as $n \rightarrow \infty$.
\end{assumption}

\begin{theorem}\label{Theorem 1}
Under Assumptions \ref{Ass1}, \ref{Ass2} and \ref{assump_delay}, the allocation rules $\eta_1$ and $\eta_2$ are strongly consistent as $n\rightarrow \infty$, i.e., $R_n(\eta_1) \rightarrow 1$ and $R_n(\eta_2) \rightarrow 1$ with probability 1, as $n \rightarrow \infty$.
\end{theorem}
\begin{proof}
Note that consistency holds only when the sequence $\{\pi_n, n\geq 1\}$ is chosen such that $\{\pi_n\} \rightarrow 0$ as $n\rightarrow \infty$.
The proof is very similar to the proof in \cite{arya2020randomized} with minor changes for strategy $\eta_2$ which are included in Appendix \ref{Proof_consistency_eta2}.
\end{proof}

Note that Assumption \ref{Ass1}, seemingly natural, is a strong assumption and it requires additional work to verify it for a particular regression setting. We verify this assumption for the histogram method in Section \ref{histogram} and for the kernel method in Section \ref{sec:KernelRegression}. 

\subsection{Histogram method} \label{histogram}
In this section, we consider the histogram method for the setting with delayed rewards.
We assume that the binwidth $h$ is chosen such that $1/h$ is an integer. At time $n$, partition $[0,1]^d$ into $M = (1/h_{\tau_n})^d$ hyper-cubes with binwidth $h_{\tau_n}$, where $\tau_n$ is the number of observed rewards by time $n$. 
For some $x \in [0,1]^d$ such that it falls in a hypercube $B(x)$, let 
$\bar{J}_i(x) = \{j: X_j \in B(x), t_j \leq n, I_j = i \}$ and $\bar{N}_i(x)$ be the size of $\bar{J}_i(x)$. Then the histogram estimate for $f_i(x)$ is defined as,
\begin{align}
\hat{f}_{i,n}(x) = \frac{1}{\bar{N}_i(x)} \sum_{j \in \bar{J}_i(x)} \label{histogramEst}
Y_j.
\end{align}
For the estimator to behave well, a proper choice of the binwidth, $\{h_n\}$ is necessary. 
Note that, we only update $h_n$ to $h_{n+1}$ when a new reward is observed, hence we denote it as $h_{\tau_n}$.  
For notational convenience, when the analysis is focused on a single arm, $i$ is dropped from the subscript of $\hat{f}$, $\bar{N}$ and $\bar{J}$. Next, using the histogram method for estimation, we prove that strong consistency holds for both strategies $\eta_1$ and $\eta_2$ in Section \ref{algorithm}.

 As already discussed, we only need to verify that Assumption \ref{Ass1} holds for histogram method. Along with Assumptions \ref{Ass2} and \ref{assump_delay}, we make the following assumptions.

\begin{assumption}
\label{ass_design_distribution} The design distribution $P_X$ is dominated by the Lebesgue measure with a density $p(x)$ uniformly bounded above and away from 0 on $[0,1]^d$; that is, $p(x)$ satisfies $\underbar{c} \leq p(x) \leq \bar{c}$ for some positive constants $\underbar{c} < \bar{c}$.
\end{assumption}
This assumption is needed to make sure that all regions in the covariate space are observed with positive probability, in order to ensure good estimation in all regions.

\begin{assumption}
\label{ass_bernstein_errors} The errors satisfy a moment condition that there exists positive constants $v$ and $c$ such that, for all integers $m \geq 2$, the extended Bernstein condition (\cite{birge1998minimum,qian2016kernel}) is satisfied, that is, 
\begin{align*}
\E|\epsilon_{j}|^m \leq \frac{m!}{2} v^2 c^{m-2}.
\end{align*}
\end{assumption}
This condition on the errors holds in a lot of settings, for example, normal distribution and bounded errors meet this requirement, thus making it useful in a wide range of applications. \\
The next two assumptions are made on the nature of the delays in observing rewards, so that we could ensure that delays are not being confounded by other factors and we observe a minimum number of rewards with time, so as to ensure proper and effective learning.
\begin{assumption}\label{ass_delay_independence} The delays, $\{d_j, j\geq 1\}$, are independent of each other, the choice of arms and also of the covariates.
\end{assumption}

Along with these assumptions, we define the modulus of continuity that is used in the following results.
\begin{definition} \label{mod_of_continuity}
The modulus of continuity, $w(h;f)$, is defined by,
$
w(h;f) = \sup\{|f(x_1) - f(x_2)|: |x_{1k} - x_{2k}| \leq h \ \text{for all}\ 1 \leq k \leq d\},\ \text{for}\ x_1,x_2 \in [0,1]^d.$
\end{definition}
\begin{lemma}[An inequality for Bernoulli trials.]\label{A.2}
For $1\leq j \leq n$, let $\tilde{W}_j$ be Bernoulli random variables, which are not necessarily independent. Assume that the conditional probability of success for $\tilde{W}_j$ given the previous observations is lower bounded by $\beta_j$, that is,
\begin{align*}
P(\tilde{W}_j = 1|\tilde{W}_i, 1 \leq i \leq j-1) \geq \beta_j \ \text{a.s.},
\end{align*}
for all $1\leq j\leq n$. 
Applying the extended Bernstein's inequality as described in \cite{qian2016kernel}, we have
\begin{align}
P\left(\sum_{j=1}^n \tilde{W}_j \leq \left(\sum_{j=1}^n \beta_j\right)/2 \right) \leq \exp\left(- \dfrac{3\sum_{j=1}^n \beta_j}{28} \right).  \label{binomial_inequality}
\end{align}
\end{lemma}

\begin{lemma} \label{lemma_theorem}
Let $\epsilon > 0$ be given. Suppose that $h$ is small enough such that $w(h;f) < \epsilon$. Then the histogram estimator $\hat{f}_n$ satisfies,
\begin{align}
\text{P}^{\eta_1}_{\mathcal{A}_n,X^n}(||\hat{f}_n - f||_\infty \geq \epsilon) &\leq M \exp\left(-\dfrac{3\pi_{n} \min_{1\leq b \leq M} N_b}{28} \right) \nonumber \\
&\quad \quad +2M \exp\left(- \dfrac{\min_{1\leq b \leq M} N_b \pi_n^2 (\epsilon - w(h_{\tau_n};f))^2}{8(v^2 + c(\pi_{n}/2)(\epsilon-w(h_{\tau_n};f)))} \right), \label{lemma_result_eta1}\\
  \text{P}^{\eta_2}_{\mathcal{A}_n,X^n}(||\hat{f}_n - f||_\infty \geq \epsilon) &\leq M \exp\left(-\dfrac{3\pi_{\tau_n} \min_{1\leq b \leq M} N_b}{28} \right) \nonumber \\
  & \quad \quad + 2M \exp\left(- \dfrac{\min_{1\leq b \leq M} N_b \pi_{\tau_n}^2 (\epsilon - w(h_{\tau_n};f))^2}{8(v^2 + c(\pi_{\tau_n}/2)(\epsilon-w(h_{\tau_n};f)))} \right), \label{Lemma_result_eta2}
\end{align}
where $P_{\mathcal{A}_n,X^n}$ denotes conditional probability given design points $X^n = (X_1,\hdots,X_n)$ and $\mathcal{A}_n = \{j: t_j\leq n\}$. Here, $N_b$ is the number of design points for which the rewards have been observed by time $n$ such that they fall in the $b$th small cube of the partition of the unit cube at time $n$.
\end{lemma}
\begin{proof}
The proof of Lemma \ref{lemma_theorem} is similar to \cite{arya2020randomized} so we skip it here. For strategy $\eta_1$, it is easy to see that a similar lemma with $h_n$ replaced by $h_{\tau_n}$ could be derived. For strategy $\eta_2$, $\pi_n$ is replaced by $\pi_{\tau_n}$ and $h_n$ replaced by $h_{\tau_n}$. This is because the result is a conditional probability result, and given $\mathcal{A}_n$ and $X^n$, $\tau_n$ is a known quantity.
\end{proof}

\begin{theorem} \label{thm:theorem}
Suppose Assumptions \ref{Ass2}-\ref{ass_delay_independence} are satisfied. 
\begin{enumerate}
\item[a)]If $\{h_n\}$ and $\{\pi_n\}$ are chosen to satisfy,
\begin{align}
\dfrac{h_{q(n)}^2 \pi_{n}^2 q(n)}{\log{n}} \rightarrow \infty \ \text{as} \ n \rightarrow \infty, \label{conditionforThm1_eta1}
\end{align}
then the histogram estimator in \eqref{histogramEst} is strongly consistent in the $L_\infty$ norm for strategy $\eta_1$, hence $\eta_1$ is strongly consistent.  
  \item[b)] If $\{h_n\}$ and $\{\pi_n\}$ are chosen to satisfy,
\begin{align}
\dfrac{h_{q(n)}^2 \pi_{q(n)}^2 q(n)}{\log{n}} \rightarrow \infty \ \text{as} \ n \rightarrow \infty, \label{conditionforThm1_eta2}
\end{align}
then the histogram estimator in \eqref{histogramEst} is strongly consistent in the $L_\infty$ norm for strategy $\eta_2$, hence $\eta_2$ is strongly consistent.
\end{enumerate}
\end{theorem}

\begin{proof}
The proofs for \text{a)} and \text{b)} are quite similar, so we prove \text{b)} here and consequently discuss \text{a)}.
Given $\mathcal{A}_n$, the indices corresponding to when rewards were obtained, we know that at time $n$, the histogram method partitions the unit cube into $M= (1/h_{\tau_n})^d$ small cubes. For each small cube $B_b, 1\leq b \leq M$, in the partition, let $N_b = \sum_{j=1}^n I(X_j \in B_b, t_j \leq n)$. Note that given $\mathcal{A}_n$, $P_{\mathcal{A}_n}(X_j \in B_b, t_j \leq n) = P_{\mathcal{A}_n}(X_j \in B_b) \geq \underbar{c}h_{\tau_n}^d$, thus using inequality \eqref{binomial_inequality} we have,
\begin{align}
  P_{\mathcal{A}_n}\left(N_b \leq \dfrac{\underbar{c} \htaun^d \tau_n}{2} \right) &\leq \exp \left(- \dfrac{3\underbar{c} \htaun^d \tau_n}{28} \right) \\
 \Rightarrow P_{\mathcal{A}_n}\left(\min_{1\leq b \leq M} N_b \leq \dfrac{\underbar{c} \htaun^d \tau_n}{2} \right) &\leq \exp \left(- \dfrac{3\underbar{c} \htaun^d \tau_n}{28} \right). \label{inequality_forabin_Nb_Thm1}
  \end{align} 
 Recall, $\tau_n = \sum_{j=1}^n I\{t_j \leq n\}$. First, we show that $\tau_n \overset{\text{a.s.}}{\rightarrow} \infty$ as $n \rightarrow \infty$ for both strategies, $\eta_1$ and $\eta_2$.
By Assumption \ref{assump_delay} and the inequality \eqref{binomial_inequality} in Lemma \ref{A.2} we have that for a large enough $n$, there exists a positive constant $a_1 > 0$ such that, $\E(\tau_n) \geq a_1 q(n)$, therefore,
\begin{align*}
P \left(\tau_n \leq \dfrac{a_1 q(n)}{2} \right) &\leq P\left(\tau_n \leq \dfrac{\E(\tau_n)}{2}\right)
 \leq \exp \left(-\dfrac{3\E(\tau_n)}{28} \right) \leq\exp \left(\dfrac{-3 a_1 q(n)}{28} \right).
\end{align*}
It is easy to see that the upper bound is summable in $n$ under the conditions \eqref{conditionforThm1_eta1} and \eqref{conditionforThm1_eta2}. By Borel-Cantelli lemma, this implies that event $\{\tau_n > a_1q(n)/2\}$ happens infinitely often, therefore $\tau_n \overset{\text{a.s.}}{\rightarrow} \infty$. Note that, by construction this implies that $h_{\tau_n} \overset{\text{a.s.}}{\rightarrow} 0$, and $\pi_{\tau_n} \overset{\text{a.s.}}{\rightarrow} 0$ as $n \rightarrow \infty$. Let $w(\htaun; f_i)$ be the modulus of continuity as in Definition \ref{mod_of_continuity}. Then, continuity of $f_i$ leads to the conclusion that $w(h_{\tau_n}; f_i) \overset{\text{a.s.}}{\rightarrow} 0$ as $n \rightarrow \infty$. 
 Thus, for any $\epsilon >0$, for large enough $n$, when $h_{\tau_n}$ is small enough, $\epsilon-w(h_{\tau_n};f_i) \geq \epsilon/2$, almost surely. Consider,
  \begin{align*}
  P_{\mathcal{A}_n} \left(|| \hat{f}_{i,n} - f_{i}||_\infty \geq \epsilon \right) &=  P_{\mathcal{A}_n} \left(|| \hat{f}_{i,n} - f_{i}||_\infty \geq \epsilon,\min_{1\leq b \leq M} N_b > \dfrac{\underbar{c} \htaun^d \tau_n}{2}  \right)\\
  &\quad \quad + P_{\mathcal{A}_n} \left(|| \hat{f}_{i,n} - f_{i}||_\infty \geq \epsilon,\min_{1\leq b \leq M} N_b \leq \dfrac{\underbar{c} \htaun^d \tau_n}{2}  \right)\\
  &\leq \E^{X^n} P_{\mathcal{A}_n,X^n} \left(|| \hat{f}_{i,n} - f_{i}||_\infty \geq \epsilon,\min_{1\leq b \leq M} N_b > \dfrac{\underbar{c} \htaun^d \tau_n}{2}  \right) \\
  &\quad \quad + P_{\mathcal{A}_n} \left(\min_{1\leq b \leq M} N_b \leq \dfrac{\underbar{c} \htaun^d \tau_n}{2}  \right),
  \end{align*}
  where we use law of iterated expectation in the first term and $\E^{X^n}$ denotes expectation with respect to $X^n$. From \eqref{Lemma_result_eta2}  and \eqref{inequality_forabin_Nb_Thm1}, we get that,
  \begin{align}
   P_{\mathcal{A}_n} \left(|| \hat{f}_{i,n} - f_{i}||_\infty \geq \epsilon \right) &\leq M \exp \left(-\dfrac{3\underbar{c}\pitaun\htaun^d \tau_n}{56} \right) + 2M \exp \left(-\dfrac{\underbar{c}\htaun^d \pitaun^2 \tau_n (\epsilon - w(Lh_{\tau_n};f_i))^2}{8(v^2 + c(\pitaun/2)\epsilon)} \right)\nonumber\\
   & \quad \quad + M \exp \left(-\dfrac{3\underbar{c} \htaun^d \tau_n}{28} \right). \label{ConditionOnAn}
  \end{align}
  Now consider,
  \begin{align}
  P(||\hat{f}_{i,n} - f_i||_\infty > \epsilon) &\leq P\left(||\hat{f}_{i,n} - f_i||_\infty \geq \epsilon, \tau_n > \dfrac{\E(\tau_n)}{2}\right) + P \left(\tau_n \leq \dfrac{\E(\tau_n)}{2} \right)\nonumber\\
  & \leq E^{\mathcal{A}_n}P_{\mathcal{A}_n}\left(||\hat{f}_{i,n} - f_i||_\infty \geq \epsilon, \tau_n > \dfrac{\E(\tau_n)}{2}\right) + P \left(\tau_n \leq \dfrac{\E(\tau_n)}{2} \right). \label{after_removing_conditioning}
  \end{align}
  Let $n_e = \lfloor \E(\tau_n)/2 \rfloor$. Then, by using condition \eqref{conditionforThm1_eta2} and \eqref{ConditionOnAn} in \eqref{after_removing_conditioning}, we have that, for large enough $n$,
  \begin{align}
  P(||\hat{f}_{i,n} - f_i||_\infty > \epsilon)
  &\leq M \exp \left(-\dfrac{3\underbar{c} \pine \hne^d n_e}{56} \right) + 2M \exp \left(-\dfrac{\underbar{c} \hne^d \pine^2 n_e (\epsilon - w(Lh_{n_e};f_i))^2}{8(v^2 + c(\pine/2)(\epsilon)} \right)\nonumber\\
  & \quad \quad  + M \exp \left(-\dfrac{3\underbar{c} \hne^d n_e}{28} \right) + \exp \left(- \dfrac{3n_e}{14} \right) \nonumber \\
 & \leq M \exp \left(-\dfrac{3\tilde{\underbar{c}} \pign \hgn^d q(n)}{112} \right) \\
 &\quad + 2M \exp \left(-\dfrac{\tilde{\underbar{c}} \hgn^d \pign^2 q(n) (\epsilon - w(Lh_{q(n)};f_i))^2}{16(v^2 + c(\pign/2)(\epsilon)} \right) \nonumber \\
 & \quad \quad + M \exp \left(-\dfrac{3\tilde{\underbar{c}} \hgn^d q(n)}{56} \right) + \exp  \left(- \dfrac{3a_1q(n)}{28}\right). \label{bound_q(n)}
  \end{align}
where, $\tilde{\underbar{c}}$ is a new constant that incorporates functions of $a_1$ and $\underbar{c}$. It can be seen that the above upper bound is summable in $n$ under the condition
\begin{align}
\dfrac{\hgn^d \pign^2 q(n)}{\log{n}} \rightarrow \infty. \label{AutABomated_result}
\end{align}
Since $\epsilon$ is arbitrary, by the Borel-Cantelli Lemma, we have that $||\hat{f}_{i,n} - f_i||_\infty \rightarrow 0$, almost surely. This is true for all arms $1 \leq i \leq \ell$. Note that the result \text{a)} is similarly obtained by using \eqref{lemma_result_eta1} from Lemma \ref{lemma_theorem} to obtain a result similar to \eqref{ConditionOnAn} but with $\pi_n$ instead of $\pi_{\tau_n}$. Now, we can invoke Theorem \ref{Theorem 1} to establish strong consistency for both the strategies using the histogram method.
\end{proof}

\subsection{Kernel Regression} \label{sec:KernelRegression}
We can obtain analogous results for strong consistency of strategy $\eta_1$ and $\eta_2$ using Nadaraya-Watson estimator. Consider a nonnegative kernel function $K(u): \mathbb{R}^d \rightarrow \mathbb{R}$ that satisfies the following Lipschitz and boundedness conditions.

\begin{assumption} \label{Ass:Ker1}
For some constants $0 < \lambda < \infty$,
$
|K(u) - K(u^\prime)| \leq \lambda ||u - u^\prime||_\infty,$ for all $u, u^\prime \in \mathbb{R}^d$.

\end{assumption}

\begin{assumption}\label{Ass:Ker2}
$\exists$ constants $L_1 \leq L, c_3 > 0$ and $c_4 \geq 1$ such that $K(u) = 0$ for $||u||_\infty > L, K(u) \geq c_3$ for $||u||_\infty \leq L_1$, and $K(u) \leq c_4$ for all $u \in \mathbb{R}^d$.
\end{assumption}
Recall, $\tau_n = \sum_{j=1}^n I(t_j \leq n)$, the number of observed rewards by time $n$. Define, $J_{i,n+1} = \{j: I_j = i, t_j \leq n, 1\leq j \leq n\}$, that is, the set of time points corresponding to pulling of arm $i$ whose rewards have been observed by time $n$. Let $M_{i,n+1}$ denote the size of $J_{i,n+1}$.

 Let $h_{\tau_n}$ denote the bandwidth, where $h_{\tau_n} \rightarrow 0$ almost surely as $n \rightarrow \infty$. For each arm $i$, the Nadaraya-Watson estimator of $f_i(x)$ is defined as,
\begin{align}
\hat{f}_{i,n+1}(x) = \dfrac{\sum_{j \in J_{i,n+1}}Y_{i,j} K \left(\frac{x - X_j}{h_{\tau_n}} \right)}{\sum_{j \in J_{i,n+1}}K \left(\frac{x - X_j}{h_{\tau_n}} \right)}. \label{NadarayaWatsonEst}
\end{align}
\begin{theorem} \label{thm:KernelTheorem}
Suppose Assumptions \ref{Ass2}-\ref{Ass:Ker2} are satisfied, and,
\begin{enumerate}
\item If $\{h_n\}$ and $\{\pi_n\}$ are chosen to satisfy, \begin{align*}
\dfrac{q(n)h_{q(n)}^{2d} \pi_{n}^4}{\log{n}} \rightarrow \infty,
\end{align*}
then the Nadaraya-Watson estimator defined in \eqref{NadarayaWatsonEst} is strongly consistent in $L_\infty$ norm for strategy $\eta_1$.
  \item If $\{h_n\}$ and $\{\pi_n\}$ are chosen to satisfy, \begin{align*}
\dfrac{q(n)h_{q(n)}^{2d} \pi_{q(n)}^4}{\log{n}} \rightarrow \infty,
\end{align*}
then the Nadaraya-Watson estimator defined in \eqref{NadarayaWatsonEst} is strongly consistent in $L_\infty$ norm for strategy $\eta_2$.
\end{enumerate}
\end{theorem}
\begin{proof}
The proof for this theorem can be found in Appendix \ref{proofconsistency_kernel_AppB}.
\end{proof}
\subsection{Strategy $\eta_1$ versus Strategy $\eta_2$} \label{compareSPlvsThis}
 \cite{arya2020randomized} conduct an analysis for the randomized allocation strategy with $h_n,\pi_n$, that is, when both sequences are updated at every time point regardless of the delays, and establish its strong consistency. It states that, for $q(n)$ as in Assumption 6, if $h_n, \pi_n$ are chosen to satisfy,
\begin{align}
\dfrac{h_n^d \pi_n^2 q(n)}{\log{n}} \rightarrow \infty \ \text{as} \ n\rightarrow \infty, \label{SPLresult}
\end{align}
then the proposed allocation rule is strongly consistent for the histogram method. Note that, in terms of handling the delays, this allocation rule is in the opposite direction of the black-box approach that simply applies an existing method on the available data (i.e., ignoring all the cases with unobserved rewards at the time of decision). The sharp contrast called for the present investigation of the alternative ways to use $\pi_n$ and $h_n$ and understand their relative strengths and weaknesses.

Now if we compare \eqref{conditionforThm1_eta1}, \eqref{conditionforThm1_eta2} and \eqref{SPLresult}, we see that 
$\eqref{SPLresult} \Rightarrow \eqref{conditionforThm1_eta2} \Rightarrow \eqref{conditionforThm1_eta1}$,
 but not vice versa, therefore \eqref{conditionforThm1_eta1} seems to give more options for the choice of the user-determined sequences, $\{h_n\}$ and $\{\pi_n\}$, to achieve consistency while there may be a trade-off in the rate of decrease of the average cumulative regret as we will see in the simulations. 
  Note that, we notice a similar relationship in Theorem \ref{thm:KernelTheorem} when using Kernel regression.
 To understand which choices of hyper-parameter sequences help minimize the cumulative regret, let us consider the regret for a strategy $\eta$,
 \begin{align*}
  R_N(\eta) &= \sum_{j = m_0+1}^N (f^*(X_j) - f_{I_j}(X_j))\\
 &= \sum_{j= m_0 + 1}^N (f_{i^*_j}(X_j) - \hat{f}_{i^*_j}(X_j) + \hat{f}_{i^*_j}(X_j) - \hat{f}_{I_j}(X_j) + \hat{f}_{I_j}(X_j) - f_{I_j}(X_j))\\
  & \leq  \sum_{j= m_0 + 1}^N (f_{i^*_j}(X_j) - \hat{f}_{i^*_j}(X_j) + \hat{f}_{\hat{i}_j}(X_j) - \hat{f}_{I_j}(X_j) + \hat{f}_{I_j}(X_j) - f_{I_j}(X_j))\\
  & \leq \sum_{j = m_0+1}^N 2 \sup_{1 \leq i \leq \ell} |f_i(X_j) - \hat{f}_{i}(X_j)| + A I\{I_j \neq \hat{i}_j\}.
  \end{align*} 
  Thus we can roughly decompose the cumulative regret into estimation error and randomization error. For the no-delay setting, \cite{qian2016randomized} study both these error components in a finite-time setting and show that, $\{h_n\}$ and $\{\pi_n\}$ can be chosen to achieve an optimal (minimax) rate of convergence for the regret. In their work, the choices of $\{h_n\}$ and $\{\pi_n\}$ also depend on the smoothness parameter of the mean reward functions. Thus in situations where the mean reward functions are simple and smoother, $\{h_n\}$ and $\{\pi_n\}$ are chosen to be fast decaying to achieve optimal rates of convergence in no-delay situations. In contrast, for scenarios where the underlying mean reward functions are more complex, they are chosen to be relatively slow decaying in order to guarantee optimal rates. Now the question that arises in the presence of delayed rewards is that, how should sequences $\{h_n\}$ and $\{\pi_n\}$ be updated, so as to minimize the resulting cumulative regret? That is, should one update $\pi_n$ to $\pi_{n+1}$ (and $h_n$ to $h_{n+1}$) at every time point irrespective of observing a reward or only update upon observing a new reward. Let us try to understand the impact of delay and the reward generating mechanisms on the two components of cumulative regret to answer this question. 

 Different nonparametric methods may be used for estimation purposes, and estimation accuracy largely depends on the complexity of the underlying mean reward functions and the amount of data available for estimation. The binwidth of methods like histogram and kernel regression, usually is a function of the number of data points available for estimation at a given point. Therefore, in the presence of delayed rewards, $h_{\tau_n}$ ($\tau_n$ being the number of observed rewards until $n$) seems to be the sensible choice for the binwidth. Choosing $h_n$ may lead to inefficient estimation due to unavailability of data points in some small neighborhood of $[0,1]^d$. Therefore, employing a binwidth sequence that guarantees optimal rates of convergence in the no-delay setting, which updates only when a new reward is obtained, seems to be the right choice from an estimation point of view. Hence, we only consider the policies ($\eta_1$ and $\eta_2$) that employ $h_{\tau_n}$ as the chosen binwidth sequence. It is important to note that from an asymptotic point of view, based on our theoretical results (Theorem \ref{thm:theorem}), estimation will improve with time, but this discussion is from a finite time perspective.

In terms of randomization error, delayed rewards affect this directly through the randomization scheme. This is tied to the exploration-exploitation dilemma which is in turn controlled by the exploration probability $\{\pi_n\}$. In the following illustrations, we try to convey the message of why carefully balancing exploration-exploitation is tied to updating the sequence $\{\pi_n\}$ carefully in the presence of delayed rewards, and the decision to do that can vary in different situations.

\textbf{Illustration 1.} Suppose that the mean reward functions are not too complex and are well-separated. In this setting, it will be easy to get good functional estimates over time, even with less observed data due to presence of large delays. Since the no-delay case is well-studied, for such a setting we could choose an exploration probability sequence $\{\pi_n\}$ that gives the optimal rate of convergence according to \cite{qian2016randomized}. Now, with the delays, we need to decide whether we want to update $\pi_n$ to $\pi_{n+1}$ for each $n$ or only when a new reward is observed.  In this setting, it would perhaps be advantageous to opt for strategy $\eta_1$, which updates at every time step irrespective of whether a reward is obtained or not. 
This is because using strategy $\eta_2$ may lead to excessive exploration which may be unnecessary in such settings even for large delay situations. Thus using $\eta_1$ will lead to a smaller randomization error. In order to illustrate that, 
let $\text{Rand}_j(\eta_1)$ and $\text{Rand}_j(\eta_2)$ denote the indicator $I(I_j \neq \hat{i}_j)$ for $\eta_1$ and $\eta_2$, respectively. Let $\sigma_t = \min\{\bar{n}: \sum_{j=m_0+1}^{\bar{n}} I(t_j \leq N) \geq t \}$, that is, $\sigma_t$ is the time index where the $t$\ts{th} reward is observed. Then we have that,
\begin{align}
\E_{\mathcal{A}_N}(\sum_{j=m_0+1}^N \text{Rand}_j(\eta_2)) &= \sum_{j=m_0+1}^N P_{\eta_2, \mathcal{A}_N}(I_j \neq \hat{i}_j)= \sum_{t = 1}^{\tau_N} (\sigma_{t+1} - \sigma_t) (\ell - 1) \pi_t, \label{Rand_eta_1_E}
\end{align}
where $\E_{\mathcal{A}_N}$ denotes conditional expectation given $\mathcal{A}_N$, the set of indices when the rewards were observed by time $N$. Here, $\tau_N = \sum_{j=m_0+1}^{N} I(t_j \leq N)$, number of rewards observed between time $m_0$ and $N$. However, for strategy $\eta_1$, since the exploration probability $\{\pi_j\}$ does not depend on delays, we have that,
\begin{align}
\E(\sum_{j = m_0 + 1}^N \text{Rand}_j(\eta_1)) = \sum_{j=m_0 + 1}^N P_{\eta_1}(I_j \neq \hat{i}_j) =  \sum_{j=1}^{N-m_0-1} (\ell -1) \pi_j. \label{eta1_rand_error}
\end{align}
For brevity sake, let us denote $\bar{N} = N - m_0 - 1$ and we start the counting process at $m_0 + 1$.
Now, given $\tau_N$, the minimum value that we can get for the R.H.S. in \eqref{Rand_eta_1_E} is when all the rewards from $m_0 +1$ until $\tau_N$ are observed instantaneously and after that no reward is observed until we hit the horizon $\bar{N}$. Likewise, an approximate maximum value of R.H.S. in \eqref{Rand_eta_1_E} is achieved when the rewards for $(m_0+1)$\ts{th} through $(\bar{N}-\tau_N)$\ts{th} arms are not observed until time $(\bar{N} - \tau_N)$, and we observe $\tau_n$ many, from time $\bar{N} - \tau_N+1$ to $\bar{N}$ respectively. Therefore,
\begin{align*}
\min_{\mathcal{A}_N} {\E_{\mathcal{A}_N}(\sum_{j = {m_0 + 1}}^N \text{Rand}_j(\eta_2))}  &= (\ell-1)[\sum_{t=1}^{\tau_N-1} \pi_t + (\bar{N}-\tau_N)\pi_{\tau_N}],  
\end{align*}
\begin{align*}
\max_{\mathcal{A}_N} {\E_{\mathcal{A}_N}(\sum_{j = {m_0 + 1}}^N \text{Rand}_j(\eta_2))} &=  (\ell-1)[(\bar{N}-\tau_N)\pi_1 + \sum_{t=2}^{\tau_N} \pi_t]. 
\end{align*}
For the sake of illustration, assume that we observe a fraction of $\bar{N}$ by time $N$, that is,  $\tau_N = \alpha \bar{N}$, for some $\alpha \in (0,1)$. Then we have that, 
\begin{align}
\min {\E(\sum_{j = {m_0 + 1}}^N \text{Rand}_j(\eta_2))}  &= (\ell -1)[\sum_{t=1}^{\tau_N-1} \pi_t + (1-\alpha)\bar{N}\pi_{\tau_N}], \label{special_case_illustration_min}\\
\max {\E(\sum_{j = {m_0 + 1}}^N \text{Rand}_j(\eta_2))}  &= (\ell - 1)[(1-\alpha)\bar{N} \pi_1 + \sum_{t = 2}^{\tau_N} \pi_t]. \label{special_case_illustration_max}
\end{align}
Notice that the terms $(1-\alpha)\bar{N}\pi_1$ and $(1-\alpha)\bar{N}\pi_{\tau_N}$ in the RHS in \eqref{special_case_illustration_min} and \eqref{special_case_illustration_max} can be fairly large and grow as $N$ increases for all reasonably fast choices of $\{\pi_n\}$ such as, $n^{-1/4}, \log^{-1}{n}$. From \eqref{eta1_rand_error}, \eqref{special_case_illustration_min} and \eqref{special_case_illustration_max}, we also get that,
\begin{align}
\sum_{t = \tau_N + 1}^{\bar{N}} (\ell - 1) (\pi_{\tau_N} - \pi_t) \leq \E(\sum_{j= m_0 + 1}^N \text{Rand}_j(\eta_2) - \text{Rand}_j(\eta_1)) & \leq \sum_{t= \tau_N + 1}^{\bar{N}} (\ell -1) (\pi_1 - \pi_t), \label{rand_error_illustration}
\end{align}
where it can be seen that $\sum_{t = \tau_N + 1}^{\bar{N}} (\ell-1) (\pi_{\tau_N} - \pi_t) > 0$ for any $N$ and $\sum_{t= \tau_N + 1}^{\bar{N}}(\ell -1) (\pi_1 - \pi_t) \rightarrow \infty$ as $N \rightarrow \infty$. Therefore, we see that using strategy $\eta_1$, which updates $\pi_n$ at every time step irrespective of having observed a reward or not, gives a lower randomization error on average as compared to strategy $\eta_2$. For example, if we choose $ \{\pi_n\} = n^{-1/4}$, $\alpha = 0.25$ (one-fourth of rewards observed) and $m_0 = 30$ (initialization phase), time horizon $N = 10000$, then  we get that the average randomization error difference approximately satisfies, 
\begin{align*}
0.02 (\ell-1) \leq  \dfrac{\E(\sum_{j= m_0 + 1}^N \text{Rand}_j(\eta_2) - \text{Rand}_j(\eta_1))}{N-(m_0+1)} \leq 0.23 (\ell - 1), 
\end{align*}
for $N = 10000, m_0 = 30$.
In situations where mean reward functions are not complex, the randomization error can be quite large and potentially dominate over the estimation error. Thus, using strategy $\eta_1$ may reduce the cumulative regret substantially as compared to strategy $\eta_2$ in such situations.

\textbf{Illustration 2.} On the other hand, there are situations in which it may be better to use strategy $\eta_2$ with $\pi_{\tau_n}$ (updating only when a new reward is observed) as the exploration probability sequence. For example, scenarios where the best arms frequently alternate over regions of covariate space in terms of maximizing reward and it is hard to tell a clear winner with less information available due to presence of large delays. Another such situation is when an arm which is inferior in majority of the covariate space, but is superior with a substantial reward gain in a very small area of the domain and it might be the case that under large delays these under-represented regions remain unexplored. As described, let us assume that the underlying mean reward functions are somewhat complex. In such settings, we would need substantial exploration for a long period of time, specially in the presence of large delays. Here, in the hope of reducing the randomization error, we could employ strategy $\eta_1$ and use an exploration probability sequence $\pi_n$, which meets the conditions in \cite{qian2016randomized} that ensure optimal convergence rates in no-delay situations. However, this could be disadvantageous in such complex settings. This is because using $\eta_1$ may lead to insufficient exploration for the inferior arms. We consider the event that a seemingly inferior arm is chosen at time $t$, that is, $I(I_t \neq \hat{i}_t)$. Then to ensure enough exploration, we need that this event occurs with a positive probability that is not too small, specially in such complex settings as discussed above.  From \cite{yang2002randomized} and \cite{qian2016kernel} for no delay settings, we know that it is necessary to have $\sum_{t=1}^\infty \pi_t = \infty$ for the algorithm to perform optimally both asymptotically and in finite time. We also know that $\tau_N \overset{\text{a.s.}}{\rightarrow} \infty$ as $N \rightarrow \infty$. Therefore, using both these facts, the sum of probability of the event $\{I(I_t \neq \hat{i}_t), t \geq 1\}$, over the time points where rewards are observed for strategy $\eta_2$ goes to $\infty$,
   \begin{align*}
    \sum_{t=1}^{\tau_N} P_{\eta_2}(I_t \neq \hat{i}_t)&= \sum_{t=1}^{\tau_N} (\ell-1)\pi_t \overset{\text{a.s.}}{\rightarrow} \infty, \ \text{as} \ N \rightarrow \infty,
   \end{align*}
   whereas, for $\eta_1$, this sum could actually be summable for large delay situations. Let $\sigma_t = \min\{\bar{n}: \sum_{j=m_0+1}^{\bar{n}} I(t_j \leq N) \geq t \}$. Let us assume that the observed rewards are equally spaced, that is, $\sigma_t = tN/\tau_N$, assuming w.l.o.g that $N/\tau_N$ is an integer. Then, we have, 
   \begin{align*}
\sum_{t=1}^{\tau_N} P_{\eta_1}(I_t \neq \hat{i}_t)&= \sum_{t=1}^{\tau_N} (\ell-1)\pi_{\sigma_t}
=  \sum_{t=1}^{\tau_N} (\ell -1) \pi_{tN/\tau_N}.
   \end{align*}
Now, it can be shown that this series is summable for various choices of $\{\pi_n\}$. For example, let $\{\pi_n\} = n^{-1/2}$, then for strategy $\eta_1$, 
\begin{align}
\sum_{t=1}^{\tau_N} P_{\eta_1}(I_t \neq \hat{i}_t) = \sum_{t=1}^{\tau_N} (\ell -1) \pi_{tN/\tau_N} &= \sum_{t=1}^{\tau_N} \left(\dfrac{tN}{\tau_N}\right)^{-1/2}\nonumber\\
& = \left(\dfrac{N}{\tau_N}\right)^{-1/2} \sum_{t=1}^{\tau_N} t^{-1/2} = O \left(\dfrac{\tau_N}{\sqrt{N}}\right). \label{summable_eta1}
\end{align}
 If the number of observed rewards are small, say $\tau_N = O(\sqrt{N})$, then the series in \eqref{summable_eta1} is summable. Therefore by Borel-Cantelli Lemma, the event $\{I_t \neq \hat{i}_t\}$ occurs only finitely many times out of all instances where the rewards are observed. This will lead to insufficient exploration and may incur large regret in areas that remain unexplored, specially in the more complex settings.
  Therefore, if we employ strategy $\eta_1$ in such settings with large delays, we may end up over-exploiting certain arms and as a result obtain insufficient number of rewards pertaining to a seemingly inferior arm, which may possibly yield higher rewards in some unexplored regions in future. This would adversely affect the performance of the algorithm and lead to high cumulative regret. Therefore, in scenarios like this, it would be advantageous to use strategy $\eta_2$. 

  Note that, $\eta_2$ can be thought of as a black-box procedure, in the sense that it only updates at the time points where at least one reward is observed as if there were no delays. From the above discussion, we can conclude that taking the black-box approach might not necessarily be the best in handling delayed rewards in a contextual bandit problem.
 In the next section, we demonstrate these ideas using four different simulation setups and illustrate the performance of strategies $\eta_1$ and $\eta_2$ in the four setups respectively. These insights also suggest the need for studying adaptive strategies for updating these parameters in a local fashion, a promising direction to explore in future.

\section{Simulations}\label{sec:simulation}
 We conduct a simulation study to compare the per-round average regret for strategies $\eta_1$ and $\eta_2$ under different delayed rewards scenarios. The per-round regret for strategy $\eta$ is given by,
\begin{align*}
r_n(\eta) = \frac{1}{n} \sum_{j=1}^n (f^*(X_j) - f_{I_j}(X_j)).
\end{align*} 
Note that, if $\frac{1}{n} \sum_{j=1}^n f^*(X_j)$ is eventually bounded above and away from 0 with probability 1, then $R_n(\eta) \rightarrow 1$ a.s. is equivalent to $r_n(\eta) \rightarrow 0$ a.s. The data has been generated from the following mean reward functions. We assume $d = 2, \ell = 2 \ (\text{or}\ 3)$ and $x \in [0,1]^2$ and the simulations run  until time $N = 8000$ with first 30 rounds of initialization. For each of the setups, we define one-dimensional functions $g_1$ and $g_2$, and then for $x_1, x_2 \in [0,1]$, we define, $f_1(x_1,x_2) = g_1(x_1)*x_2$ and $f_2(x_1,x_2) = g_2(x_1)*x_2$. \\
\textbf{Setup 1: }  In this setup, we consider two well-separated sinusoidal functions, where one is a shifted above version of the other.
\begin{align*}
g_1(x) = (-2\sin(20\pi x) + 3), \ g_2(x) = (-2\sin(20 \pi x) + 2); \ \ x \in [0,1].
\end{align*}
\textbf{Setup 2: } Consider three piecewise-linear functions that are well-separated but over different regions in the covariate space. Then, $f_1(x_1,x_2) = x_2 g_1(x_1),  
f_2(x_1,x_2) = x_2 g_2(x_1),
f_3(x_1,x_2) = x_2 g_3(x_1)$.
\begin{align*}
  g_1(x) = \begin{cases}
  1 & 0 \leq x < 0.5\\
  -10x + 6 & 0.5 \leq x < 0.6\\
  0 & x \geq 0.6
  \end{cases},& \ 
  g_2(x)  = \begin{cases}
  0 & 0 \leq x < 0.5\\
  10x - 5 & 0.5 \leq x < 0.6\\
  1 & x \geq 0.6
  \end{cases},
\end{align*}
\begin{align*}
  g_3(x) = \begin{cases}
  0 & 0 \leq x < 0.3\\
  20x - 6 & 0.3 \leq x < 0.4\\
  2 & 0.4 \leq x < 0.6\\
  -20x + 14 & 0.6 \leq x < 0.7\\
  0 & x \geq 0.7.
  \end{cases}
\end{align*}
\textbf{Setup 3: } Consider two sinusoidal functions such that the best arm alternates rapidly as the functions oscillate.
\begin{align*}
g_1(x) = 2\cos(5 \pi x) + 2, \ \ g_2(x) = -2\sin(5\pi x) + 2, \ \text{for} \ x \in [0,1].
\end{align*}
\textbf{Setup 4: } Consider a setup where one arm dominates over majority of the covariate space, except for a small area where it incurs a considerably high regret.
\begin{align*}
g_1(x) = 1, \ \text{for all}\ x \in [0,1];  \ \ 
g_2(x) = 
\begin{cases}
0 & 0 \leq x < 0.5, 0.505 \leq x \leq 1\\
100000x - 50000 & 0.5 \leq x < 0.502\\
200 & 0.502 \leq x < 0.503\\
-100000*x + 50500 & 0.503 \leq x < 0.505.\\
\end{cases}
\end{align*}
We look at both the setups 1) $d = 1$, when $f_1(x) = g_1(x)$ and $f_2(x) = g_2(x)$ and 2) $d=2$, when $f_1(x_1, x_2) = g_1(x_1)*x_2$ and $f_2(x_1,x_2) = g_2(x_1)*x_2$, but only the results for 2) are displayed in Figure \ref{fig: Simulation_result_eta2_better}. The one dimensional functions $g_i$ for each of these setups are plotted in Figure \ref{fig: Simulation_result_eta2_better}.

\subsection{The simulation process and results}
We simulate the data from the above mentioned true mean reward functions as:
$
Y_{i,j} = f_i(X_j) + 0.5 \epsilon_{j}, \ i \in \{1,2,3\}, j \in \mathbb{N},
$
where $\epsilon_{j} \overset{\text{i.i.d.}}{\sim} N(0,1)$. We use Nadaraya-Watson estimator with Gaussian kernel to estimate the mean reward functions. We run both strategies $\eta_1$ and $\eta_2$ as in Section \ref{algorithm}. We consider the following choices of hyper-parameter sequences but in our discussion, we only illustrate a few combinations to make a comparison for the sake of brevity. 
 $$\pi_n = \{n^{-1/4}, \log^{-1}{n}, \log^{-2}{n}; n \geq 1\}\  \text{and} \  h_n =\{ n^{-1/4}, n^{-1/6}, \log^{-1}{n}; n \geq 1\}.$$

Both strategies $\eta_1$ and $\eta_2$ are run for 60 independent replications (time horizon $N = 8000$). Then the regret is averaged for each time point over the replications, to give a more accurate estimate of the total regret accumulated up to a given time horizon. We create delay scenarios governing when a reward will be observed. We consider the following delay scenarios in the increased order of severity of delays,\\
 \textbf{No delay}; Every reward is observed instantaneously. \\
 \textbf{Delay 1:} Geometric delay with probability of success (observing the reward) $p = 0.3$.\\
  \textbf{Delay 2:} Every 5\ts{th} reward is not observed by time $N$ and other rewards are obtained with a geometric ($p = 0.3$) delay.\\
    \textbf{Delay 3:} Each case has probability 0.7 to delay and the delay is half-normal with scale parameter, $\sigma = 1500$.\\
    \begin{figure}[H]
 \centering
   \includegraphics[scale=0.3]{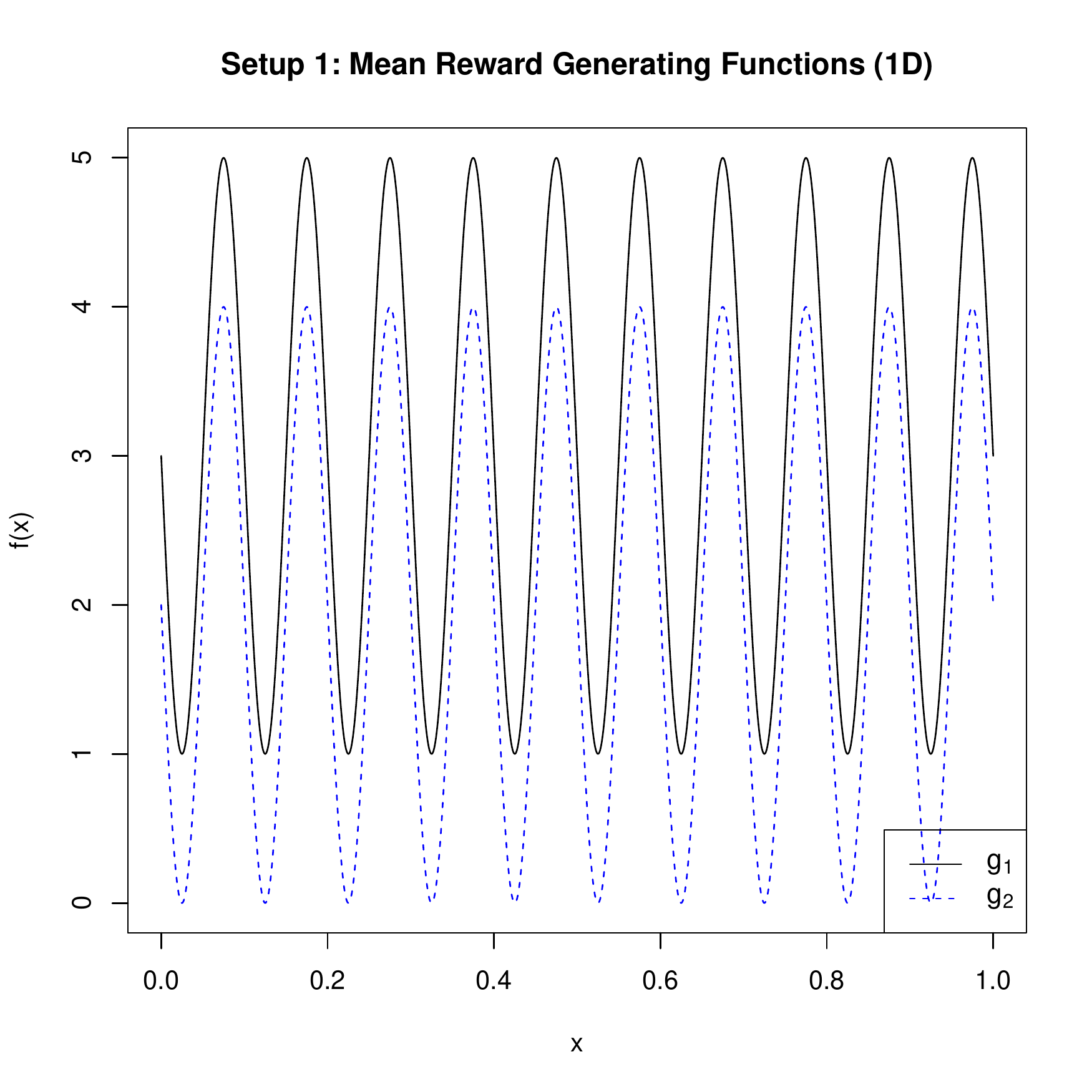}
   \includegraphics[scale=0.3]{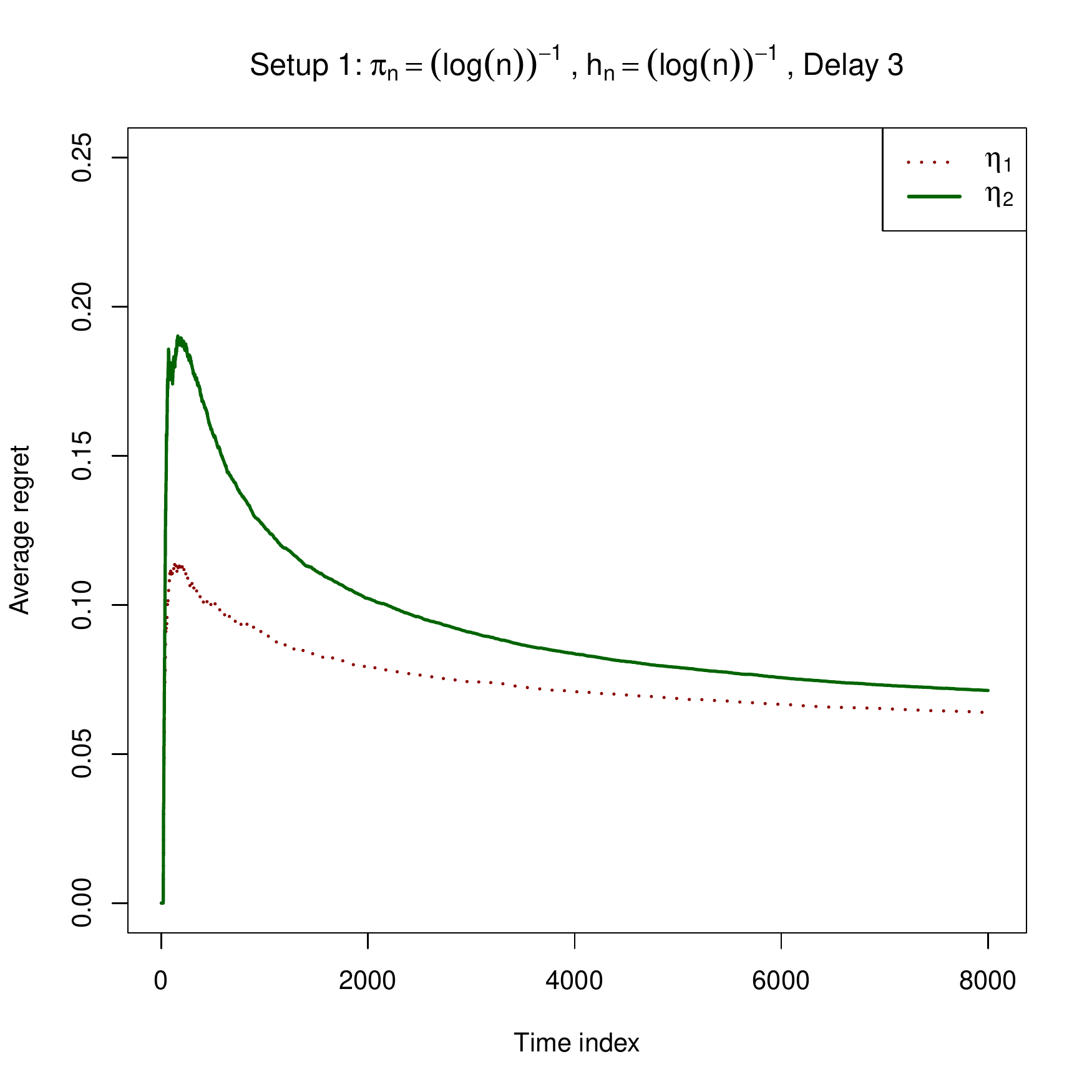}
   \includegraphics[scale=0.3]{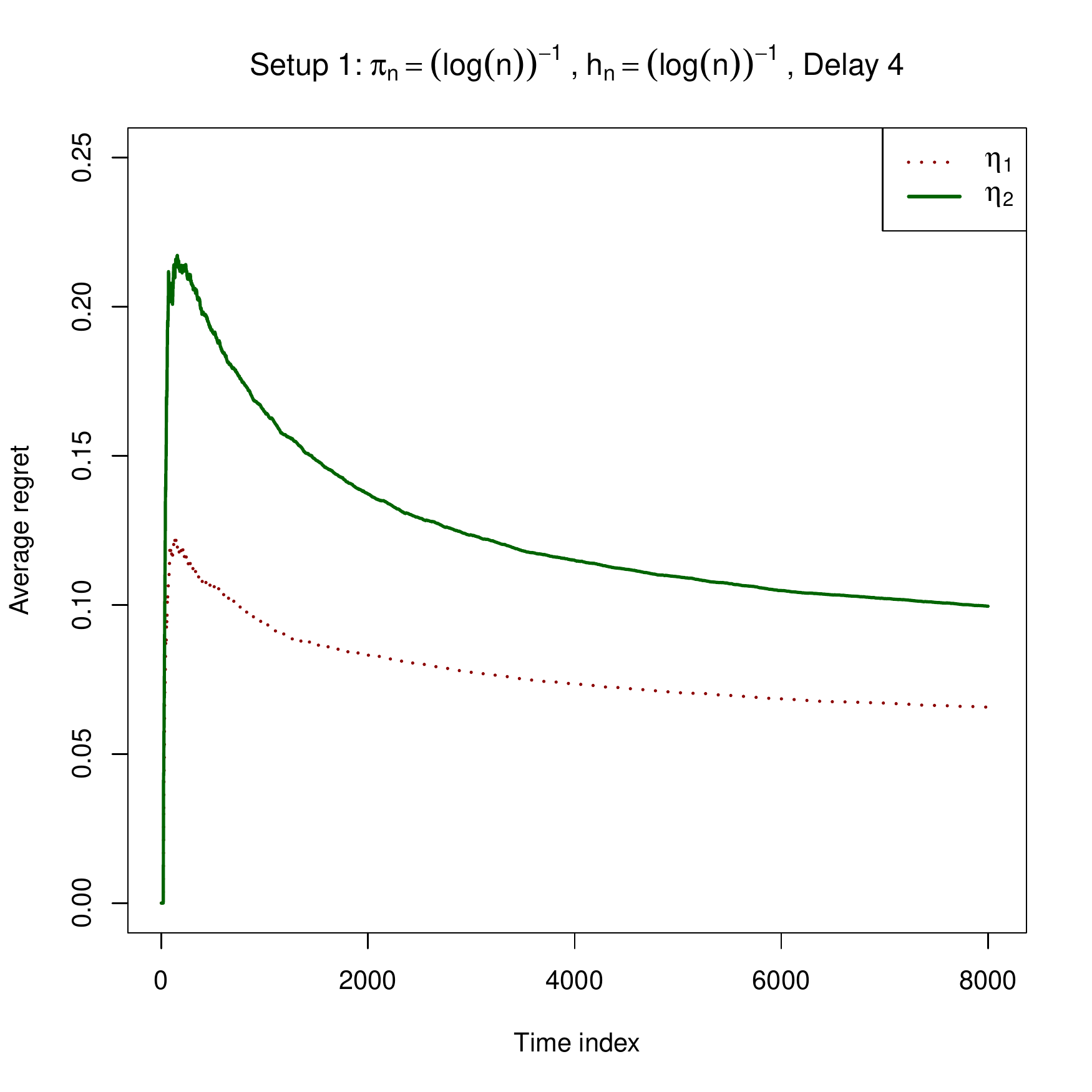}\\
   \includegraphics[scale=0.3]{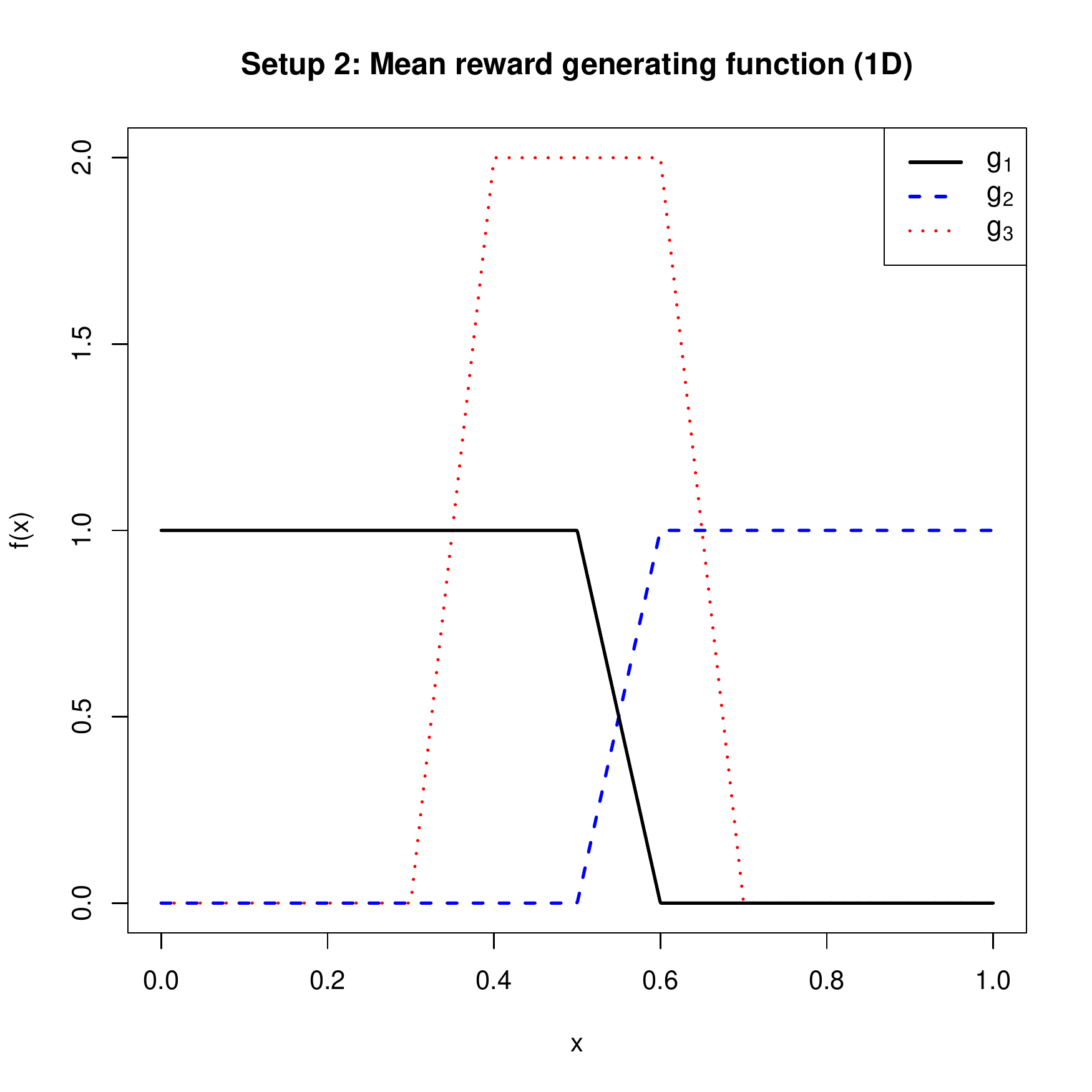}
   \includegraphics[scale=0.3]{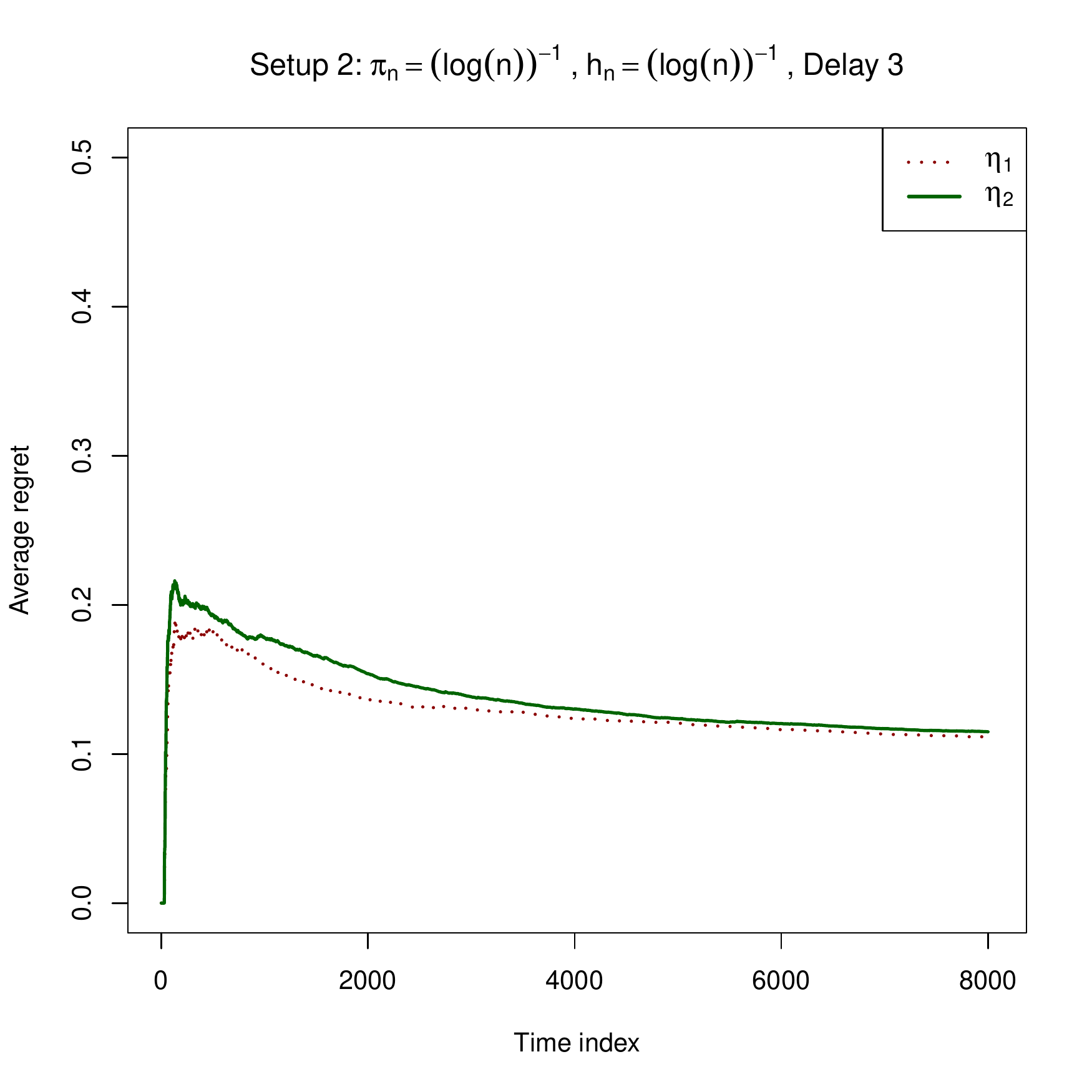}
   \includegraphics[scale=0.3]{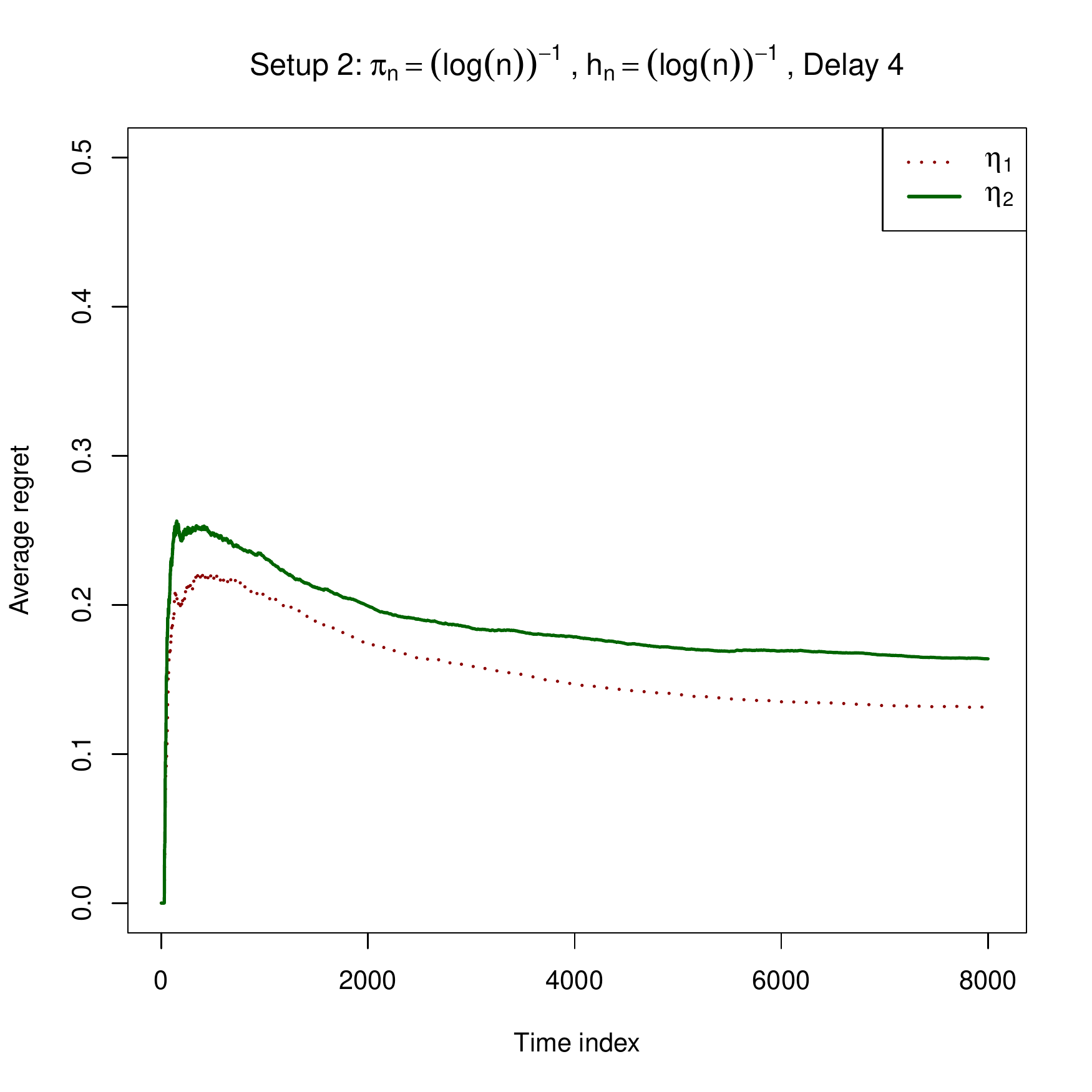}\\
\includegraphics[scale=0.3]{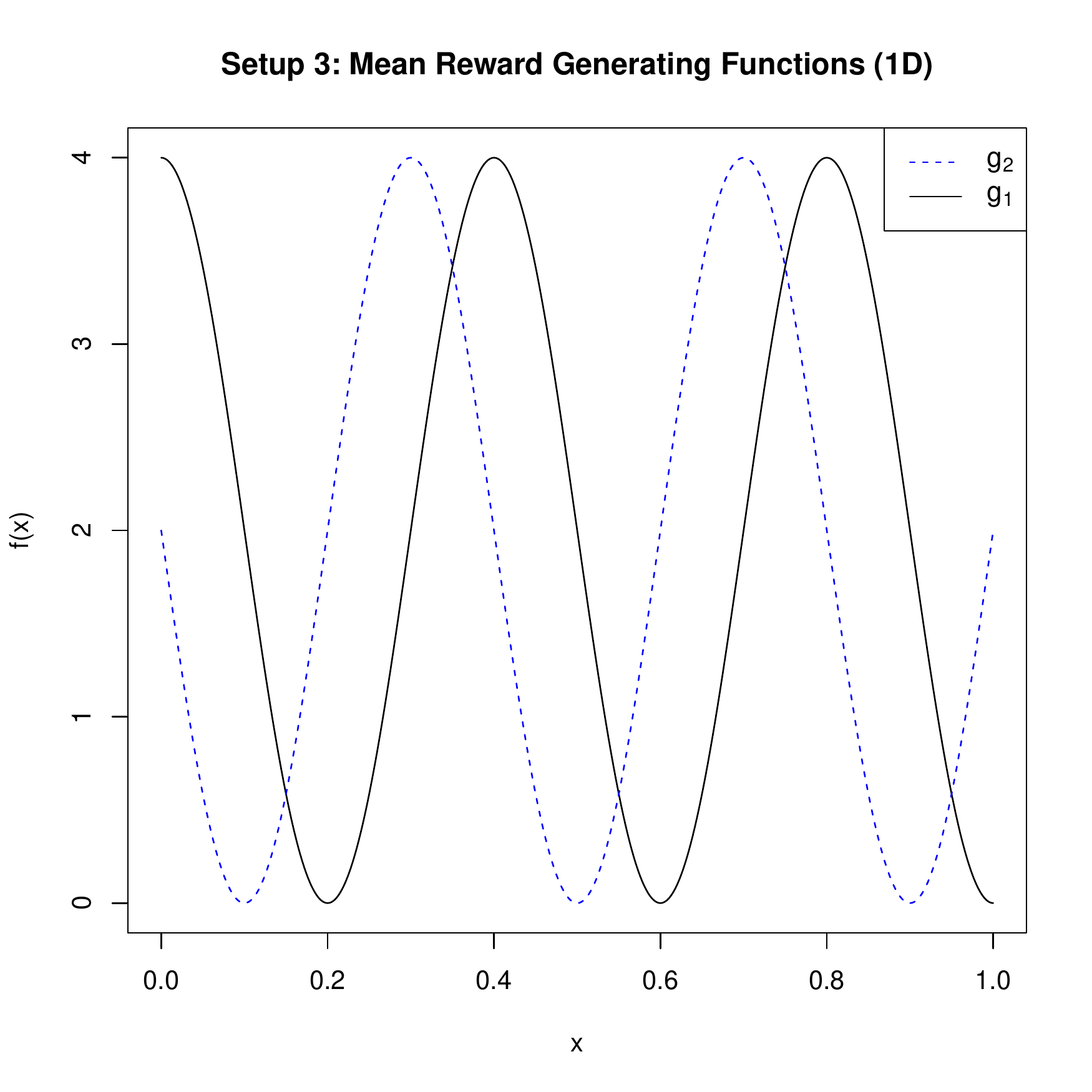}
   \includegraphics[scale=0.3]{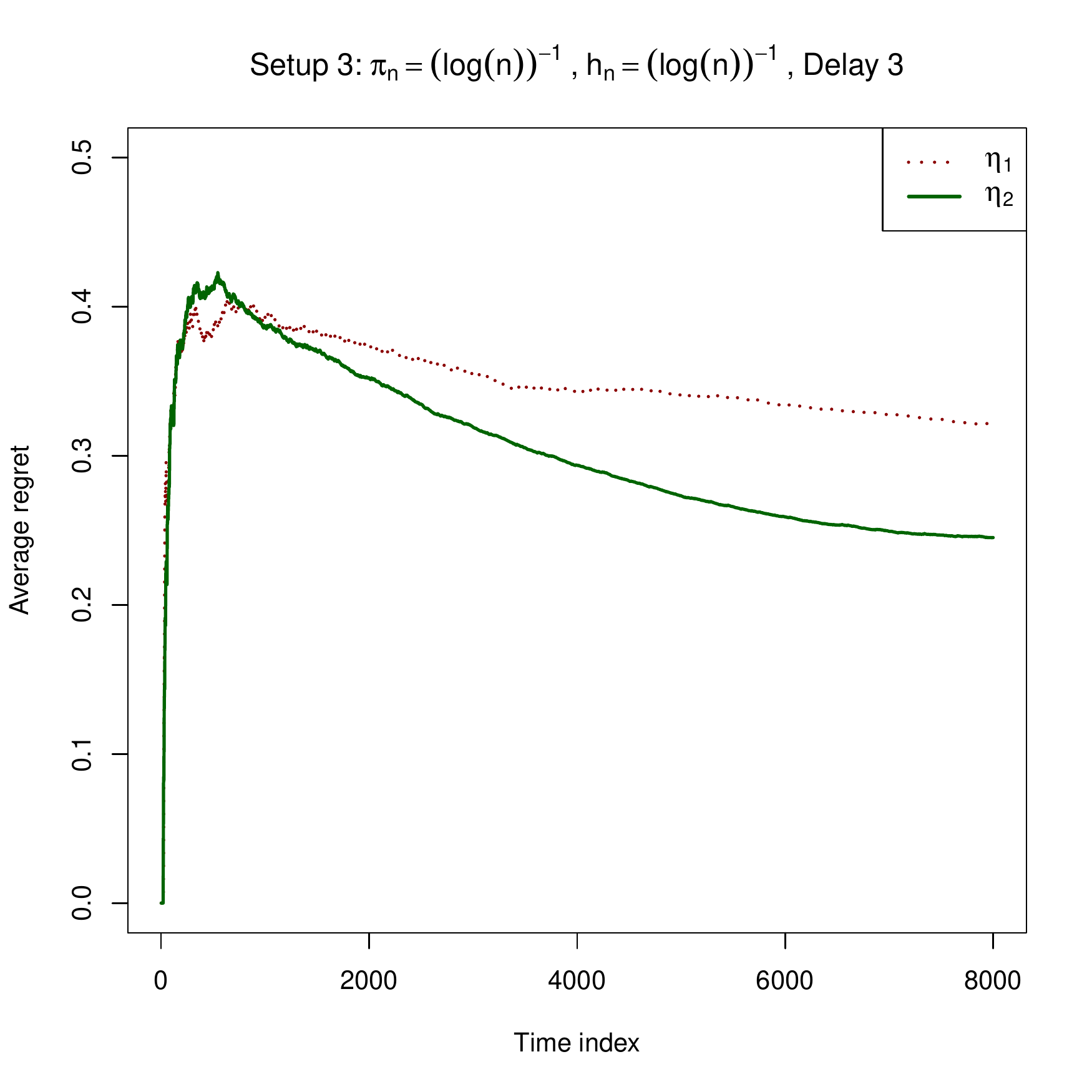}
    \includegraphics[scale=0.3]{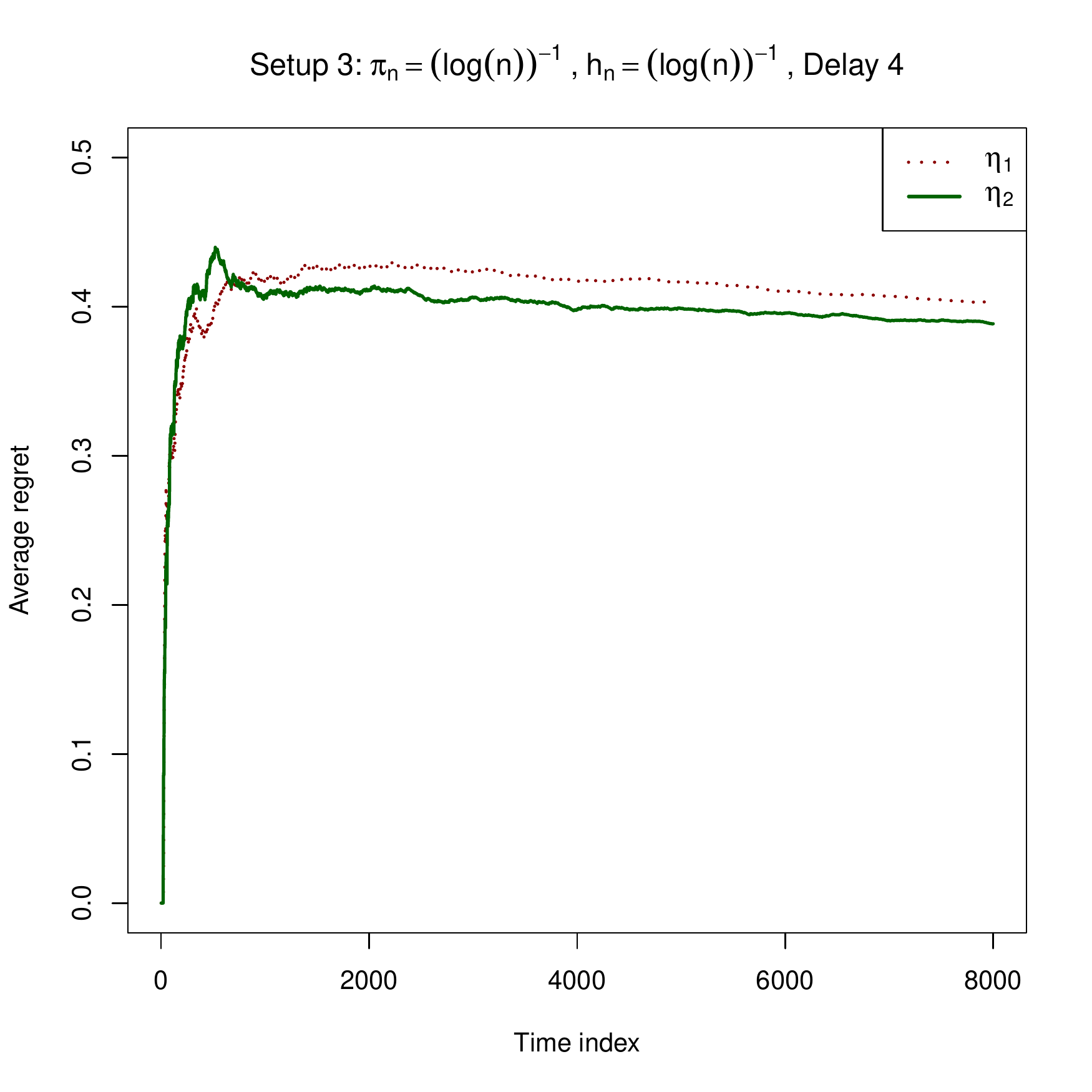}\\
   \includegraphics[scale=0.3]{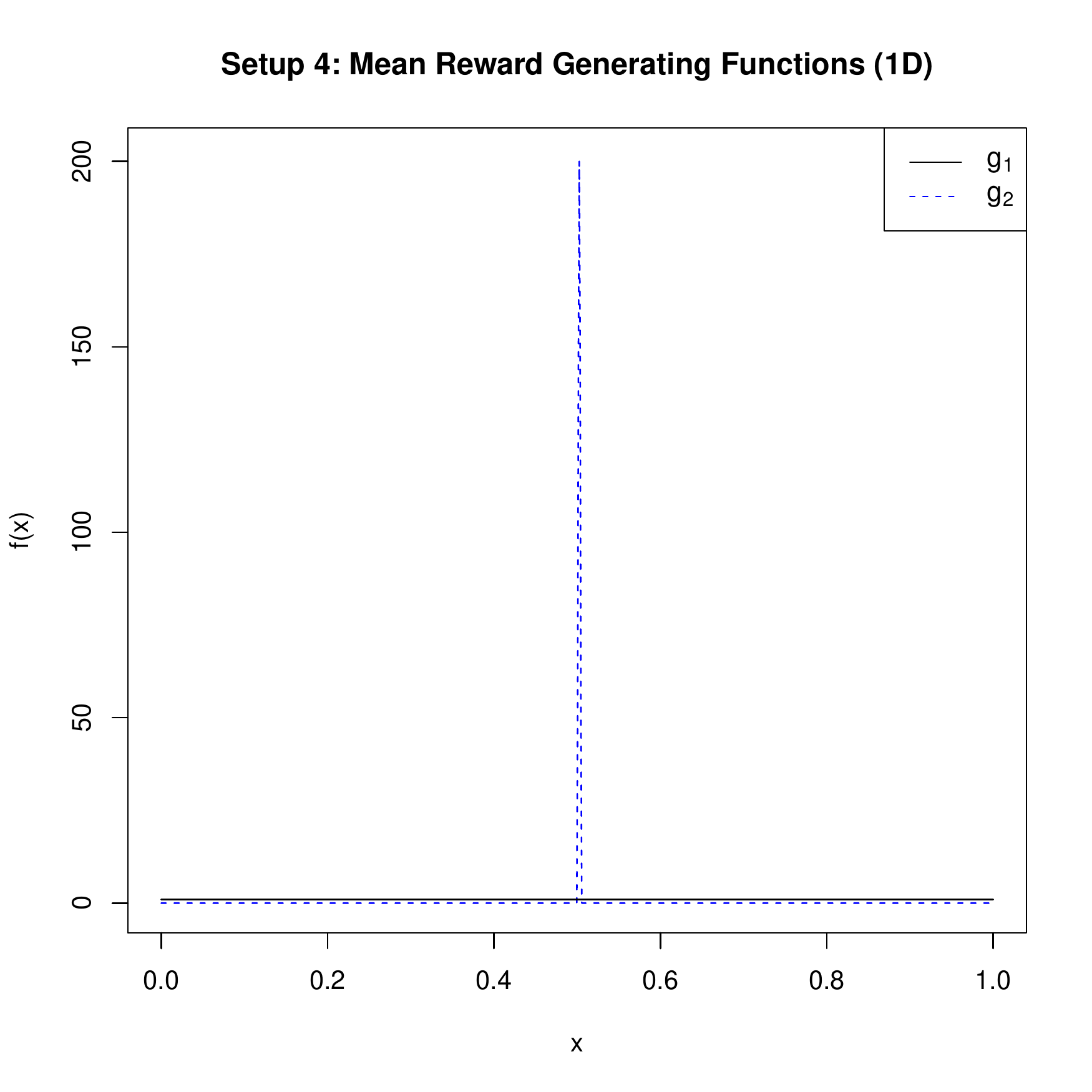}
   \includegraphics[scale=0.3]{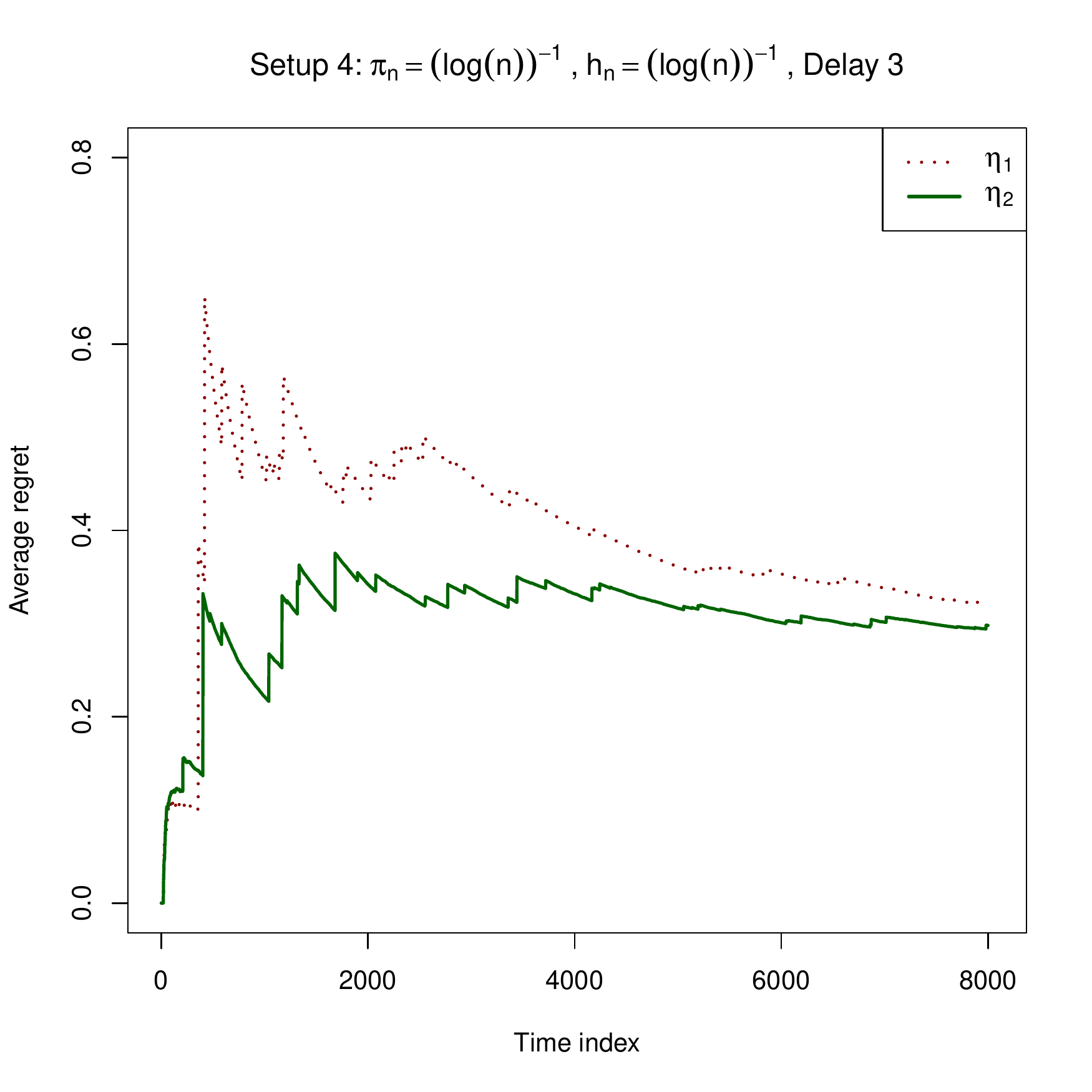}
   \includegraphics[scale=0.3]{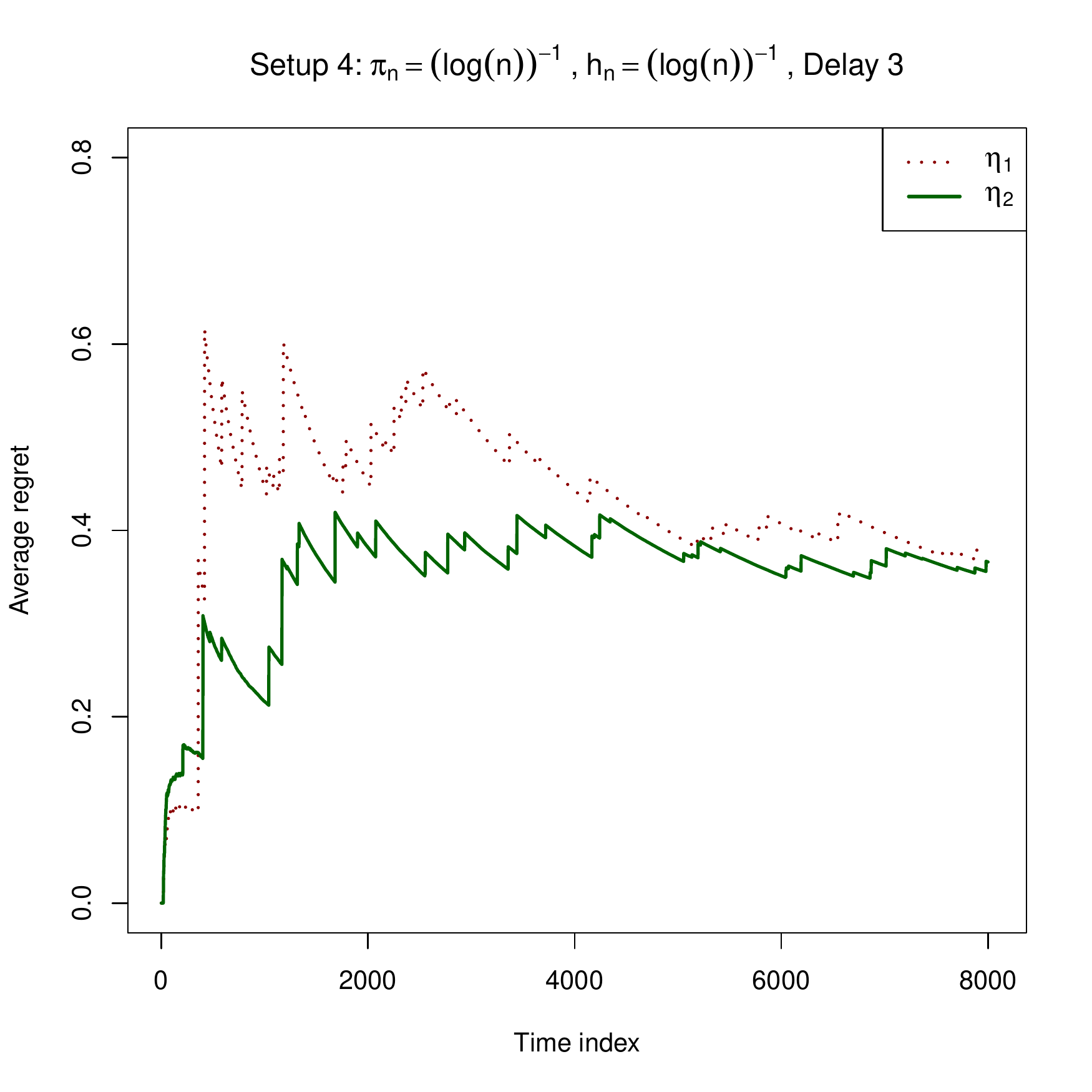}
  \caption{Strategy $\eta_1$ has lower cumulative average regret in setups 1 and 2 (first two rows) and strategy $\eta_2$ has lower cumulative average regret in setups 3 and 4 (rows third and fourth).}
 \label{fig: Simulation_result_eta2_better}
 \end{figure}
   \textbf{Delay 4:} In this case we increase the number of non-observed rewards. Divide the data into four equal consecutive parts (quarters), such that, in part 1, we only observe every 10\ts{th} (with Geom(0.3) delay) observation by time $N$ and not observe the remaining; in part 2, we only observe every 15\ts{th} observation; in part 3, only observe every 20\ts{th} observation; in part 4, only observe every 25\ts{th} observation.

In our simulations, we note that the difference in the cumulative regret is most discernible in the more extreme delay situations, that is, delay 3 and delay 4 in our setup. Therefore, we only illustrate the results on those two delay scenarios. The plots in Figure \ref{fig: Simulation_result_eta2_better} can be used to compare performance of strategy $\eta_1$ and $\eta_2$. On the $y$-axis is the average regret plotted against time on the $x$-axis. The rows in the figure correspond to the simulation setups and columns 2 and 3 correspond to Delay 3 and Delay 4 respectively. For illustration, we only show the plots corresponding to one choice of hyper-parameter sequences, $\{h_n\} = (\log{n})^{-1}$ and $\{\pi_n\} = (\log{n})^{-1}$, however results from other combinations show similar trends and are included in Appendix \ref{Simulation_extra_Appendix}.

Note that in setups 1 and 2, $\eta_1$ performs better than $\eta_2$ in terms of reducing the overall average regret. Both these setups consist of mean reward functions that are well-separated and clear winners in terms of reward gain in substantial portions of the covariate space. Therefore, it is likely that one can get good estimation even in large delay setting when only small amount of observed data is available for estimation. Thus, in these settings, controlling for the randomization error is crucial, which is better achieved by using $\pi_n$ instead of $\pi_{\tau_n}$, as illustrated in Section \ref{compareSPlvsThis}. On the contrary, in Setup 3 and 4, we notice that strategy $\eta_2$ performs better than $\eta_1$ in terms of lower average regret. This can be attributed to the fact that under large delay settings, one may require more exploration for a longer period of time to get good estimates for the complex mean reward functions. Therefore, using $\pi_{\tau_n}$ instead of $\pi_n$ helps improve the mean reward function estimation by exploring for a longer time, leading to a greater chance of exploring the more localized high regret incurring regions of the covariate space.  Another interesting observation is that for setups 1 and 2, the average regret curves for strategies $\eta_1$ and $\eta_2$ are closer with Delay 3 and much separated with Delay 4. Whereas, in setups 3 and 4, an opposite trend is seen, where the  difference in the average regret curves for $\eta_1$ and $\eta_2$ is more pronounced with Delay 3 as compared to Delay 4. A possible reason for this could be that the mean reward functions for setups 1 and 2 are easily distinguishable even with as few observations as with Delay 4, thus fast and continuous exploitation helps reduce the regret. However, the mean reward functions in setups 3 and 4 are harder to distinguish and perhaps with so few observations as in Delay 4, it is hard to do a good job in estimation even while exploring more using $\pi_{\tau_n}$.

\section{Conclusion}\label{sec:conclusion}

In this work, we present a case on the importance of carefully choosing a contextual bandit strategy based on the expected delay situation. Delays are assumed to be independent, but unbounded and could potentially be infinite as long as we expect to see a minimum number of observed rewards in finite time, and have some knowledge of a lower bound to the expected number of observations. We propose two $\epsilon$-greedy like strategies, adopting a nonparametric approach to modeling the mean reward regression functions. In both strategies, the binwidth sequence $\{h_n\}$ is updated only when new rewards are observed, but the difference lies in updating the exploration probability $\{\pi_n\}$. In one strategy, $\{\pi_n\}$ is only updated when a new reward is observed (like a black-box procedure), while in the second strategy, $\{\pi_n\}$ is updated at every time point irrespective of having observed a reward or not. We establish strong consistency for both the strategies and compare the necessary condition required to achieve consistency with the analogous condition that appeared in \cite{arya2020randomized}. Then, using some theoretical illustrations and simulation examples, we show that both these strategies may be advantageous in different settings depending on the underlying data generating scenarios and the severity of the delays in observing rewards. Therefore, based on these empirical results, we recommend that the choice of hyper-parameters $\{h_n\}$ and $\{\pi_n\}$ should depend on the context of the problem, delay scenario, and some broad knowledge of the data generating process. An immediate future direction based on these results is to devise an adaptive strategy which decides whether to update the hyperparameter sequences or not in a more localized way. Conducting a finite-time regret analysis to theoretically prove the insights obtained would help better understand the problem and we hope to address it in future work. It is important to note that optimal arm identifiability and regret minimization may not agree with each other in all problems. It is possible that two different algorithms achieve about the same cumulative regret, despite of one being poor at identifying the best arms as compared to the other, thus is a different problem altogether and requires a different set of tools to address the problem. In our knowledge, best arm identification in delayed rewards for contextual bandits has not been studied so far and would be an interesting future work to consider.

\bibliographystyle{apalike}
\bibliography{STAT_DelayedBandits_SakshiArya.bib}

\appendix
\section{Appendix} 
Here, we present supporting material that includes detailed proofs of Theorems 1 and 3 in the main paper and additional figures for more simulation results. 
\subsection{Proof of consistency of the proposed strategy}\label{Proof_consistency_eta2}
While strong consistency for strategy $\eta_1$ follows exactly from the proof of strong consistency in \cite{arya2020randomized}, some changes are required for proving the same for strategy $\eta_2$.
\begin{proof}[Proof of Theorem 1 for strategy $\eta_2$]
Since the ratio $R_n(\eta_2)$ is always upper bounded by 1, we only need to work on the lower bound direction. Note that,
\begin{align}
R_n(\eta_2) &= \dfrac{\sum_{j=1}^n f_{\hat{i}_j}(X_j)}{\sum_{j=1}^n f^*(X_j)} + \dfrac{\sum_{j=1}^n (f_{I_j}(X_j) - f_{\hat{i}_j}(X_j))}{\sum_{j=1}^n f^*(X_j)}\nonumber \\
&\geq \dfrac{\sum_{j=1}^n f_{\hat{i}_j}(X_j)}{\sum_{j=1}^n f^*(X_j)} - \dfrac{\frac{1}{n} \sum_{j=1}^n A I_{\{I_j \neq \hat{i}_j\}}}{\frac{1}{n} \sum_{j=1}^n f^*(X_j)}, \label{inequality_begin}
\end{align}
where the inequality follows from Assumption 2. Let $U_j = I_{\{I_j \neq \hat{i}_j\}}$. Since $(1/n)\sum_{j=1}^n f^*(X_j)$ converges a.s. to $\E f^*(X) > 0$, the second term on the right hand side in the above inequality converges to zero almost surely if $({1}/{n}) \sum_{j=1}^n U_j \overset{\text{a.s.}}{\rightarrow} 0$.
Note that for $j \geq m_0 +1$, $U_j$'s are independent Bernoulli random variables with success probability $(\ell-1)\pi_{\tau_j}$.
Now consider,
\begin{align*}
\sum_{j=m_0+1}^\infty \text{Var} \left(\dfrac{U_j}{j} \mid \mathcal{A}_j \right) &= \sum_{j=m_0+1}^{\infty} \dfrac{(\ell - 1)\pi_j (1-(\ell-1)\pi_j)}{j^2}
& \leq \sum_{j=m_0+1}^{\infty} \dfrac{(\ell - 1)\pi_1 (1-(\ell-1)\pi_1)}{j^2}.
\end{align*}
As the right hand side is a non-random quantity, we get,
\begin{align*}
\sum_{j=m_0+1}^\infty \text{Var} \left(\dfrac{U_j}{j} \right) \leq \sum_{j=m_0+1}^{\infty} \dfrac{(\ell - 1)\pi_1 (1-(\ell-1)\pi_1)}{j^2} < \infty.
\end{align*}
Therefore, we have that $\sum_{m_0+1}^\infty ((U_j - (\ell -1)\pi_j)/j)$ converges almost surely. It then follows by Kronecker's lemma that,
\begin{align*}
\dfrac{1}{n} \sum_{j=1}^n (U_j - (\ell-1)\pi_j) \overset{\text{a.s.}}{\rightarrow} 0.
\end{align*}
We know that $\tau_j \overset{\text{a.s.}}{\rightarrow} \infty$ as $j \rightarrow \infty$ using Assumption 3 as shown in the proof of Theorem 2 of the paper. Hence, $\pi_{\tau_j} \rightarrow 0$ almost surely, as $j \rightarrow \infty$ (the speed depending on the delay times). Thus, we will have ${1}/{n} \sum_{j=1}^n (\ell-1)\pi_j \rightarrow 0$ since $\pi_j\rightarrow 0$ a.s., as $j\rightarrow \infty$. Hence, ${1}/{n} \sum_{j=1}^n U_j\rightarrow 0$ a.s., as $n\rightarrow \infty$.

To show that $R_n(\delta_\pi) \overset{\text{a.s.}}{\rightarrow} 1$, it remains to show that
\begin{align*}
 \dfrac{\sum_{j=1}^n f_{\hat{i}_j}(X_j)}{\sum_{j=1}^n f^*(X_j)} \overset{\text{a.s.}}{\rightarrow} 1 \ \text{or equivalently,}\ \dfrac{\sum_{j=1}^n (f_{\hat{i}_j}(X_j) - f^*(X_j))}{\sum_{j=1}^n f^*(X_j)} \overset{\text{a.s.}}{\rightarrow} 0.
 \end{align*} 
Given the observed reward timings $\{t_j: t_j \leq n , 1\leq j \leq n\}$, let $\sigma_j = \min\{\bar{n}: \sum_{k=m_0+1}^{\bar{n}} I(t_k \leq N) \geq j \}$, that is, $\sigma_j$ is the time index where the $j$\ts{th} reward is observed.
 By the definition of $\hat{i}_j$, for $j \geq m_0 + 1$, $\hat{f}_{\hat{i}_j,\sigma_j}(X_j) \geq \hat{f}_{i^*(X_j),\sigma_j}(X_j)$ and thus,
\begin{align*}
f_{\hat{i}_j}(X_j) - f^*(X_j) &= f_{\hat{i}_j}(X_j) - \hat{f}_{\hat{i}_j,\sigma_j}(X_j) + \hat{f}_{\hat{i}_j, \sigma_j}(X_j) - \hat{f}_{i^*(X_j),\sigma_j}(X_j)\\
&\quad \quad \quad + \hat{f}_{i^*(X_j),\sigma_j}(X_j) - f^*(X_j)\\
&\geq f_{\hat{i}_j}(X_j) - \hat{f}_{\hat{i}_j, \sigma_j}(X_j) + \hat{f}_{i^*(X_j),\sigma_j}(X_j) - f_{i^*(X_j)}(X_j)\\
&\geq -2 \sup_{1\leq i \leq \ell} ||\hat{f}_{i,\sigma_j} - f_i||_\infty.
\end{align*}
For $1\leq j \leq m_0$, we have $f_{\hat{i}_j}(X_j) - f^*(X_j) \geq -A$.  Based on Assumption 1, $||\hat{f}_{i,\sigma_j} - f_i||_\infty \overset{\text{a.s.}}{\rightarrow} 0$ as $j \rightarrow \infty$ for each $i$, and thus $\sup_{1\leq i \leq \ell} || \hat{f}_{i,\sigma_j} - f_i||_\infty \overset{\text{a.s.}}{\rightarrow} 0$. Then it follows that, for $n > m_0$,
\begin{align*}
&\dfrac{\sum_{j=1}^n (f_{\hat{i}_j}(X_j) - f^*(X_j))}{\sum_{j=1}^n f^*(X_j)} \\
&\quad \quad \geq \dfrac{-Am_0/n - (2/n)\sum_{j=m_0+1}^n \sup_{1\leq i \leq \ell} ||\hat{f}_{i,\sigma_j} - f_i||_\infty}{(1/n)\sum_{j=1}^n f^*(X_j)}.
\end{align*}
The right hand side converges to 0 almost surely and hence the conclusion follows. 
\end{proof}

Next, we recall some important definitions and inequalities that will be used in the proof for Theorem 3.\\
\begin{definition}
Let $h_{\tau_n}$ denote the bandwidth, where $h_{\tau_n} \rightarrow 0$ almost surely as $n \rightarrow \infty$. For each arm $i$, the Nadaraya-Watson estimator of $f_i(x)$ is defined as,
\begin{align}
\hat{f}_{i,n+1}(x) = \dfrac{\sum_{j \in J_{i,n+1}}Y_{i,j} K \left(\frac{x - X_j}{h_{\tau_n}} \right)}{\sum_{j \in J_{i,n+1}}K \left(\frac{x - X_j}{h_{\tau_n}} \right)}. \label{NadarayaWatsonEst}
\end{align}
\end{definition}

\begin{definition} \label{mod_of_continuity}
Let $x_1, x_2 \in [0,1]^d$. Then $w(h;f)$ denotes a modulus of continuity defined by,
$
w(h;f) = \sup\{|f(x_1) - f(x_2)|: |x_{1k} - x_{2k}| \leq h \ \text{for all}\ 1 \leq k \leq d\}.$

\end{definition}

\subsection{\bf An inequality for Bernoulli trials.}\label{A.2}
For $1\leq j \leq n$, let $\tilde{W}_j$ be Bernoulli random variables, which are not necessarily independent. Assume that the conditional probability of success for $\tilde{W}_j$ given the previous observations is lower bounded by $\beta_j$, that is,
\begin{align*}
P(\tilde{W}_j = 1|\tilde{W}_i, 1 \leq i \leq j-1) \geq \beta_j \ \text{a.s.},
\end{align*}
for all $1\leq j\leq n$. 
Appylying the extended Bernstein's inequality as described in \cite{qian2016kernel}, we have
\begin{align}
P\left(\sum_{j=1}^n \tilde{W}_j \leq \left(\sum_{j=1}^n \beta_j\right)/2 \right) \leq \exp\left(- \dfrac{3\sum_{j=1}^n \beta_j}{28} \right).  \label{binomial_inequality}
\end{align}

\subsection{Proof for consistency using Kernel Regression}\label{proofconsistency_kernel_AppB}
Recall, $J_{i,n+1} = \{j: I_j = i, t_j \leq n,1 \leq j \leq n\}$ and $M_{i,n+1}$ is the size of $J_{i,n+1}$, $\mathcal{A} = \{j: t_j \leq n\}$ and $\tau_n = \sum_{j=1}^n I(t_j \leq n)$.\\
\begin{lemma} \label{Lemma_kernel_theorem}
Under the setting of the kernel estimation in Section 5.2 of the paper, let $A \subset [0,1]^d$ be a hypercube with side-width $h$. For a given arm $i$, if Assumptions $4,5,7$ and $8$ are satisfied, then for any $\epsilon >0$,
\begin{align*}
P_{\mathcal{A}_n, X^n}& \left(\sup_A \sum_{j \in J_{i,n+1}} \epsilon_j K \left(\dfrac{x - X_j}{h_{\tau_n}}\right) > \dfrac{\tau_n \epsilon}{1 - 1/\sqrt{2}} \right) \\
& \leq \exp \left(- \dfrac{\tau_n \epsilon^2}{4 c_4^2 v^2} \right) + \exp \left(- \dfrac{\tau_n \epsilon}{4 c_4 c} \right) + \sum_{k=1}^\infty 2^{kd} \exp \left(- \dfrac{2^k \tau_n \epsilon^2}{\lambda^2 v^2} \right)+ \sum_{k=1}^\infty 2^{kd} \exp \left(- \dfrac{2^{k/2} \tau_n \epsilon}{2 \lambda c} \right),
\end{align*} 
where $P_{\mathcal{A}_n, X^n}$ denotes conditional probability given $\mathcal{A}_n = \{j: t_j \leq n\}$ and $X^n = \{X_1, \hdots, X_n\}$. 
\end{lemma}
\begin{proof}
The proof of this lemma follows exactly from the analogous lemma but without delays in \cite{qian2016kernel}. The results follow because we condition on $\mathcal{A}_n$, and given $\mathcal{A}_n$, $\tau_n$ is a known quantity which plays the role of $n$ in the no-delay situation as in \cite{qian2016kernel}. 
\end{proof}
Next, we restate Theorem 3 from the paper and provide a proof.
\begin{theorem}\label{theorem 3 proof}
Suppose Assumptions 2-8 are satisfied.
\begin{enumerate}
\item If $\{h_{\tau_n}\}$ and $\{\pi_{\tau_n}\}$ are chosen to satisfy, \begin{align}
\dfrac{q(n)h_{q(n)}^{2d} \pi_{n}^4}{\log{n}} \rightarrow \infty,\label{ConditionEta1Thm3}
\end{align}
then the Nadaraya-Watson estimator defined in \eqref{NadarayaWatsonEst} is strongly consistent in $L_\infty$ norm for strategy $\eta_1$.
  \item If $\{h_{\tau_n}\}$ and $\{\pi_{\tau_n}\}$ are chosen to satisfy, \begin{align}
\dfrac{q(n)h_{q(n)}^{2d} \pi_{q(n)}^4}{\log{n}} \rightarrow \infty, \label{ConditionEta2Thm3}
\end{align}
then the Nadaraya-Watson estimator defined in \eqref{NadarayaWatsonEst} is strongly consistent in $L_\infty$ norm for strategy $\eta_2$.
\end{enumerate}
\end{theorem}
\begin{proof}[Proof of Theorem 3]
Here, we prove the result for strategy $\eta_2$ and discuss how the proof for strategy $\eta_1$ follows similarly. For each $x \in [0,1]^d$,
\begin{align*}
|\hat{f}_{i,n+1} - f_i(x)| &= \left|\dfrac{\sum_{j \in J_{i,n+1}}Y_{i,j}K \left(\dfrac{x-X_j}{h_{\tau_n}} \right)}{\sum_{j \in J_{i,n+1}} K \left(\dfrac{x - X_j}{h_{\tau_n}} \right)} - f_i(x) \right|\\
&= \left|\dfrac{\sum_{j \in J_{i,n+1}}(f_i(X_j) + \epsilon_j) K \left(\dfrac{x-X_j}{h_{\tau_n}} \right)}{\sum_{j \in J_{i,n+1}} K \left(\dfrac{x - X_j}{h_{\tau_n}} \right)} - f_i(x) \right|\\
&= \left|\dfrac{\sum_{j \in J_{i,n+1}}(f_i(X_j) - f_i(x)) K \left(\dfrac{x-X_j}{h_{\tau_n}} \right)}{\sum_{j \in J_{i,n+1}} K \left(\dfrac{x - X_j}{h_{\tau_n}} \right)} + \dfrac{\sum_{j \in J_{i,n+1}} \epsilon_j K \left(\dfrac{x-X_j}{h_{\tau_n}} \right)}{\sum_{j \in J_{i,n+1}} K \left(\dfrac{x - X_j}{h_{\tau_n}} \right)} \right|\\
& \leq \sup_{x,y: ||x-y||_\infty \leq Lh_{\tau_n}} |f_i(x) - f_i(y)| + \left| \dfrac{\frac{1}{M_{i,n+1}h_{\tau_n}^d}\sum_{j \in J_{i,n+1}} \epsilon_j K \left(\dfrac{x-X_j}{h_{\tau_n}} \right)}{\frac{1}{M_{i,n+1}h_{\tau_n}^d}\sum_{j \in J_{i,n+1}} K \left(\dfrac{x - X_j}{h_{\tau_n}} \right)} \right|,
\end{align*}
where the last inequality follows from the bounded support assumption of kernel function $K(\cdot)$. It was shown in the proof of Theorem 2 of the paper, $\tau_n \overset{\text{a.s.}}{\rightarrow} \infty$ as $n \rightarrow \infty$. Thus, by uniform continuity of the function $f_i$,
\begin{align*}
\lim_{n\rightarrow \infty} \sup_{x,y: ||x-y||_\infty \leq Lh_{\tau_n}} |f_i(x) - f_i(y)| = 0, \ \ \text{almost surely.}
\end{align*}
Therefore we only need,
\begin{align}
\sup_{x \in [0,1]^d} \left|\dfrac{\frac{1}{M_{i,n+1}h_{\tau_n}^d}\sum_{j \in J_{i,n+1}} \epsilon_j K \left(\dfrac{x-X_j}{h_{\tau_n}} \right)}{\frac{1}{M_{i,n+1}h_{\tau_n}^d}\sum_{j \in J_{i,n+1}} K \left(\dfrac{x - X_j}{h_{\tau_n}} \right)} \right| \overset{\text{a.s.}}{\rightarrow} 0 \ \text{as} \ n \rightarrow \infty. \label{ToShowConsistencyResult}
\end{align}
We first show that,
\begin{align}
\inf_{x \in [0,1]^d} \dfrac{1}{M_{i,n+1} h_{\tau_n}^d} \sum_{j \in J_{i,n+1}} K \left(\dfrac{x - X_j}{h_{\tau_n}} \right) > \dfrac{c_3\underbar{c} L_1^d \pi_{\tau_n}}{2}, \label{denominatorPostiveCondition}
\end{align}
almost surely for large enough $n$. Indeed, for each $n \geq m_0 + 1$, given $\tau_n$, we can partition the unit cube $[0,1]^d$ into $\tilde{B}$ bins with bin width $L_1h_{\tau_n}$ such that $\tilde{B} \leq 1/(L_1h_{\tau_n})^d$. We denote these bins  by $\tilde{A}_1, \tilde{A}_2, \hdots, \tilde{A}_{\tilde{B}}$. Let $\sigma_t = \inf\{\tilde{n}: \sum_{j=1}^{\tilde{n}} I(t_j \leq n) \geq t\}$. Given an arm $i$ and $1\leq k \leq \tilde{B}$, for every $x \in \tilde{A}_{k}$, given $\tau_n$ we have that,
\begin{align*}
\sum_{j \in J_{i,n+1}} K \left(\dfrac{x - X_j}{h_{\tau_n}}\right) &= \sum_{t=1}^{\tau_n} I(I_{\sigma_t} = i) K \left(\dfrac{x - X_{\sigma_t}}{h_{\tau_n}}\right)\\
& \geq \sum_{t=1}^{\tau_n} I(I_{\sigma_t} = i, X_{\sigma_t} \in \tilde{A}_k) K \left(\dfrac{x - X_{\sigma_t}}{h_{\tau_n}}\right)\\
& \geq c_3 \sum_{j=1}^{\tau_n} I(I_{\sigma_t} = i, X_{\sigma_t} \in \tilde{A}_k),
\end{align*}
where the last inequality follows from Assumption 8 (boundedness of kernels) in the paper. Therefore,
\begin{align*}
P_{\mathcal{A}_n, X^n} & \left(\inf_{x \in \tilde{A}_k} \dfrac{1}{M_{i,n+1} h_{\tau_n}^d} \sum_{j \in J_{i,n+1}} K \left(\dfrac{x-X_j}{h_{\tau_n}} \right) \leq \dfrac{c_3 \underbar{c} L_1^d \pi_{\tau_n}}{2}\right)\\
& \leq P_{\mathcal{A}_n, X^n} \left(\inf_{x \in \tilde{A}_k} \dfrac{1}{\tau_{n} h_{\tau_n}^d}\sum_{j \in J_{i,n+1}} K \left(\dfrac{x-X_j}{h_{\tau_n}} \right) \leq \dfrac{c_3 \underbar{c} L_1^d \pi_{\tau_n}}{2} \right)\\
& \leq P_{\mathcal{A}_n, X^n} \left(\dfrac{c_3}{\tau_{n} h_{\tau_n}^d} \sum_{t = 1}^{\tau_n} I(I_{\sigma_t} = i, X_{\sigma_t} \in \tilde{A}_k)  \leq \dfrac{c_3 \underbar{c} L_1^d \pi_{\tau_n}}{2} \right)\\
& \leq P_{\mathcal{A}_n, X^n}  \left( \sum_{t = 1}^{\tau_n} I(I_{\sigma_t} = i, X_{\sigma_t} \in \tilde{A}_k) \leq \dfrac{\underbar{c} \tau_n (L_1 h_{\tau_n})^d \pi_{\tau_n}}{2} \right).
\end{align*}
Note that, $P_{\mathcal{A}_n, X^n}(I_{\sigma_t} = i, X_{\sigma_t} \in \tilde{A}_k) \geq \underbar{c} (L_1 h_{\tau_n})^d \pi_{\tau_n}$ by independence of arms chosen and covariates (Assumption 6), for all $1 \leq t \leq n$.
\begin{align*}
 P_{\mathcal{A}_n, X^n} \left(\sum_{t=1}^{\tau_n} I(I_{\sigma_t} = i, X_{\sigma_t} \in \tilde{A}_k) \leq \dfrac{\underbar{c}\tau_n (L_1h_{\tau_n})^d \pi_{\tau_n}}{2} \right) &\leq \exp \left(- \dfrac{3 \underbar{c} \tau_n (L_1h_{\tau_n})^d \pi_{\tau_n}}{28} \right).
 \end{align*} 
 Therefore we get that,
 \begin{align}
 P_{\mathcal{A}_n,X^n} \left(\inf_{x \in \tilde{A}_k} \dfrac{1}{M_{i,n+1} h_{\tau_n}^d} \sum_{j \in J_{i,n+1}} K \left(\dfrac{x-X_j}{h_{\tau_n}} \right) \leq \dfrac{c_3 \underbar{c} L_1^d \pi_{\tau_n}}{2} \right) \leq \exp \left(- \dfrac{3 \underbar{c} \tau_n (L_1h_{\tau_n})^d \pi_{\tau_n}}{28} \right). \label{conditiononAnXn}
 \end{align}
 Now consider,
 \begin{align*}
 P& \left( \inf_{x \in \tilde{A}_k} \dfrac{1}{M_{i,n+1} h_{\tau_n}^d} \sum_{j \in J_{i,n+1}} K \left(\dfrac{x-X_j}{h_{\tau_n}} \right) \leq \dfrac{c_3 \underbar{c} L_1^d \pi_{\tau_n}}{2}\right)\\
& = P\left( \inf_{x \in \tilde{A}_k} \dfrac{1}{M_{i,n+1} h_{\tau_n}^d} \sum_{j \in J_{i,n+1}} K \left(\dfrac{x-X_j}{h_{\tau_n}} \right) \leq \dfrac{c_3 \underbar{c} L_1^d \pi_{\tau_n}}{2}, \tau_n > \dfrac{\E(\tau_n)}{2}\right) \\
& \quad \quad + P \left( \inf_{x \in \tilde{A}_k} \dfrac{1}{M_{i,n+1} h_{\tau_n}^d} \sum_{j \in J_{i,n+1}} K \left(\dfrac{x-X_j}{h_{\tau_n}} \right) \leq \dfrac{c_3 \underbar{c} L_1^d \pi_{\tau_n}}{2}, \tau_n \leq \dfrac{\E(\tau_n)}{2}\right)\\
& \leq EP_{\mathcal{A}_n,X^n} \left( \inf_{x \in \tilde{A}_k} \dfrac{1}{M_{i,n+1} h_{\tau_n}^d} \sum_{j \in J_{i,n+1}} K \left(\dfrac{x-X_j}{h_{\tau_n}} \right) \leq \dfrac{c_3 \underbar{c} L_1^d \pi_{\tau_n}}{2}, \tau_n > \dfrac{\E(\tau_n)}{2}\right) + P \left(\tau_n \leq \dfrac{\E(\tau_n)}{2} \right)\\
& \leq \exp \left(- \dfrac{3\underbar{c}(L_1h_{\tau_n})^d \pi_{\tau_n} (\E(\tau_n))}{56} \right) + \exp \left(-\dfrac{3 \E(\tau_n)}{28} \right),
 \end{align*}
 where the last inequality followed from \eqref{conditiononAnXn} and the Bernstein's inequality. Hence,
 \begin{align*}
 P& \left(\inf_{x \in [0,1]^d} \dfrac{1}{M_{i,n+1} h_{\tau_n}^d} \sum_{j \in J_{i,n+1}} K \left(\dfrac{x - X_j}{h_{\tau_n}} \right) \leq \dfrac{c_3 \underbar{c} L_1^d \pi_{\tau_n}}{2} \right)\\
& \leq \sum_{k = 1}^{\tilde{B}} P \left(\inf_{\tilde{A}_k} \dfrac{1}{M_{i,n+1} h_{\tau_n}^d} \sum_{j \in J_{i,n+1}} K \left(\dfrac{x - X_j}{h_{\tau_n}} \right) \leq \dfrac{c_3 \underbar{c} L_1^d \pi_{\tau_n}}{2} \right)\\
& \leq \tilde{B} \left(\exp \left(- \dfrac{3\underbar{c}(L_1h_{\tau_n})^d \pi_{\tau_n} (\E(\tau_n))}{56} \right) + \exp \left(-\dfrac{3 \E(\tau_n)}{28} \right) \right)\\
& \leq \tilde{B} \left(\exp \left(- \dfrac{3\tilde{\underbar{c}}(L_1h_{q(n)})^d \pi_{q(n)} (q(n))}{56} \right) + \exp \left(-\dfrac{3 a_1 q(n)}{28} \right) \right),
 \end{align*}
 where the last inequality follows from Assumption 3 and \eqref{ConditionEta2Thm3}.
 Here, $\tilde{\underbar{c}}$ and $a_1$ are constants due to the use of Assumption 3, which says that $\E(\tau_n) \geq a_1 q(n)$ for some constant $a_1 > 0$. Also, the same condition $\frac{q(n)h_{q(n)}^{2d} \pi_{q(n)}^4}{\log{n}} \rightarrow \infty$ ensures that the RHS above is summable, and by Borel-Cantelli Lemma, we have \eqref{denominatorPostiveCondition}.  

Now, in order to prove \eqref{ToShowConsistencyResult}, we now need to show that,
\begin{align}
\sup_{x \in [0,1]^d} \left|\dfrac{1}{M_{i,n+1} h_{\tau_n}^d} \sum_{j \in J_{i,n+1}} \epsilon_j K \left(\dfrac{x - X_j}{h_{\tau_n}}\right) \right| = o(\pi_{\tau_n}), \ \text{almost surely.} \label{AlmostSurelyNum}
\end{align}
For each $n > m_0+1$, we can partition the unit cube $[0,1]^d$ into $B$ bins with bin length $h_{\tau_n}$ such that $B \leq 1/h_{\tau_n}^d$. We denote these bins by $A_1, A_2,\hdots, A_B$. Then given $\epsilon > 0$, consider,
\begin{align}
P_{\mathcal{A}_n, X^n} &\left(\sup_{x \in [0,1]^d} \left|\dfrac{1}{M_{i,n+1}h_{\tau_n}^d} \sum_{j \in J_{i,n+1}} \epsilon_j K \left(\dfrac{x - X_j}{h_{\tau_n}} \right) \right| > \pi_{\tau_n} \epsilon \right)\nonumber\\
&\leq B \max_{1\leq k \leq B} P_{\mathcal{A}_n, X^n} \left(\sup_{x \in A_k} \left|\dfrac{1}{M_{i,n+1}h_{\tau_n}^d} \sum_{j \in J_{i,n+1}} \epsilon_j K \left(\dfrac{x - X_j}{h_{\tau_n}} \right) \right| > \pi_{\tau_n} \epsilon \right)\nonumber\\
& \leq B \max_{1\leq k \leq B} P_{\mathcal{A}_n, X^n} \left(\sup_{x \in A_k} \left|\dfrac{1}{M_{i,n+1}h_{\tau_n}^d} \sum_{j \in J_{i,n+1}} \epsilon_j K \left(\dfrac{x - X_j}{h_{\tau_n}} \right) \right| > \pi_{\tau_n} \epsilon, \dfrac{M_{i,n+1}}{\tau_n} > \dfrac{\pi_{\tau_n}}{2} \right)\nonumber \\
&\quad \quad + B P_{\mathcal{A}_n,X^n} \left(\dfrac{M_{i,n+1}}{\tau_n}\leq \dfrac{\pi_{\tau_n}}{2} \right)\nonumber \\
&\leq B \max_{1\leq k \leq B} P_{\mathcal{A}_n, X^n} \left(\sup_{x \in A_k} \left|\sum_{j \in J_{i,n+1}} \epsilon_j K \left(\dfrac{x - X_j}{h_{\tau_n}} \right) \right| > \dfrac{\tau_n \pi_{\tau_n}^2 h_{\tau_n}^d \epsilon}{2}\right)+ B P_{\mathcal{A}_n,X^n} \left(\dfrac{M_{i,n+1}}{\tau_n}\leq \dfrac{\pi_{\tau_n}}{2} \right)\nonumber \\
& \leq  B \max_{1\leq k \leq B} P_{\mathcal{A}_n, X^n} \left(\sup_{x \in A_k} \left|\sum_{j \in J_{i,n+1}} \epsilon_j K \left(\dfrac{x - X_j}{h_{\tau_n}} \right) \right| > \dfrac{\tau_n \pi_{\tau_n}^2 h_{\tau_n}^d \epsilon}{2}\right)+ B \exp \left(- \dfrac{3 \tau_n \pi_{\tau_n}}{28} \right), \label{UpperBoundNumConditionAnXn}
\end{align}
where the last inequality follows from \eqref{binomial_inequality}. Note that using Lemma \ref{Lemma_kernel_theorem},
\begin{align}
P_{\mathcal{A}_n, X^n}& \left(\sup_{x \in A_k} \left|\sum_{j \in J_{i,n+1}} \epsilon_j K \left(\dfrac{x - X_j}{h_{\tau_n}} \right) \right| > \dfrac{\tau_n \pi_{\tau_n}^2 h_{\tau_n}^d \epsilon}{2}\right)\nonumber\\
&\leq 2\exp \left(- \dfrac{(\sqrt{2} - 1)^2 \tau_n \pi_{\tau_n}^4 h_{\tau_n}^{2d} \epsilon^2}{32 c_4^2 v^2} \right) + 2\exp \left(- \dfrac{(\sqrt{2} - 1) \tau_n \pi_{\tau_n}^2 h_{\tau_n}^{d} \epsilon}{8 \sqrt{2} c_4 c} \right)\nonumber\\
&\quad  + 2\sum_{k=1}^\infty 2^{kd} \exp \left(- \dfrac{2^k (\sqrt{2} - 1)^2 \tau_n \pi_{\tau_n}^4 h_{\tau_n}^{2d} \epsilon^2}{8 \lambda^2 v^2} \right)+ 2\sum_{k=1}^\infty 2^{kd} \exp \left(- \dfrac{2^{k/2} (\sqrt{2} - 1) \tau_n \pi_{\tau_n}^2 h_{\tau_n}^{d} \epsilon}{4 \sqrt{2} \lambda c} \right). \label{UpperBoundFirstTermNumConditionOnAnXn}
\end{align}
Using \eqref{UpperBoundNumConditionAnXn} and \eqref{UpperBoundFirstTermNumConditionOnAnXn}, we get that,
\begin{align*}
P_{\mathcal{A}_n, X^n} &\left(\sup_{x \in [0,1]^d} \left|\dfrac{1}{M_{i,n+1}h_{\tau_n}^d} \sum_{j \in J_{i,n+1}} \epsilon_j K \left(\dfrac{x - X_j}{h_{\tau_n}} \right) \right| > \pi_{\tau_n} \epsilon \right) \\
& \leq 2B \exp \left(- \dfrac{(\sqrt{2} - 1)^2 \tau_n \pi_{\tau_n}^4 h_{\tau_n}^{2d} \epsilon^2}{32 c_4^2 v^2} \right) + 2B \exp \left(- \dfrac{(\sqrt{2} - 1) \tau_n \pi_{\tau_n}^2 h_{\tau_n}^{d} \epsilon}{8 \sqrt{2} c_4 c} \right)\nonumber\\
&\quad  + 2B \sum_{k=1}^\infty 2^{kd} \exp \left(- \dfrac{2^k (\sqrt{2} - 1)^2 \tau_n \pi_{\tau_n}^4 h_{\tau_n}^{2d} \epsilon^2}{8 \lambda^2 v^2} \right)+ 2B \sum_{k=1}^\infty 2^{kd} \exp \left(- \dfrac{2^{k/2} (\sqrt{2} - 1) \tau_n \pi_{\tau_n}^2 h_{\tau_n}^{d} \epsilon}{4 \sqrt{2} \lambda c} \right) \\
& \quad \quad + B \exp \left(- \dfrac{3 \tau_n \pi_{\tau_n}}{28} \right).
\end{align*}
Now consider,
\begin{align*}
P& \left(\sup_{x \in [0,1]^d} \left|\dfrac{1}{M_{i,n+1}h_{\tau_n}^d} \sum_{j \in J_{i,n+1}} \epsilon_j K \left(\dfrac{x - X_j}{h_{\tau_n}} \right) \right| > \pi_{\tau_n} \epsilon\right) \\
& \leq \E P_{\mathcal{A}_n, X^n} \left(\sup_{x \in [0,1]^d} \left|\dfrac{1}{M_{i,n+1}h_{\tau_n}^d} \sum_{j \in J_{i,n+1}} \epsilon_j K \left(\dfrac{x - X_j}{h_{\tau_n}} \right) \right| > \pi_{\tau_n} \epsilon, \tau_n > \dfrac{\E(\tau_n)}{2}\right)  + P \left(\tau_n \leq \dfrac{\E(\tau_n)}{2} \right).
\end{align*}
Let $n_e = \lfloor \E(\tau_n)/2 \rfloor$, then using condition \eqref{ConditionEta2Thm3},
\begin{align*}
P& \left(\sup_{x \in [0,1]^d} \left|\dfrac{1}{M_{i,n+1}h_{\tau_n}^d} \sum_{j \in J_{i,n+1}} \epsilon_j K \left(\dfrac{x - X_j}{h_{\tau_n}} \right) \right| > \pi_{\tau_n} \epsilon\right) \\
& \leq 2B\exp \left(- \dfrac{(\sqrt{2} - 1)^2 n_e \pi_{n_e}^4 h_{n_e}^{2d} \epsilon^2}{32 c_4^2 v^2} \right) + 2B\exp \left(- \dfrac{(\sqrt{2} - 1) n_e \pi_{n_e}^2 h_{n_e}^{d} \epsilon}{8 \sqrt{2} c_4 c} \right)\\
&\quad + 2B \sum_{k=1}^\infty 2^{kd} \exp \left(- \dfrac{2^k (\sqrt{2} - 1)^2 n_e \pi_{n_e}^4 h_{n_e}^{2d} \epsilon^2}{8 \lambda^2 v^2} \right)+ 2B\sum_{k=1}^\infty 2^{kd} \exp \left(- \dfrac{2^{k/2} (\sqrt{2} - 1) n_e \pi_{n_e}^2 h_{n_e}^{d} \epsilon}{4 \sqrt{2} \lambda c} \right) \\
& \quad \quad + B \exp \left(- \dfrac{3 n_e \pi_{n_e}}{28} \right) + \exp \left(- \dfrac{3 \E(\tau_n)}{28} \right)
\end{align*}
\begin{align*}
& \leq 2B \exp \left(- \dfrac{(\sqrt{2} - 1)^2 \tilde{a}_1 q(n) \pi_{q(n)}^4 h_{q(n)}^{2d} \epsilon^2}{64 c_4^2 v^2} \right) + 2B \exp \left(- \dfrac{(\sqrt{2} - 1)\tilde{a}_2 q(n) \pi_{q(n)}^2 h_{q(n)}^{d} \epsilon}{16 \sqrt{2} c_4 c} \right)\\
&\quad + 2B \sum_{k=1}^\infty 2^{kd} \exp \left(- \dfrac{2^k (\sqrt{2} - 1)^2 \tilde{a}_1 q(n) \pi_{q(n)}^4 h_{q(n)}^{2d} \epsilon^2}{16 \lambda^2 v^2} \right)\\
& \quad \quad + 2B \sum_{k=1}^\infty 2^{kd} \exp \left(- \dfrac{2^{k/2} (\sqrt{2} - 1) \tilde{a}_2 q(n) \pi_{q(n)}^2 h_{q(n)}^{d} \epsilon}{8 \sqrt{2} \lambda c} \right)\\
& \quad \quad \quad + B \exp \left(- \dfrac{3 \tilde{a}_3  q(n) \pi_{q(n)}}{56} \right) + \exp \left(- \dfrac{3 a_1 q(n)}{28} \right),
\end{align*}
where $\tilde{a}_1$ is a constant that comes from Assumption 3 and the choice of hyperparameter sequence when applied to the constant $a_1$, where $a_1$ is a positive constant such that $\E(\tau_n) \geq a_1 q(n)$, for large enough $n$. Using condition \eqref{ConditionEta2Thm3}, $\frac{q(n)\pi_{q(n)}^4 h_{q(n)}^{2d}}{\log{n}} \rightarrow \infty$, it is easy to see that RHS above is summable. Then, by Borel-Cantelli Lemma we can conclude \eqref{AlmostSurelyNum}, thus proving the theorem. Note, following the same lines of proof, we could prove the strong consistency for $\eta_1$ by just replacing $\pi_{\tau_n}$ with $\pi_n$.
\end{proof}
\begin{figure}[h!]
 \centering
   \includegraphics[scale=0.3]{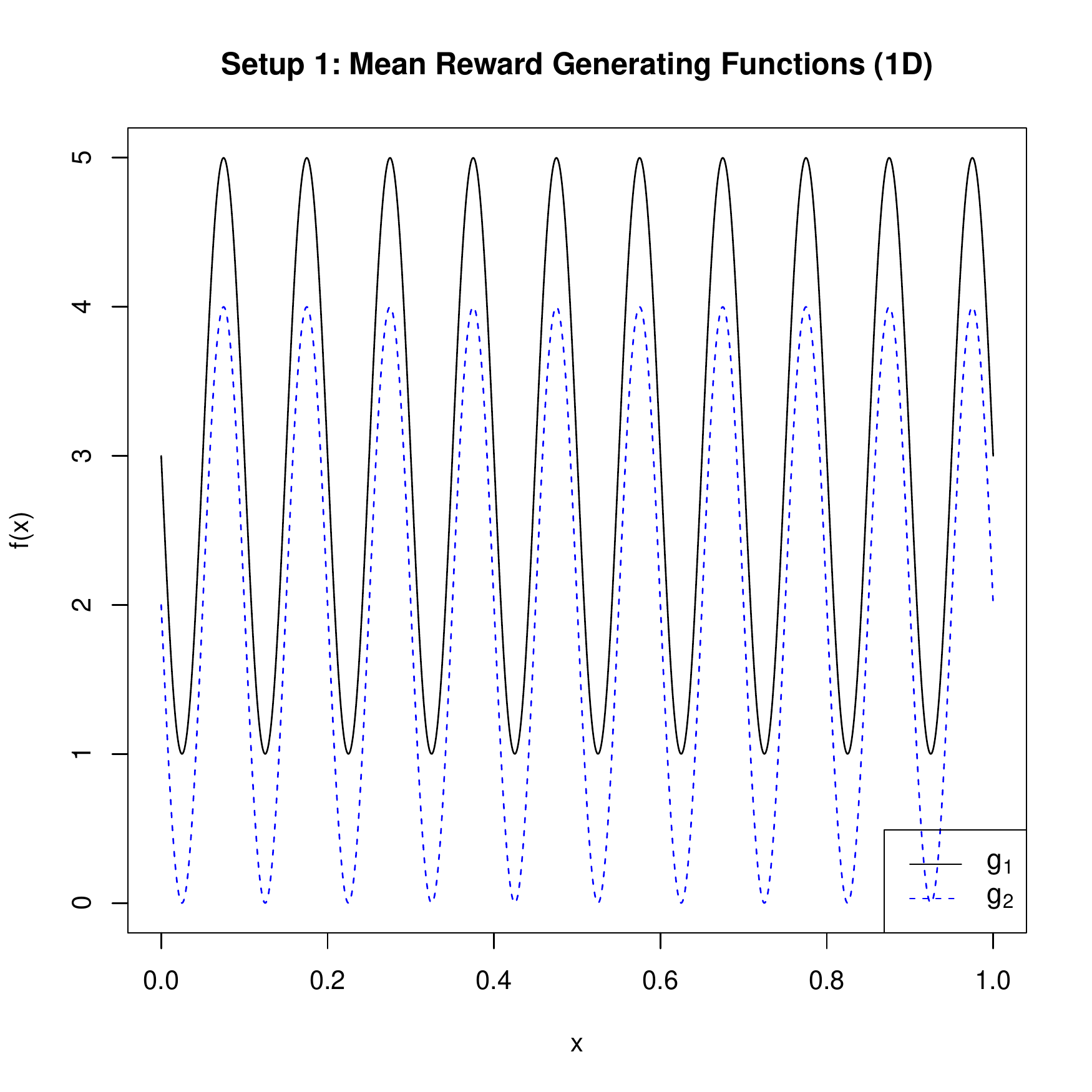}
   \includegraphics[scale=0.3]{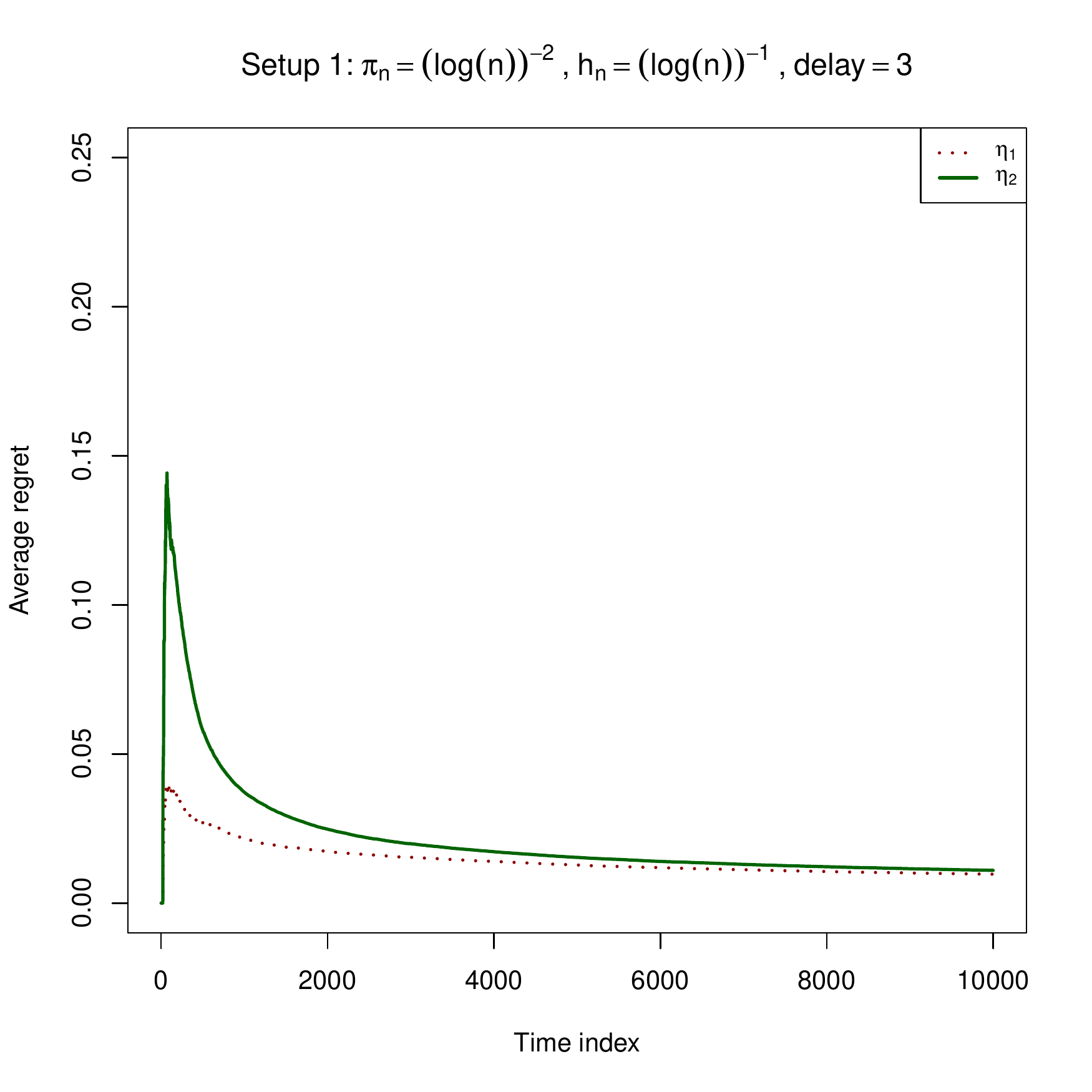}
   \includegraphics[scale=0.3]{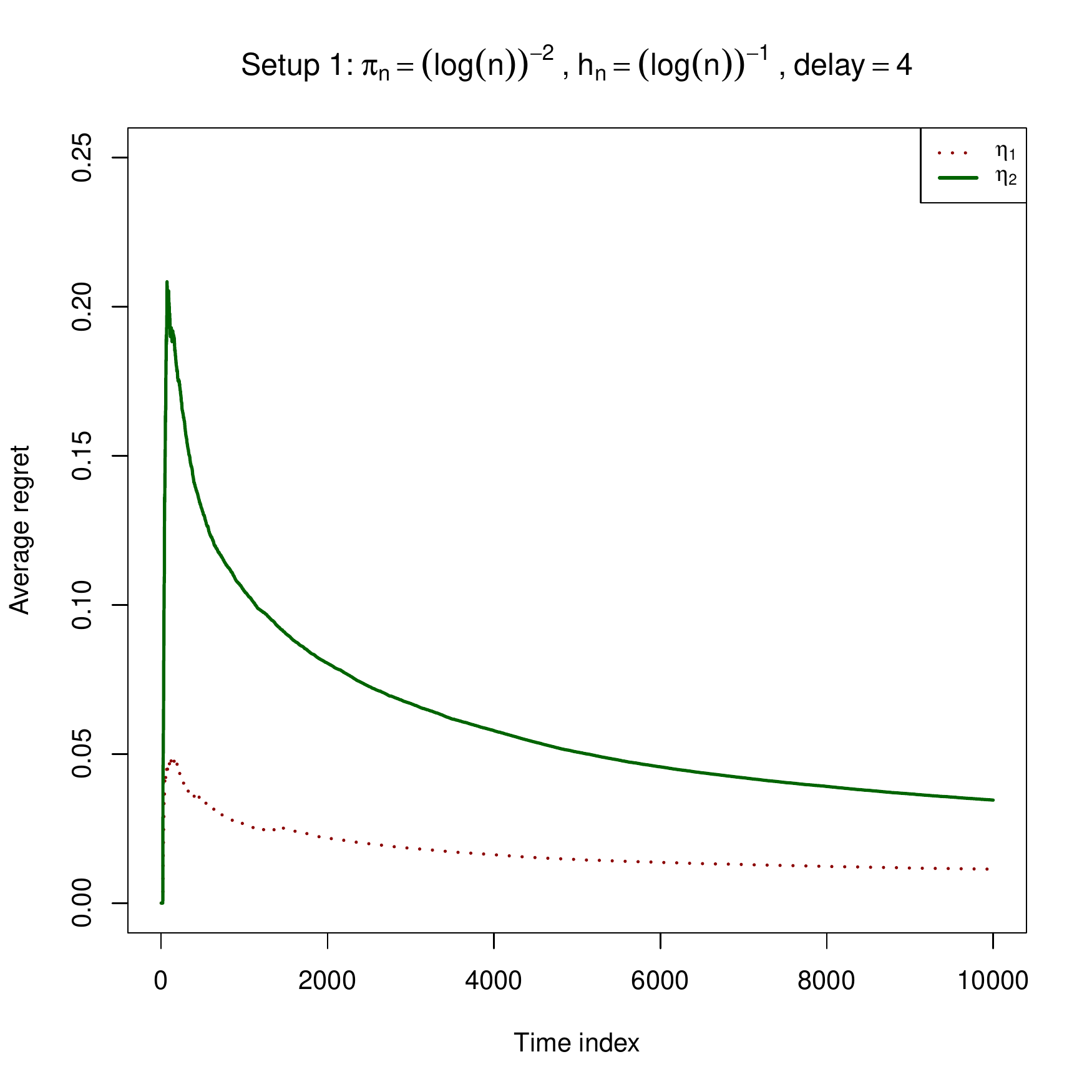}\\
   \includegraphics[scale=0.3]{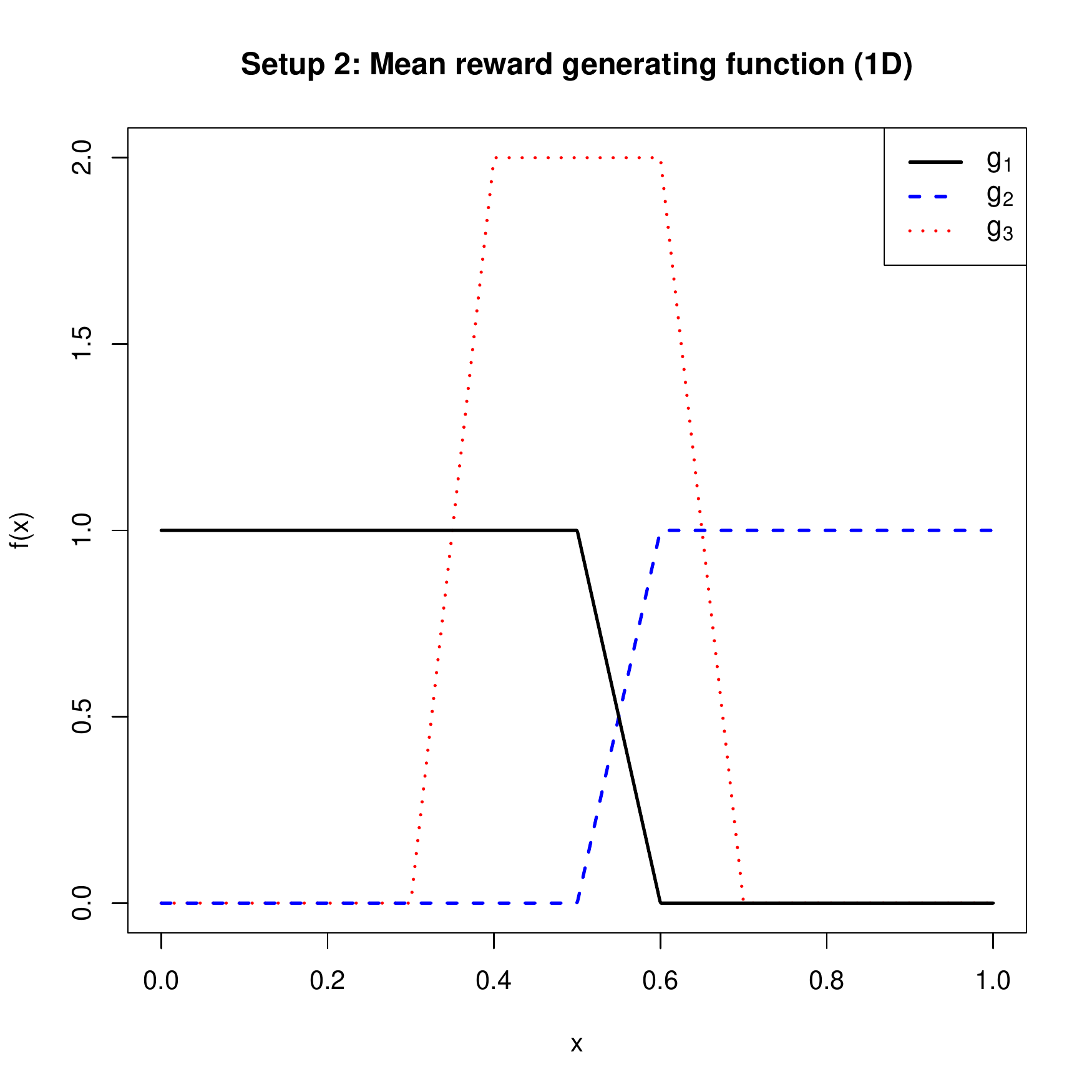}
   \includegraphics[scale=0.3]{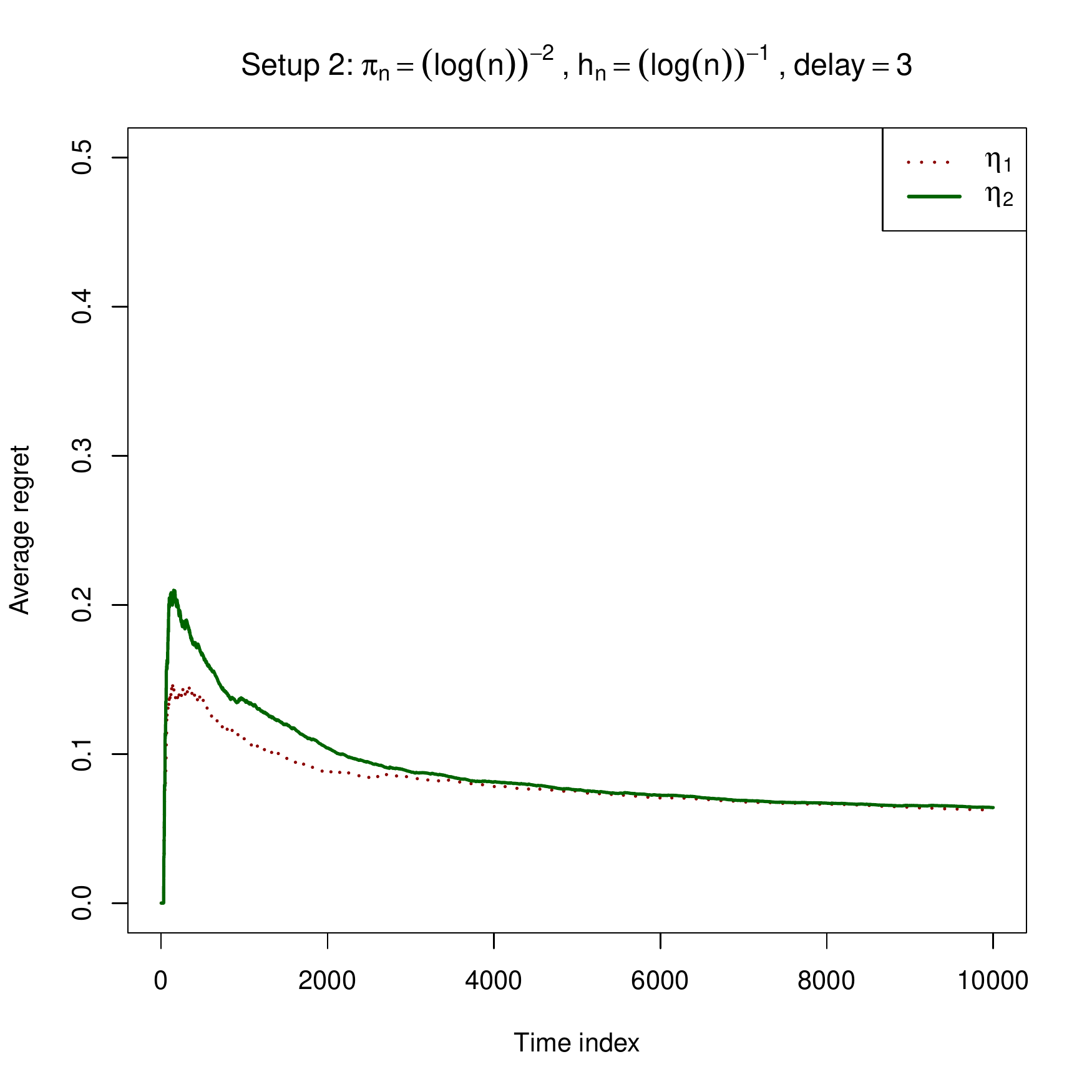}
   \includegraphics[scale=0.3]{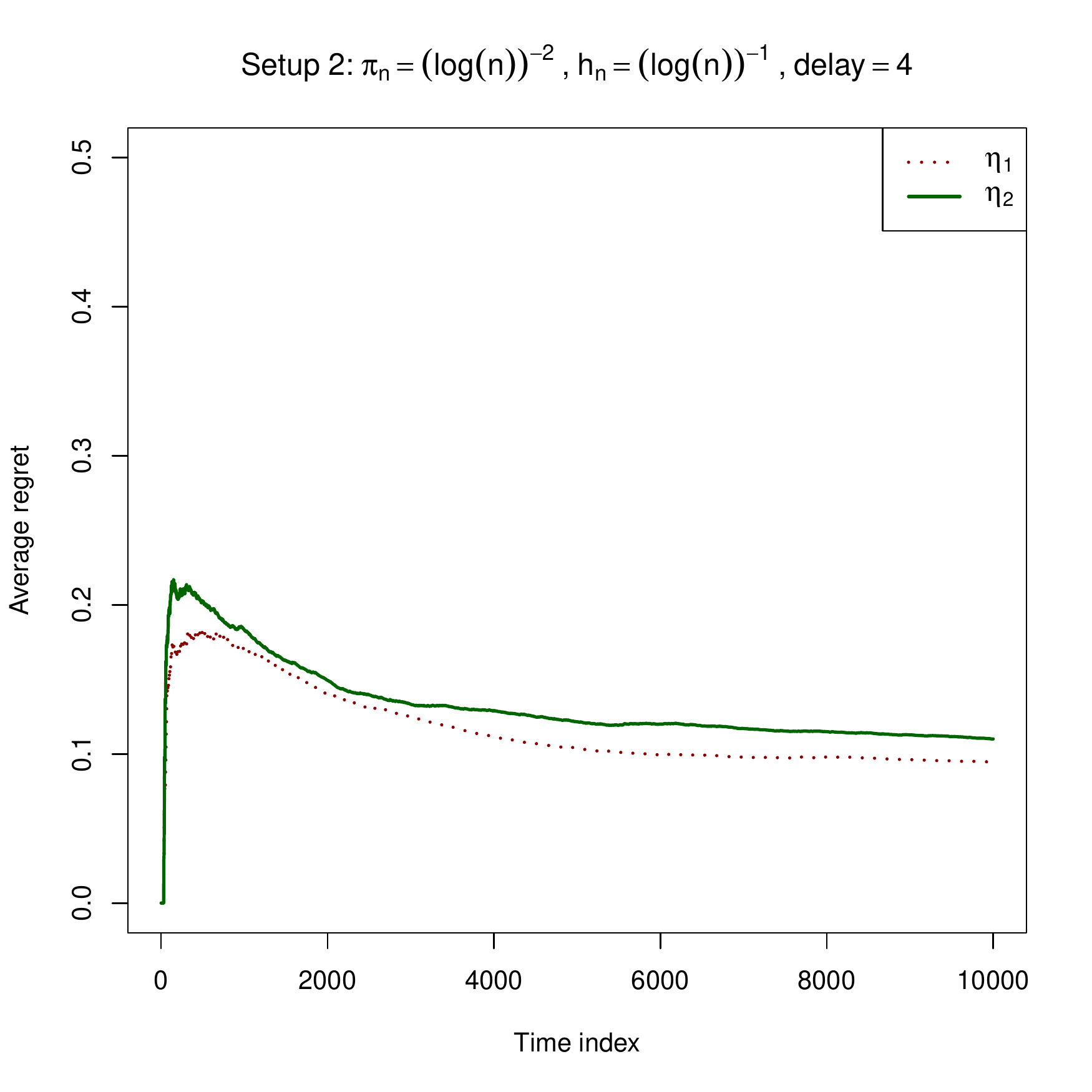}\\
\includegraphics[scale=0.3]{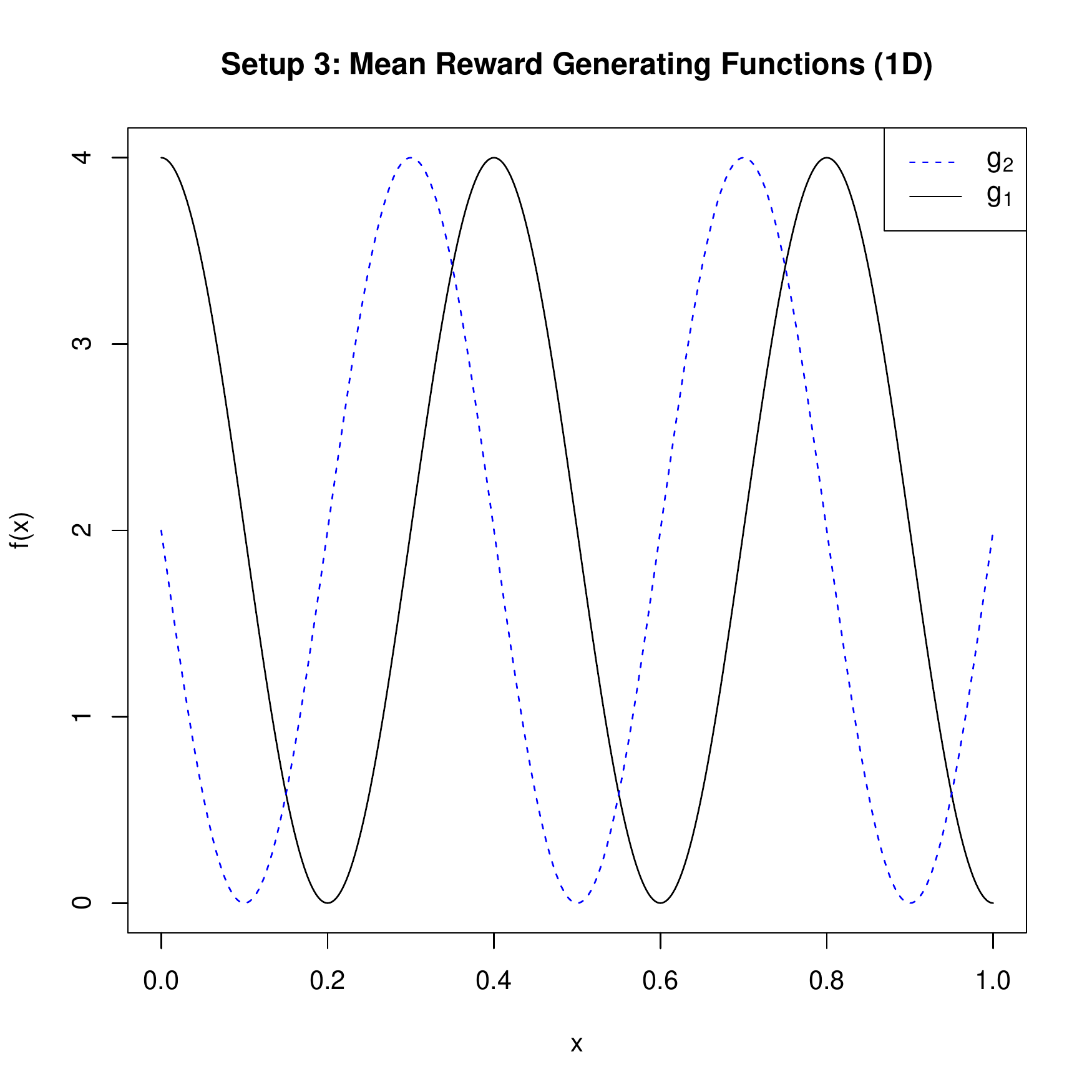}
   \includegraphics[scale=0.3]{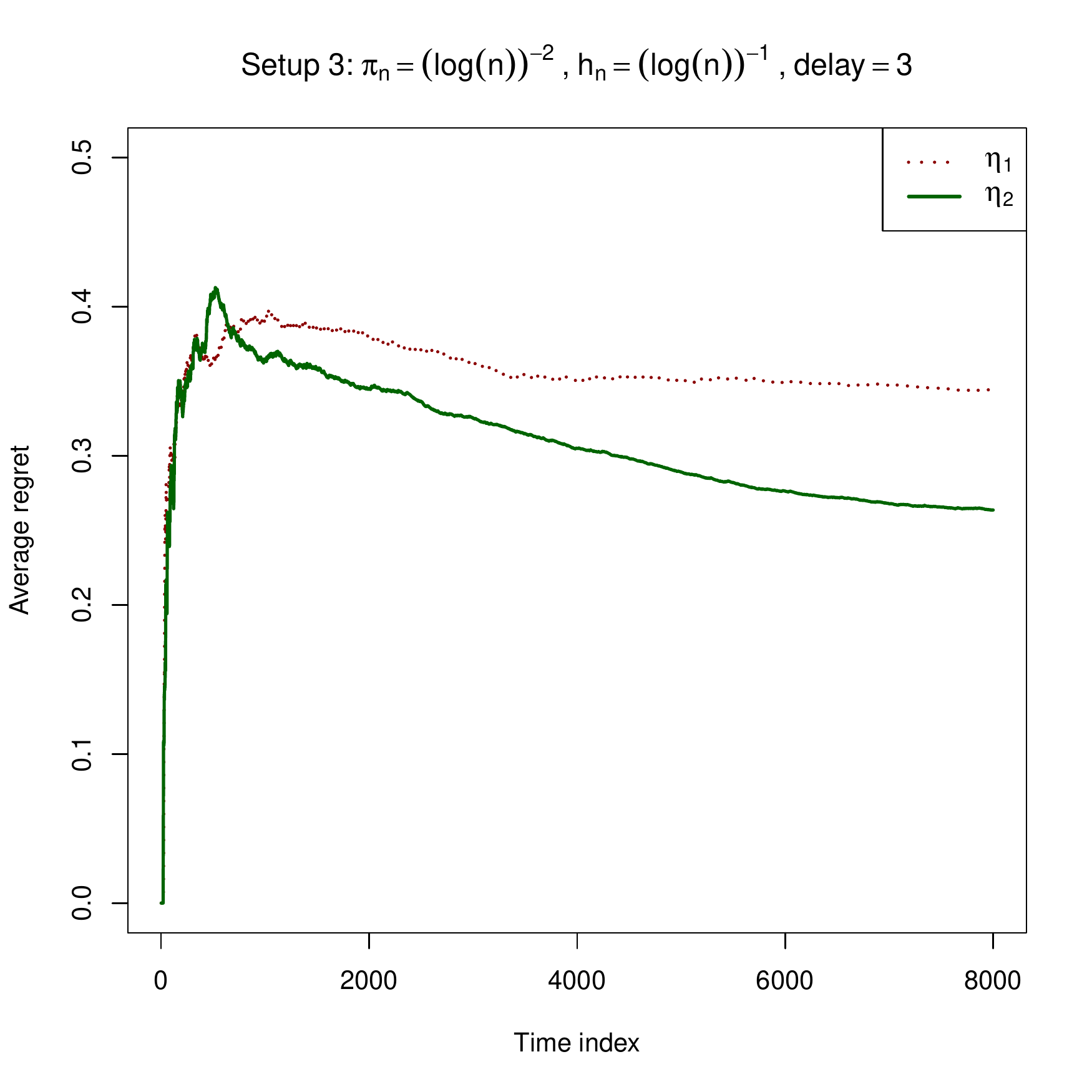}
    \includegraphics[scale=0.3]{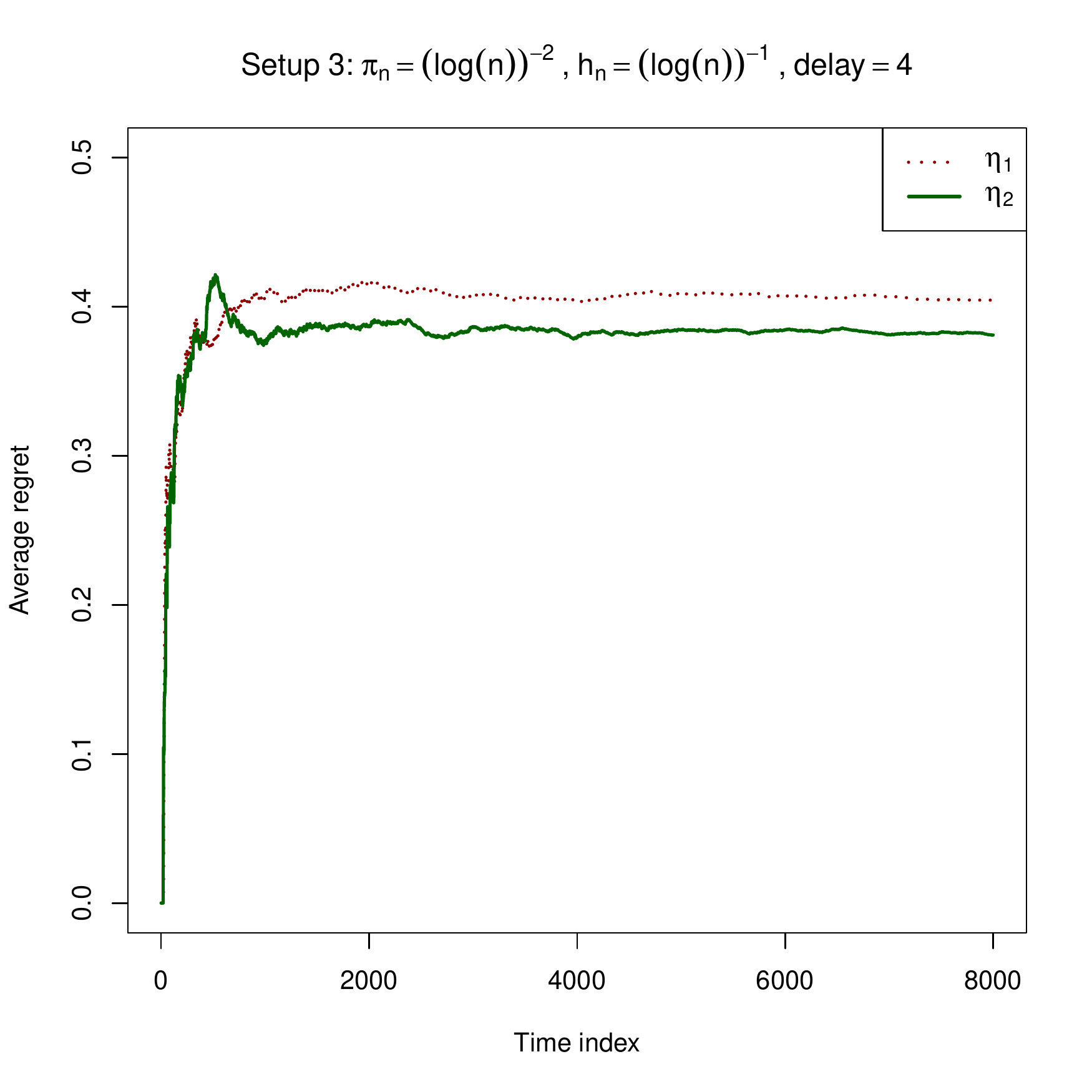}\\
   \includegraphics[scale=0.3]{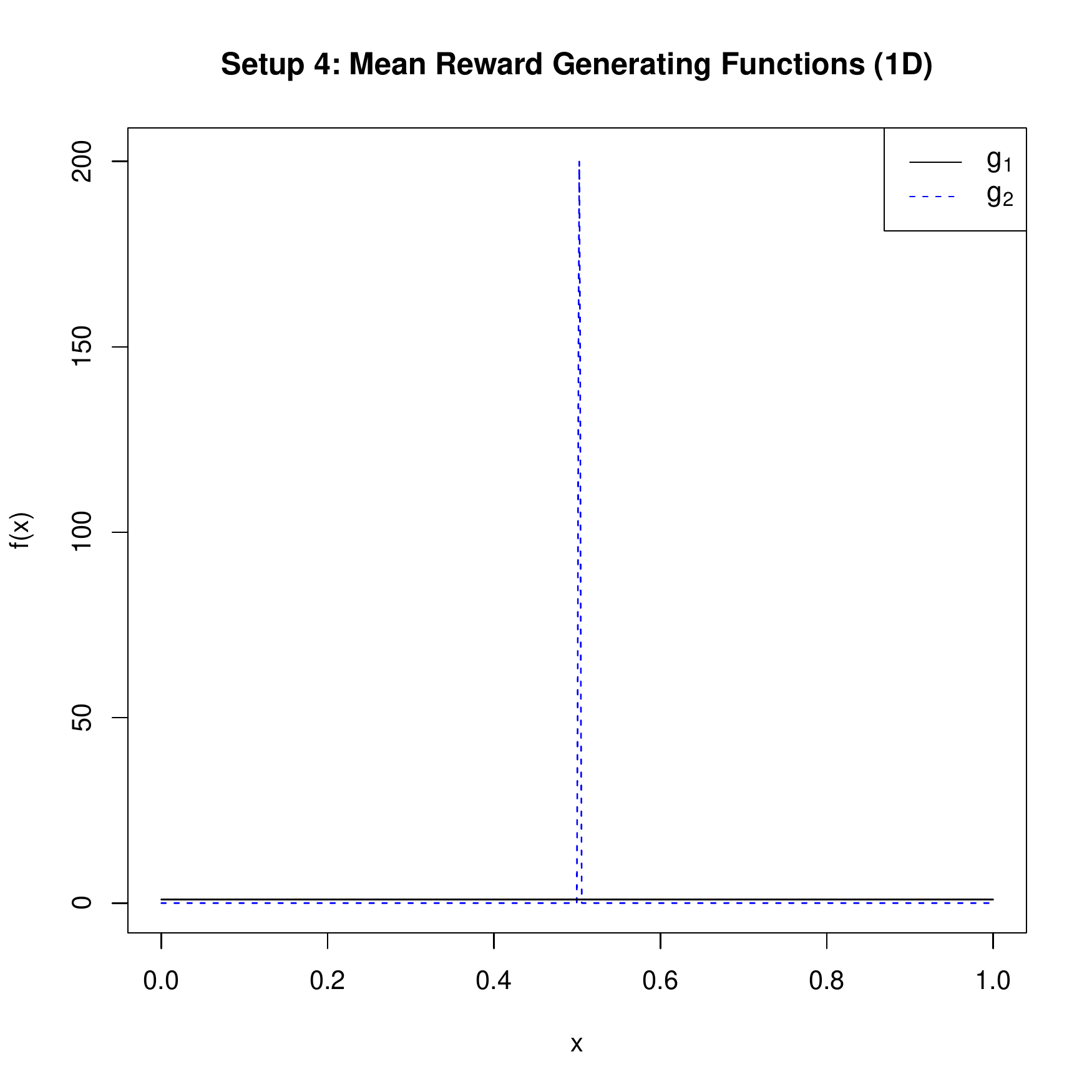}
   \includegraphics[scale=0.3]{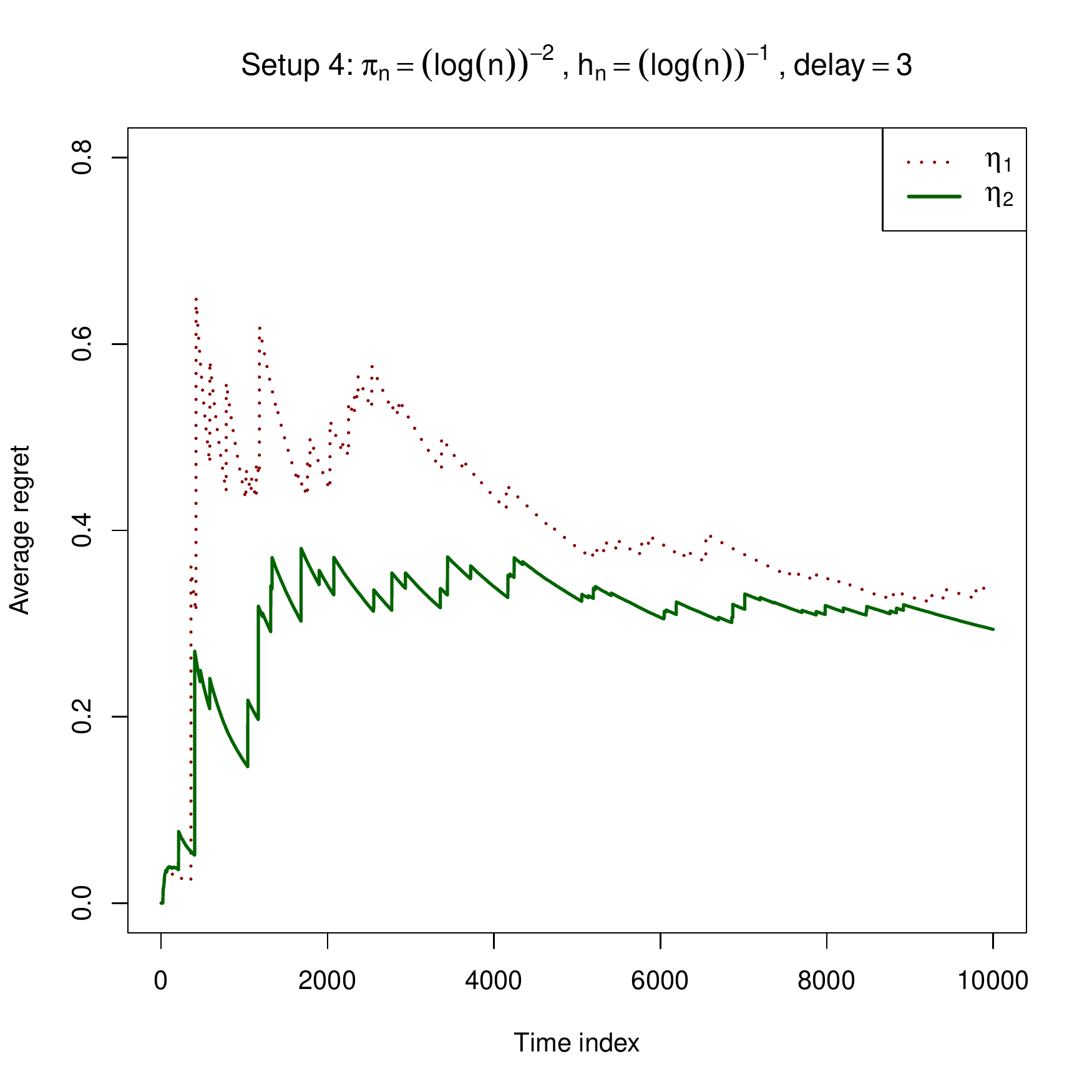}
   \includegraphics[scale=0.3]{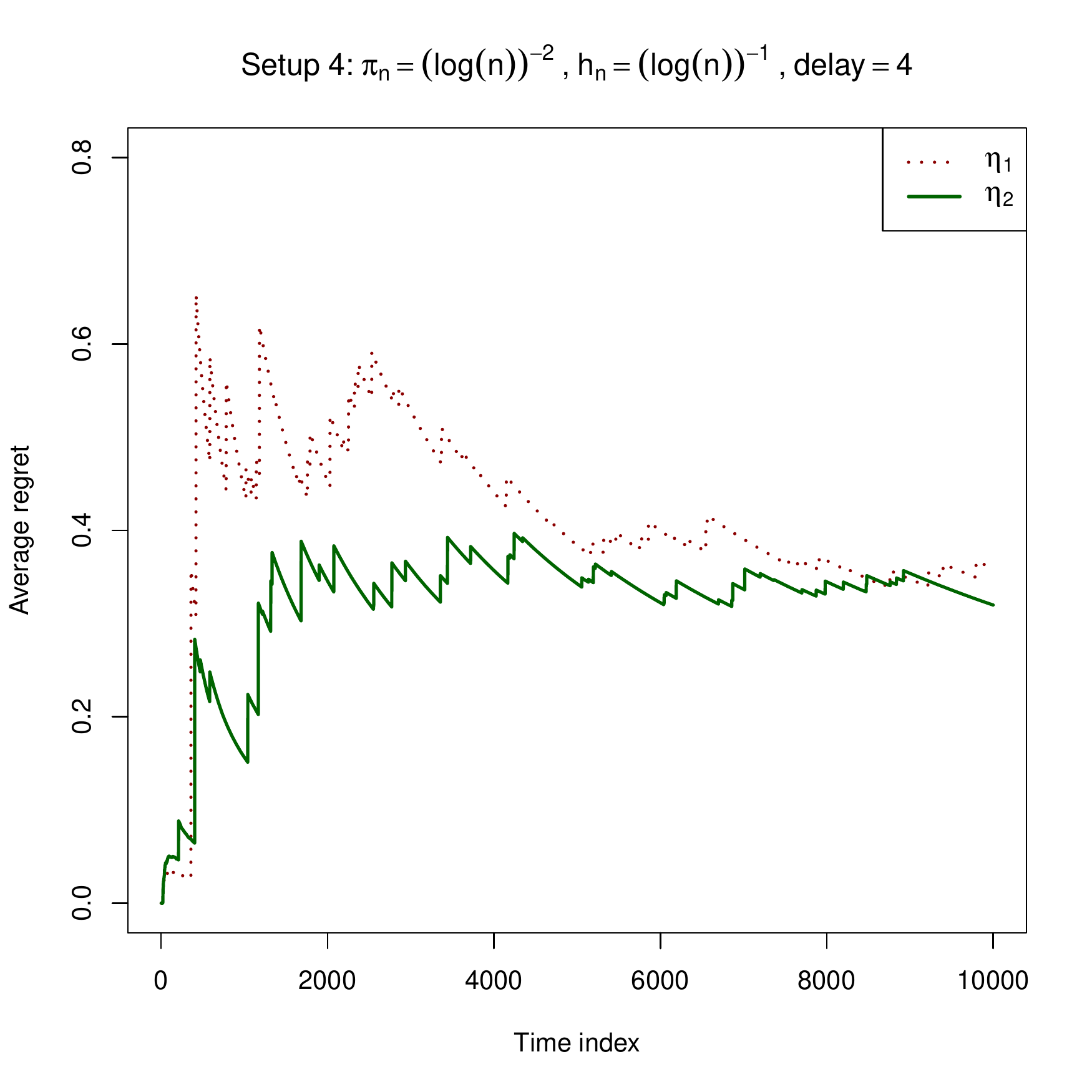}
  \caption{Each row represents a Setup, with first column depicting a one-dimensional function used to generate the mean reward functions. The second and the third column depict the average regret over time for Delay 3 and Delay 4 respectively.}
 \label{fig: Simulation_result_h5_pi6}
 \end{figure}

 \begin{figure}[h!]
 \centering
   \includegraphics[scale=0.3]{PDF_Figures_Supporting/Supp_S1S2S3_Arya_11.pdf}
   \includegraphics[scale=0.3]{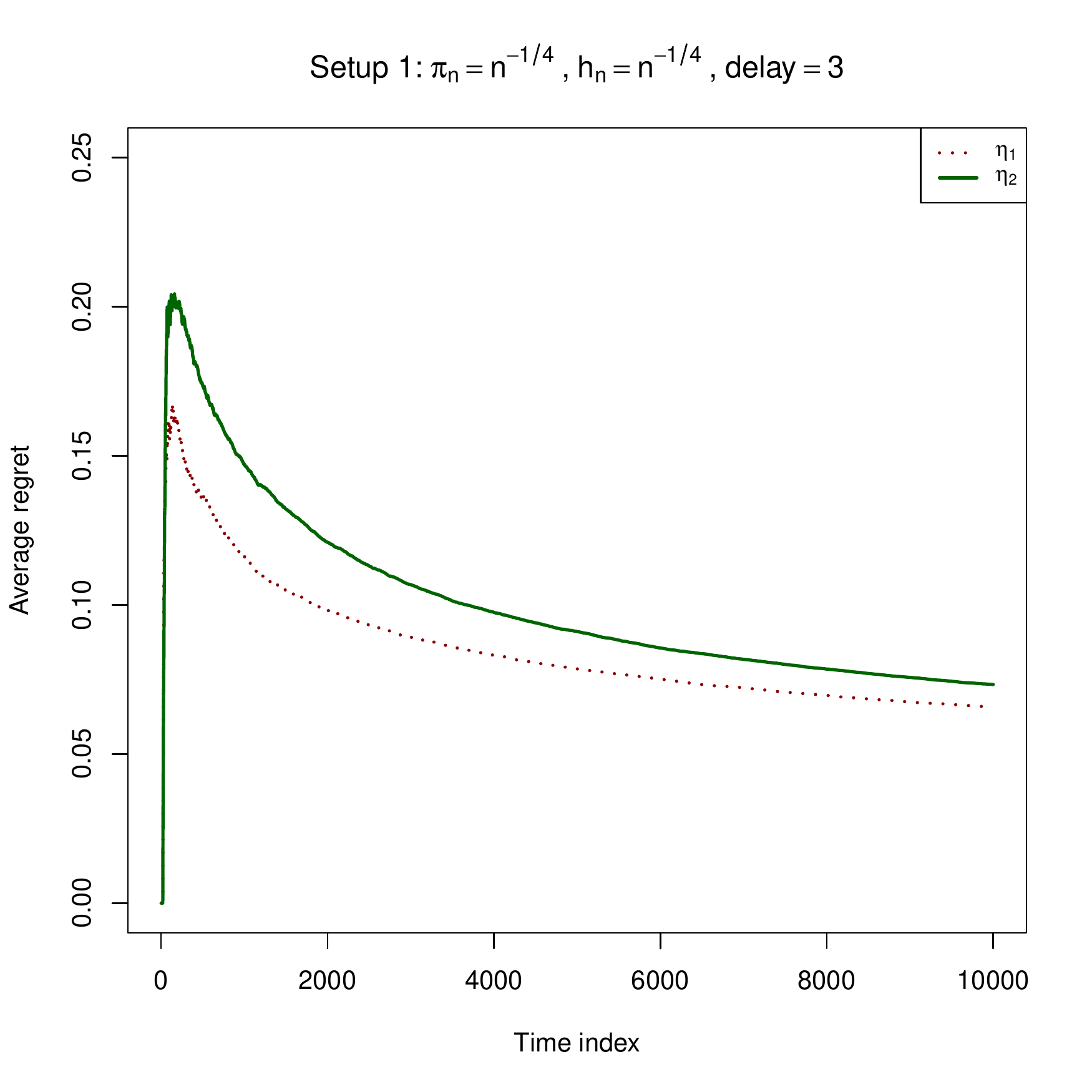}
   \includegraphics[scale=0.3]{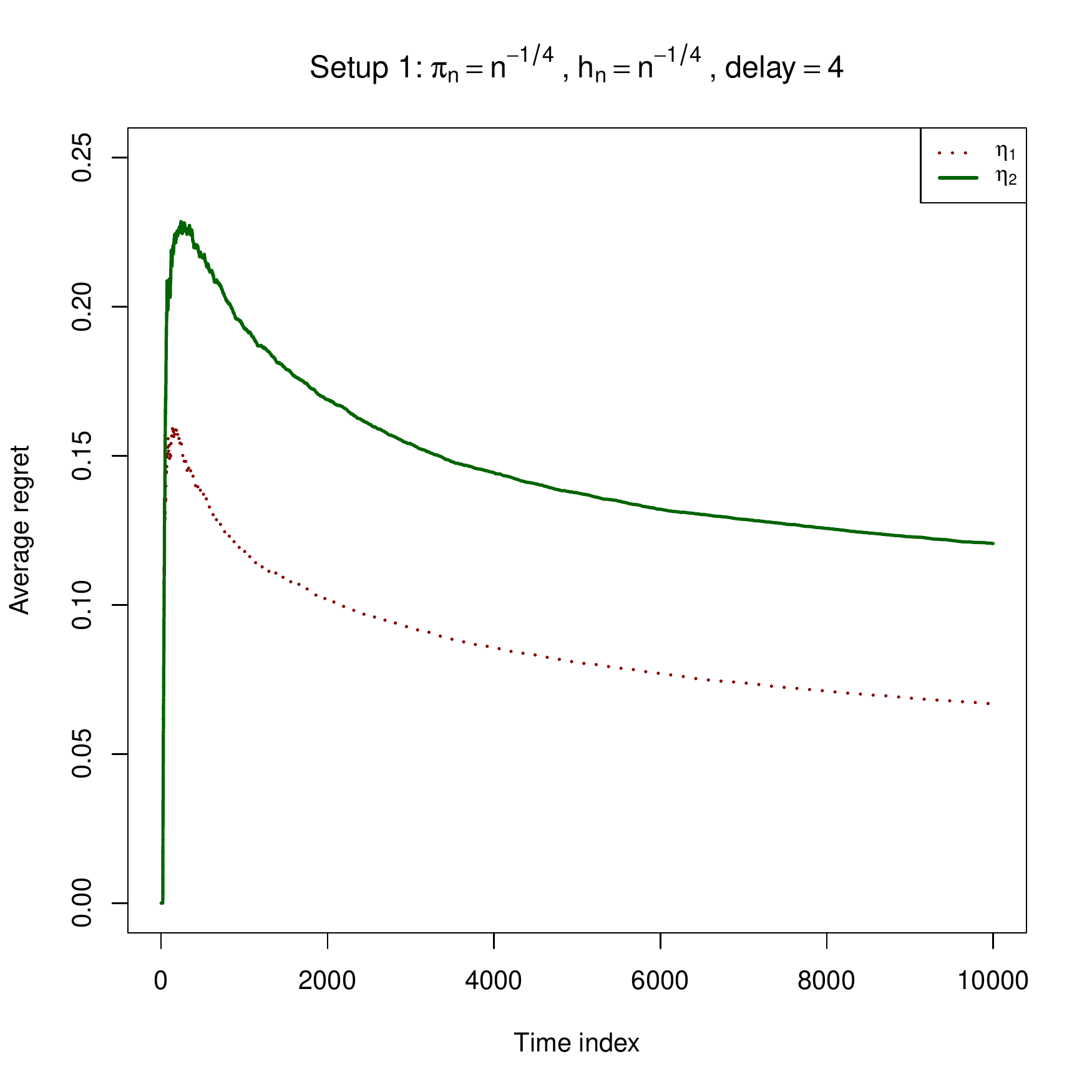}\\
   \includegraphics[scale=0.3]{PDF_Figures_Supporting/Supp_S1S2S3_Arya_21.pdf}
   \includegraphics[scale=0.3]{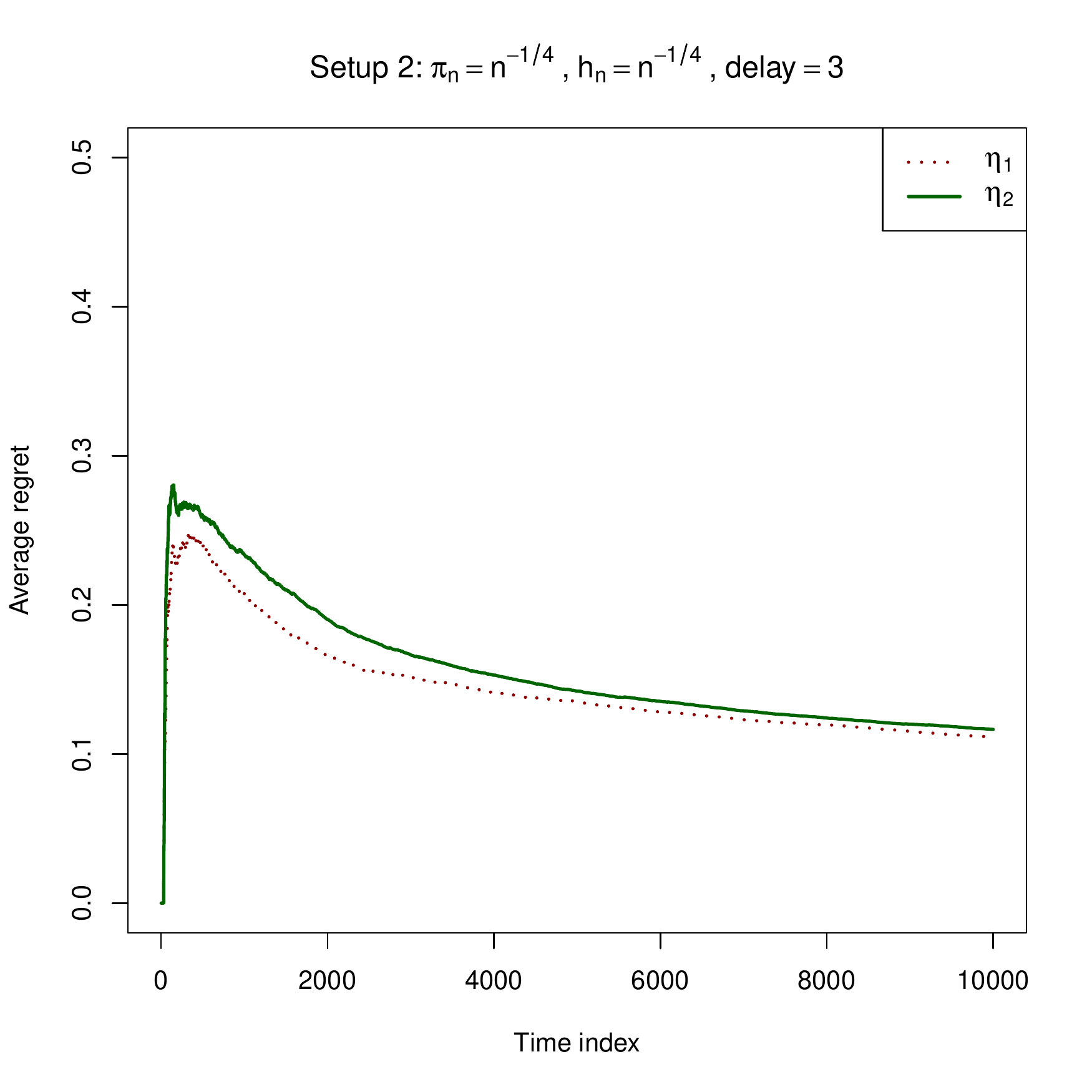}
   \includegraphics[scale=0.3]{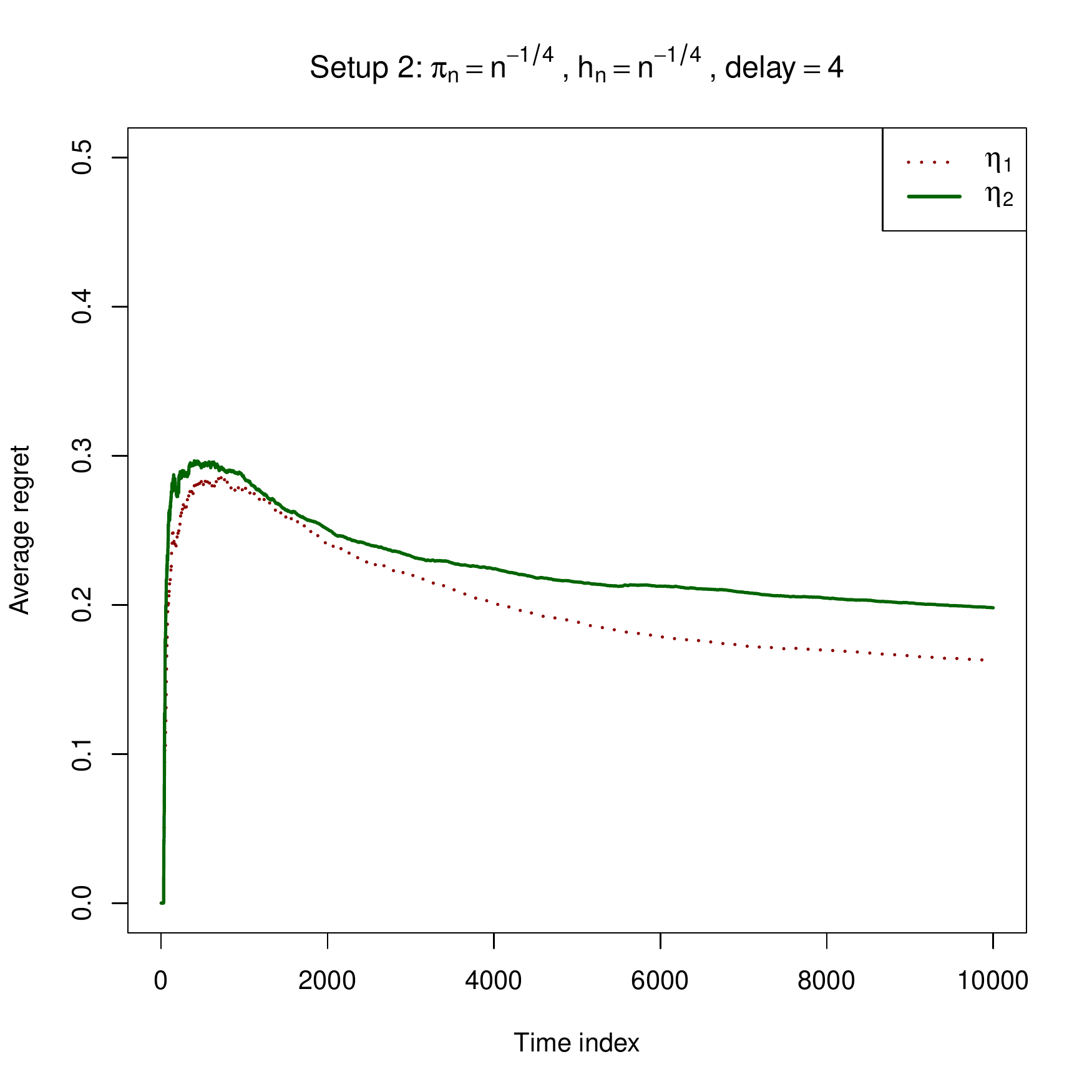}\\
\includegraphics[scale=0.3]{PDF_Figures_Supporting/Supp_S1S2S3_Arya_31.pdf}
   \includegraphics[scale=0.3]{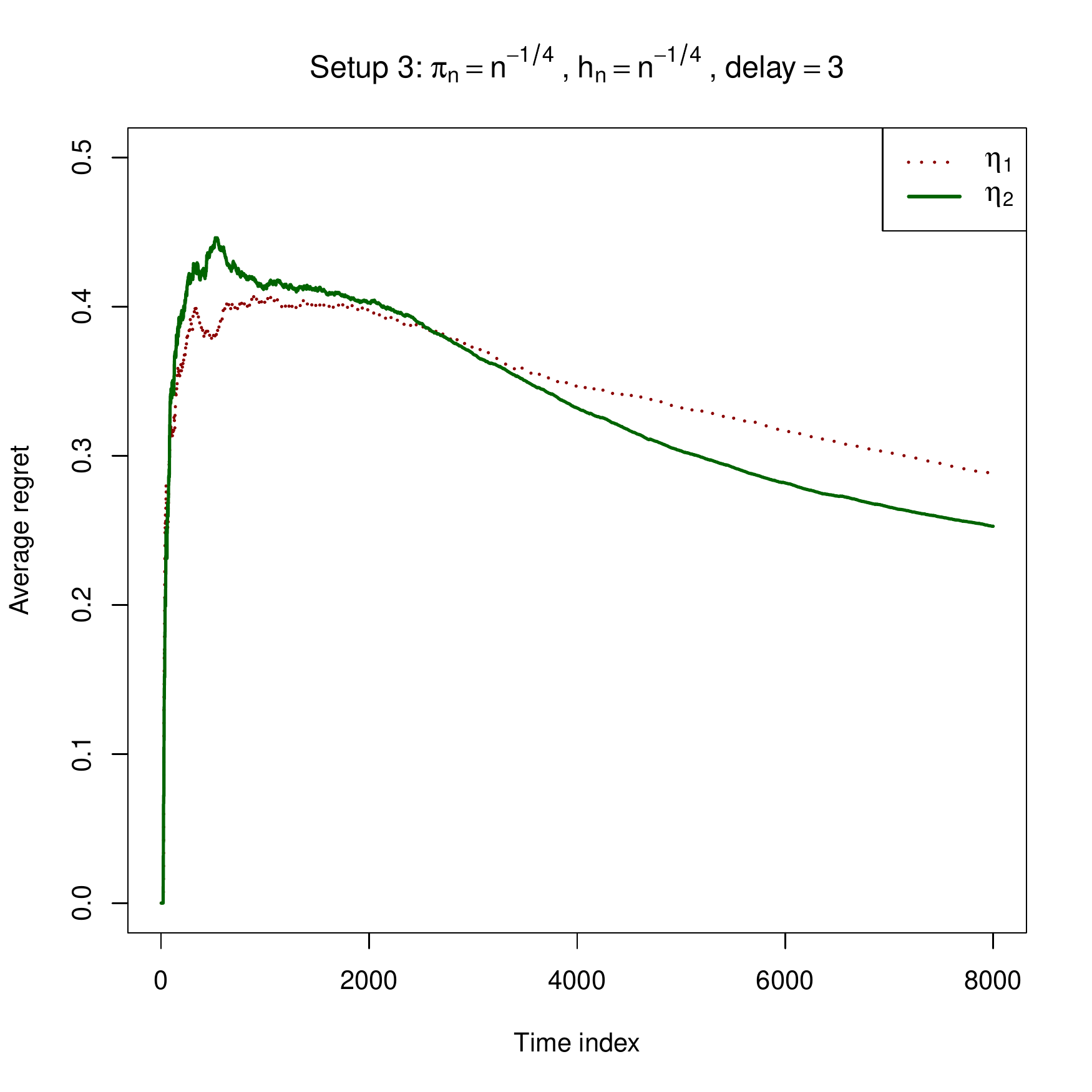}
    \includegraphics[scale=0.3]{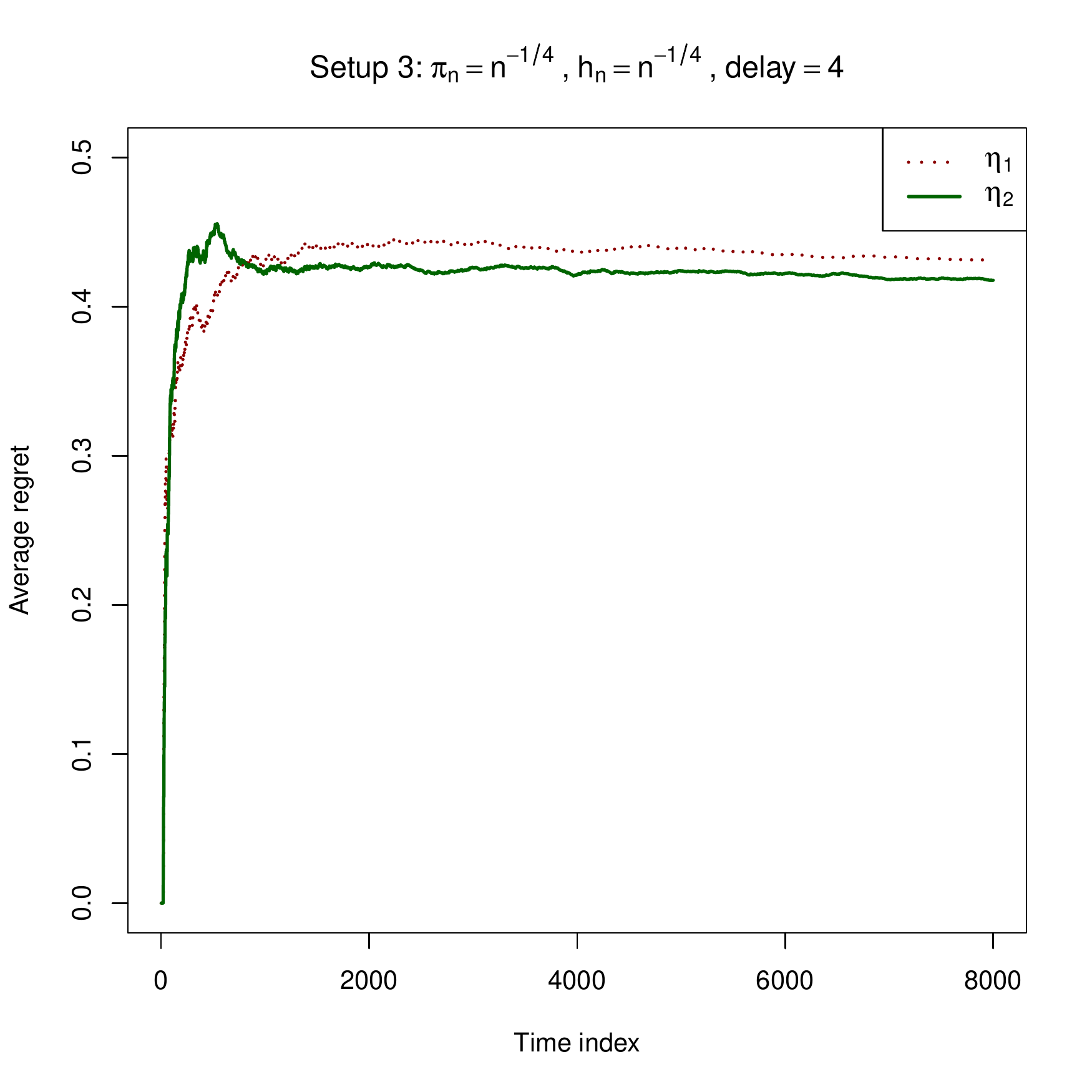}\\
   \includegraphics[scale=0.3]{PDF_Figures_Supporting/Supp_S1S2S3_Arya_41.pdf}
   \includegraphics[scale=0.3]{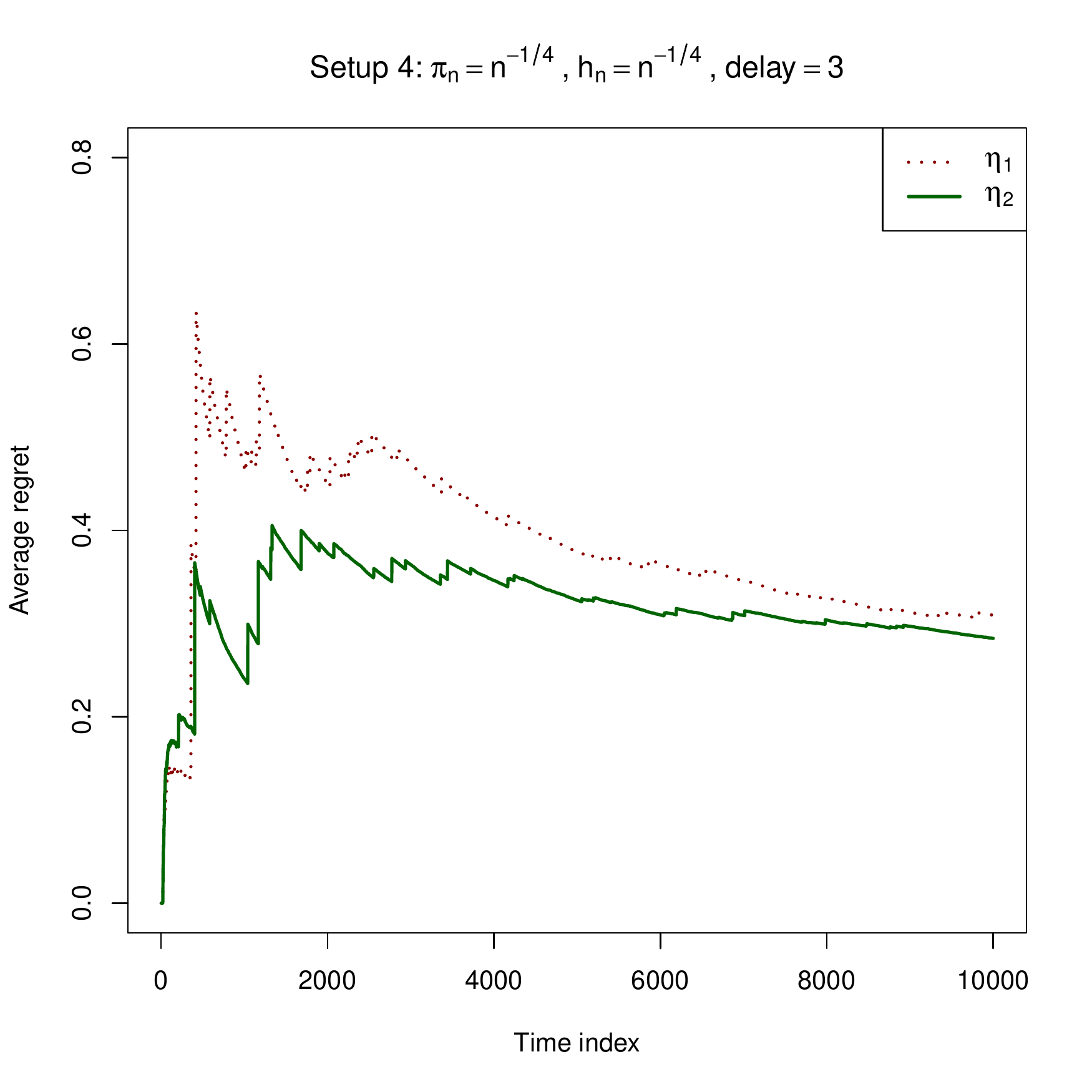}
   \includegraphics[scale=0.3]{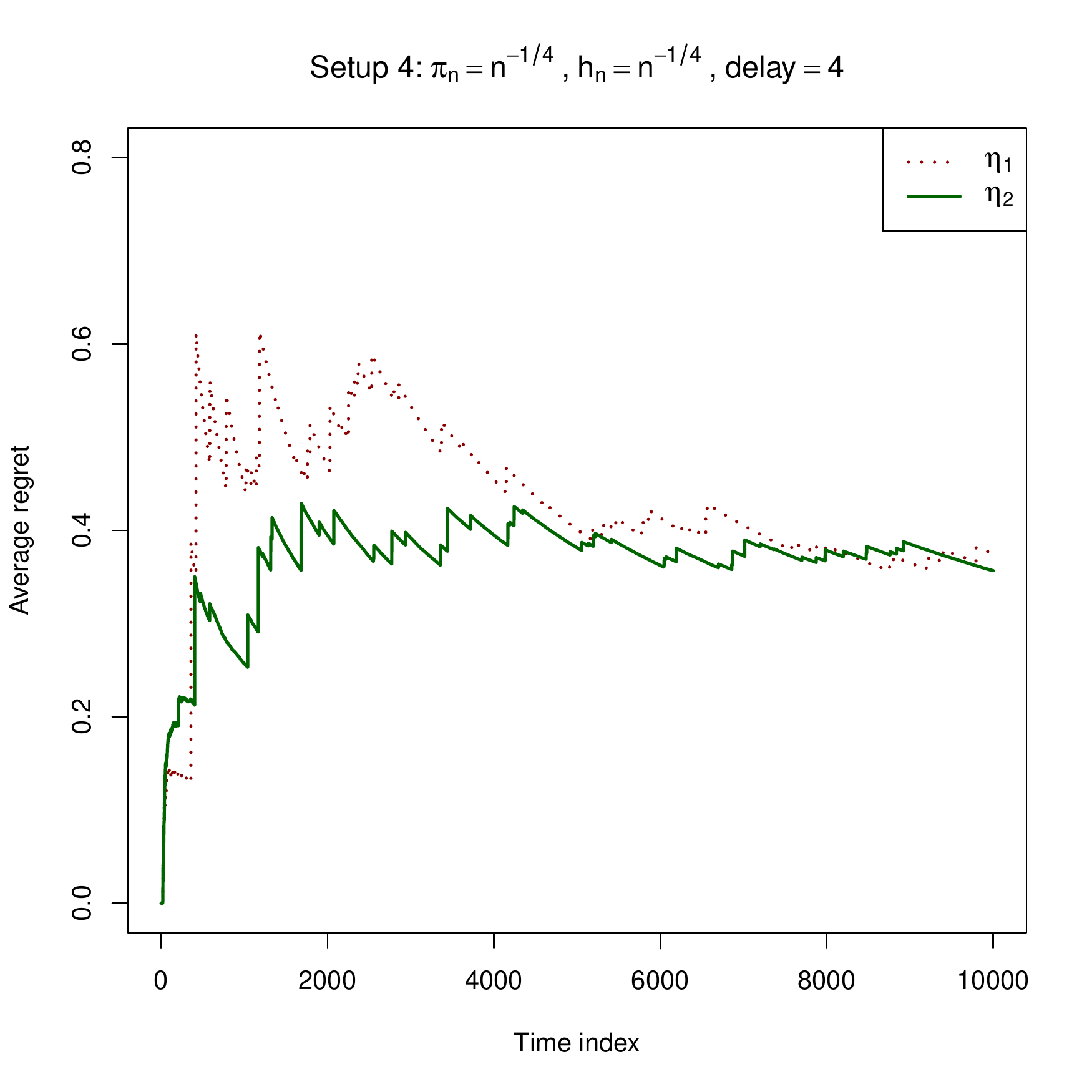}
  \caption{Strategy $\eta_1$ has lower cumulative average regret in Setup 1 and 2 (first two rows) and strategy $\eta_2$ has lower cumulative average regret in Setup 3 and 4 (rows third and fourth).}
 \label{fig: Simulation_result_h2_pi2}
 \end{figure}

 \begin{figure}[h!]
 \centering
   \includegraphics[scale=0.3]{PDF_Figures_Supporting/Supp_S1S2S3_Arya_11.pdf}
   \includegraphics[scale=0.3]{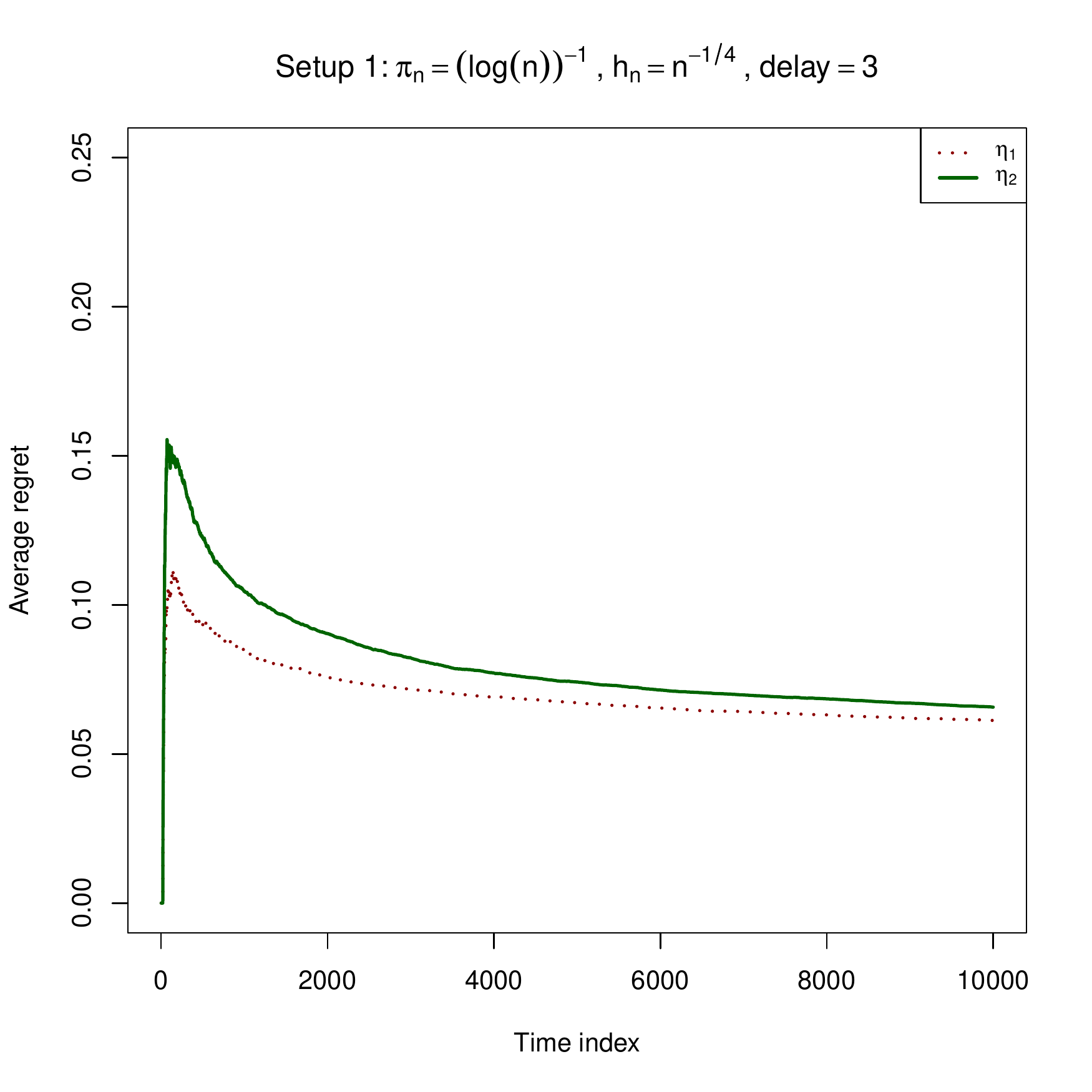}
   \includegraphics[scale=0.3]{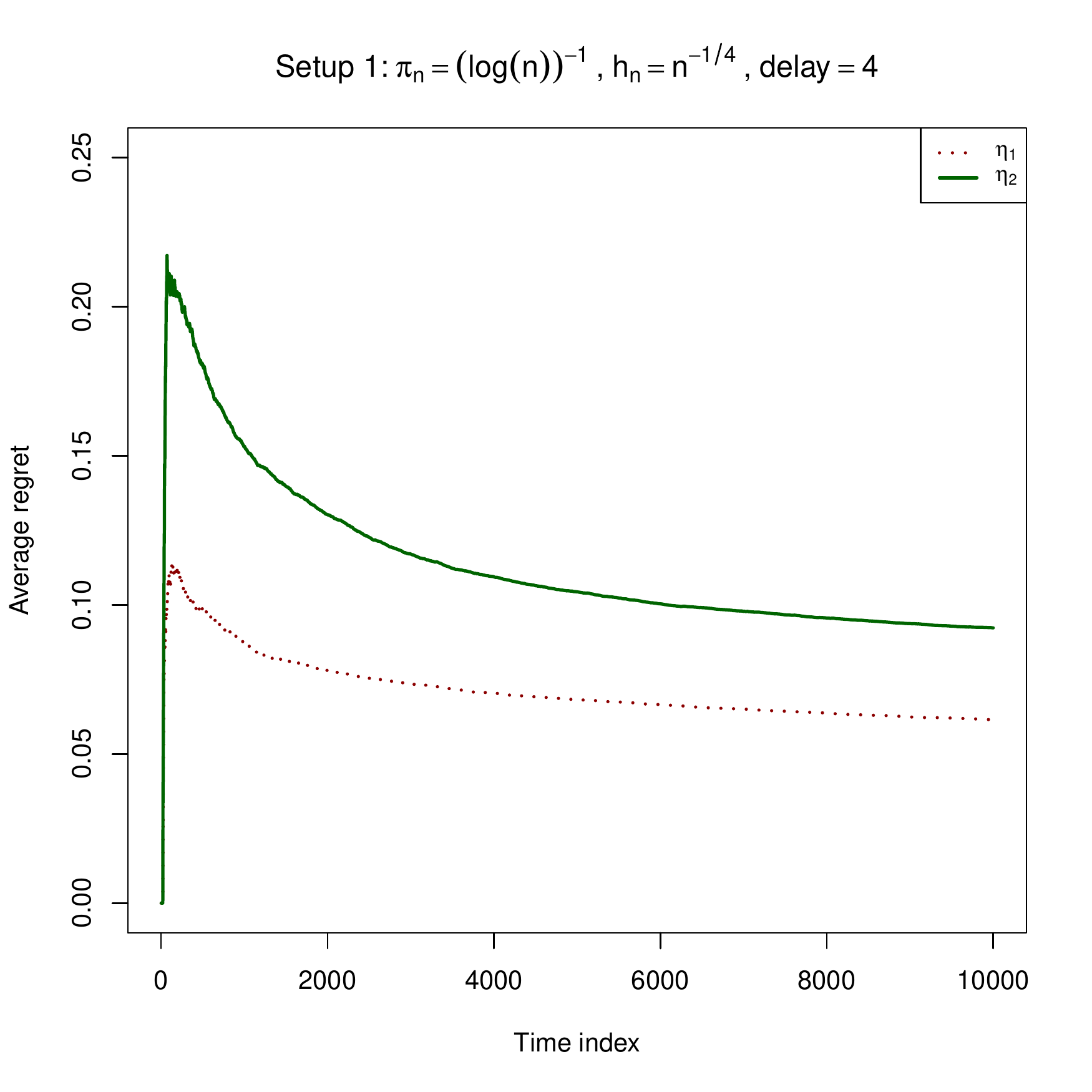}\\
   \includegraphics[scale=0.3]{PDF_Figures_Supporting/Supp_S1S2S3_Arya_21.pdf}
   \includegraphics[scale=0.3]{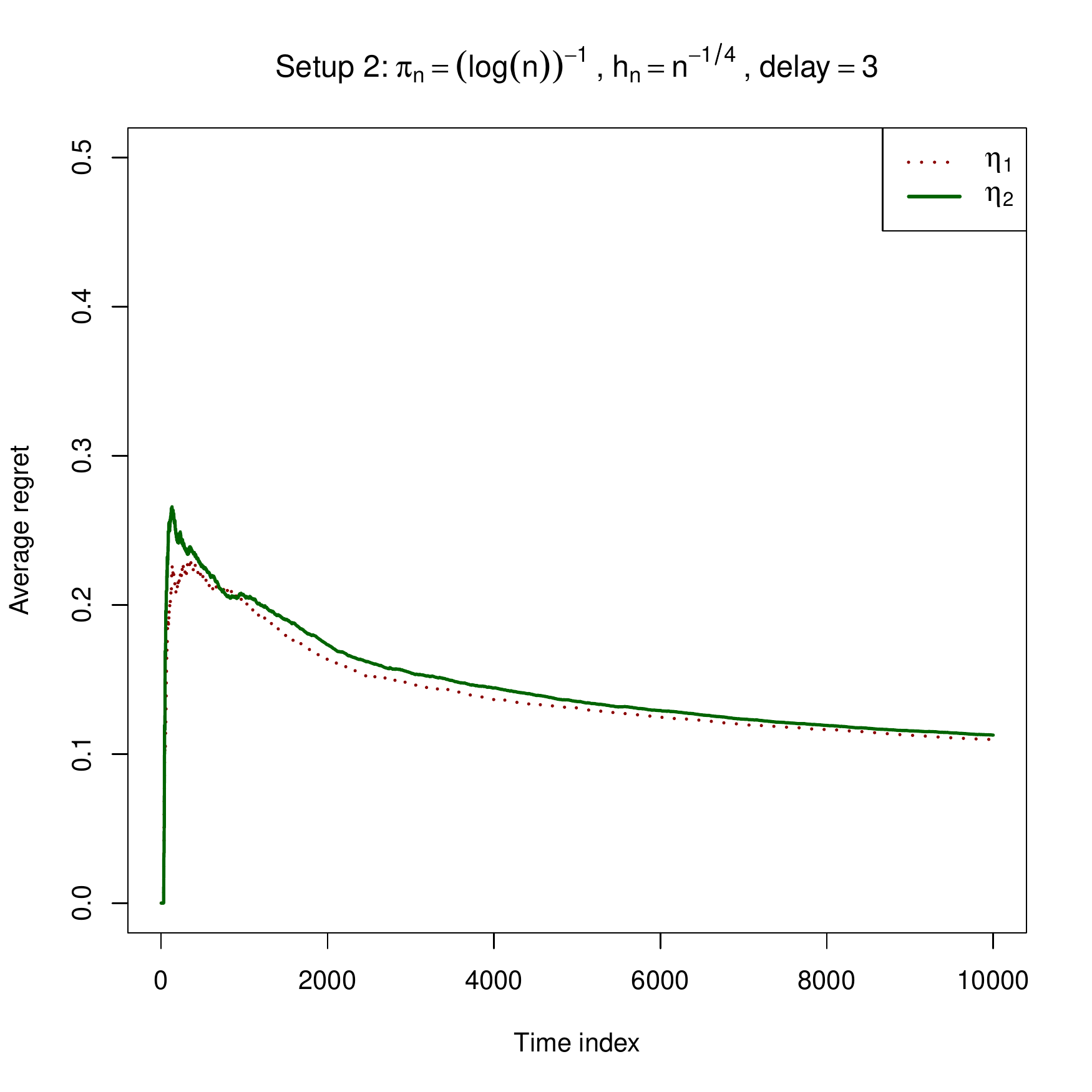}
   \includegraphics[scale=0.3]{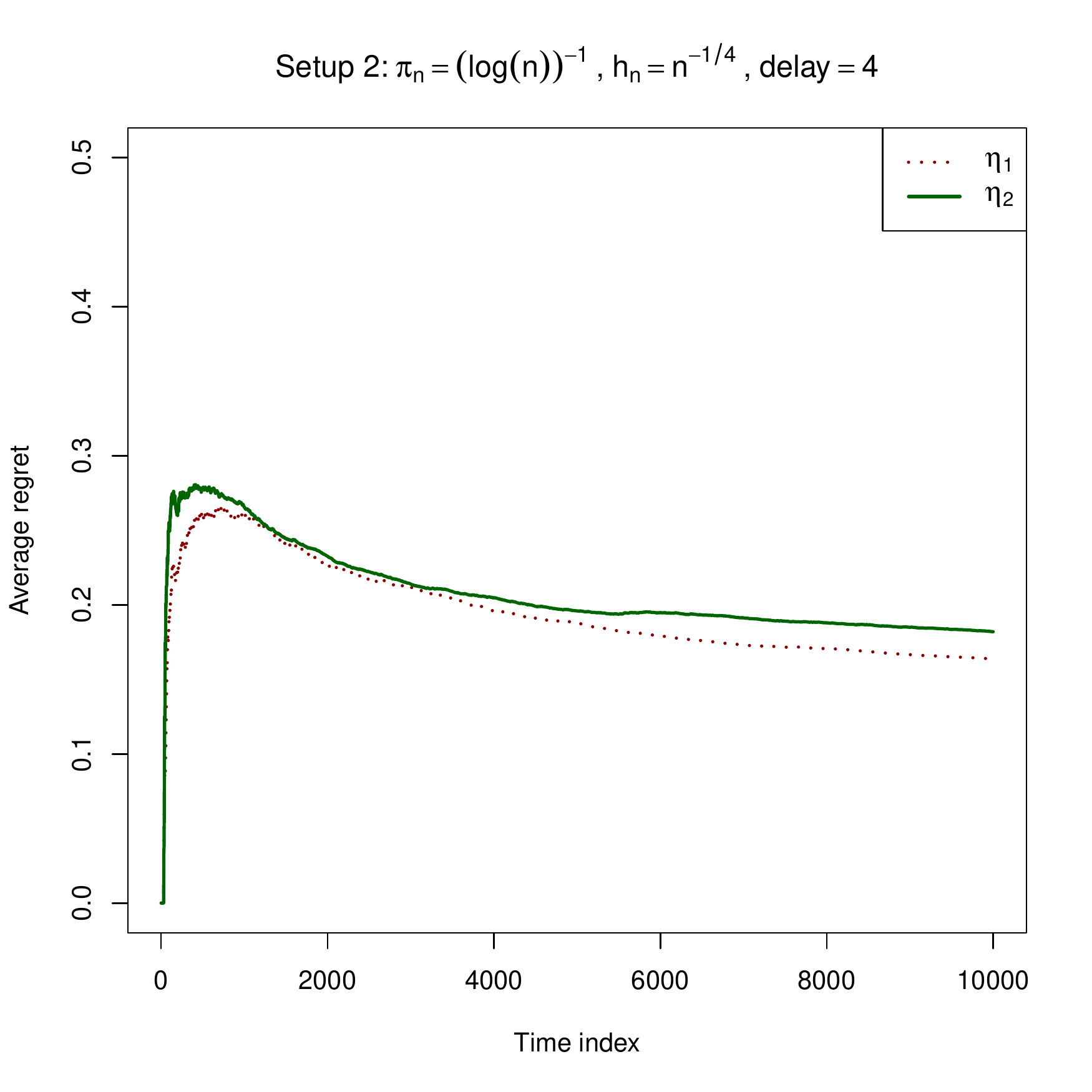}\\
\includegraphics[scale=0.3]{PDF_Figures_Supporting/Supp_S1S2S3_Arya_31.pdf}
   \includegraphics[scale=0.3]{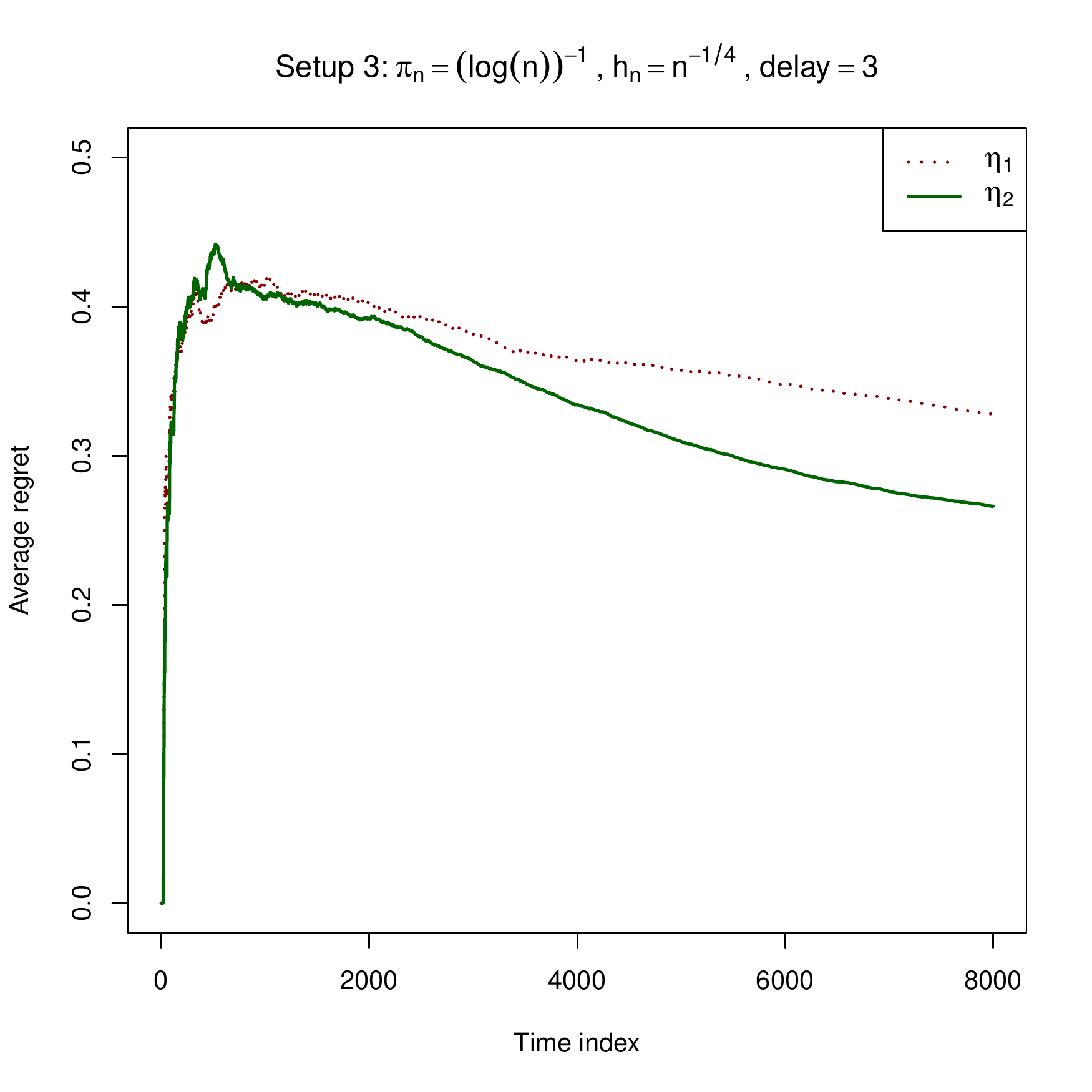}
    \includegraphics[scale=0.3]{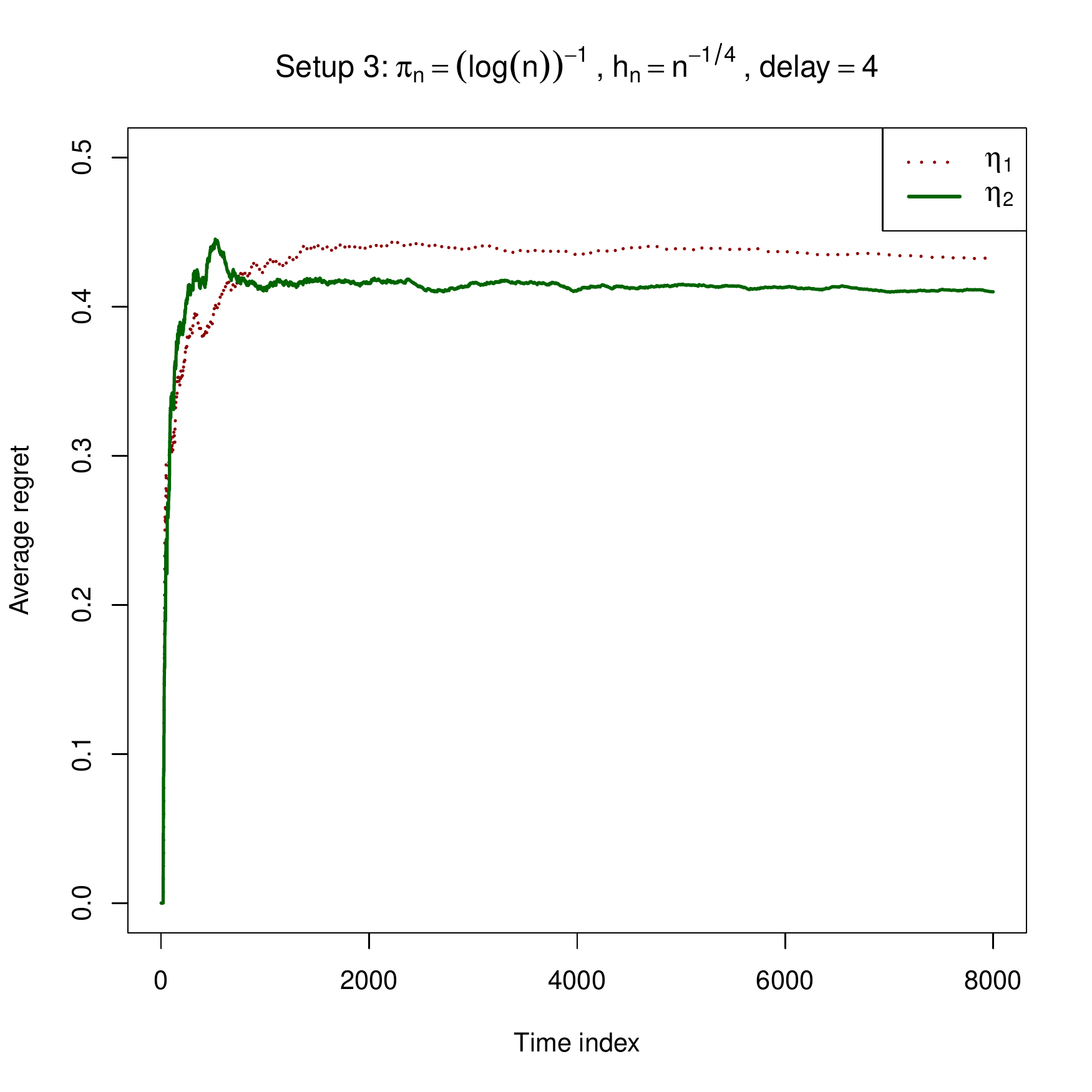}\\
   \includegraphics[scale=0.3]{PDF_Figures_Supporting/Supp_S1S2S3_Arya_41.pdf}
   \includegraphics[scale=0.3]{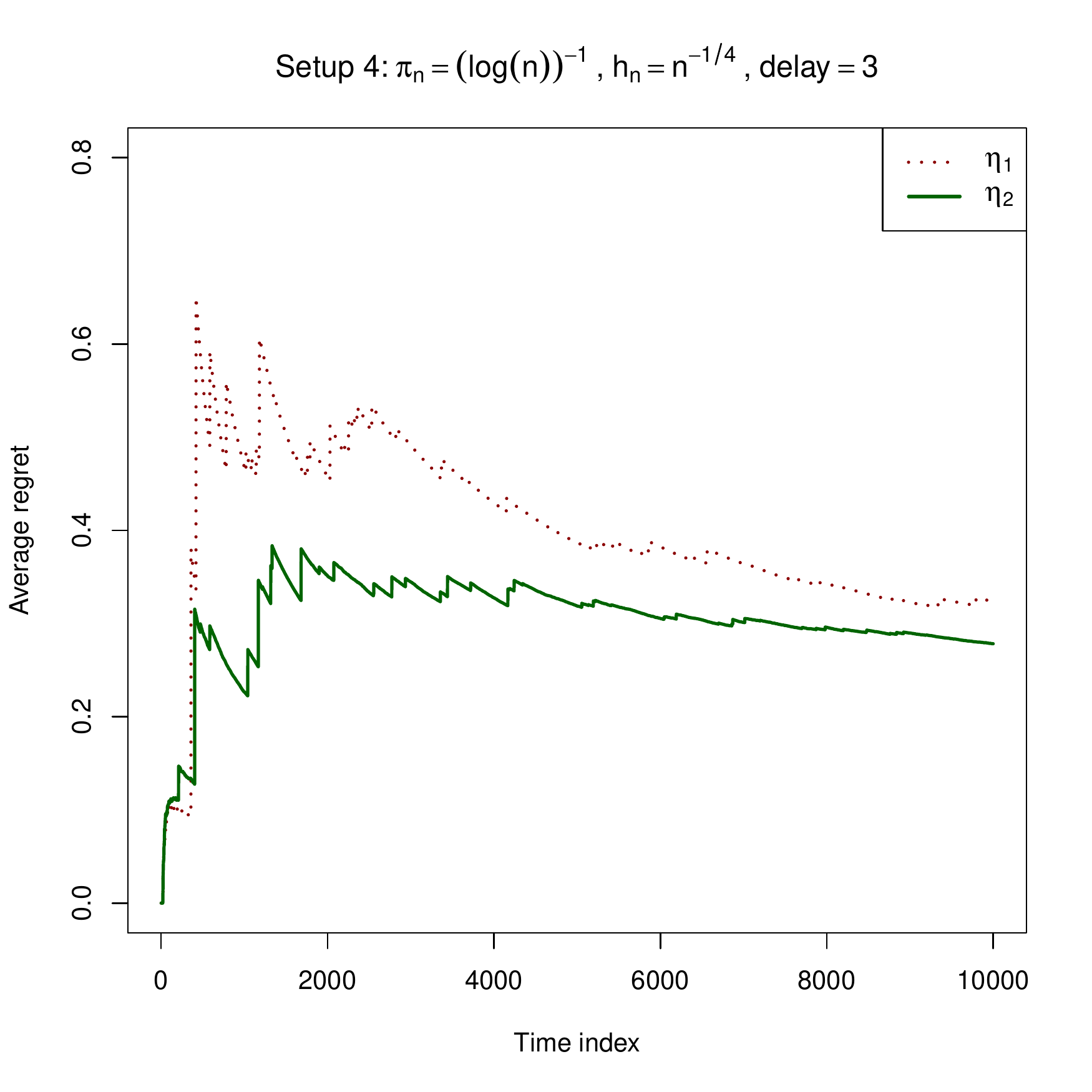}
   \includegraphics[scale=0.3]{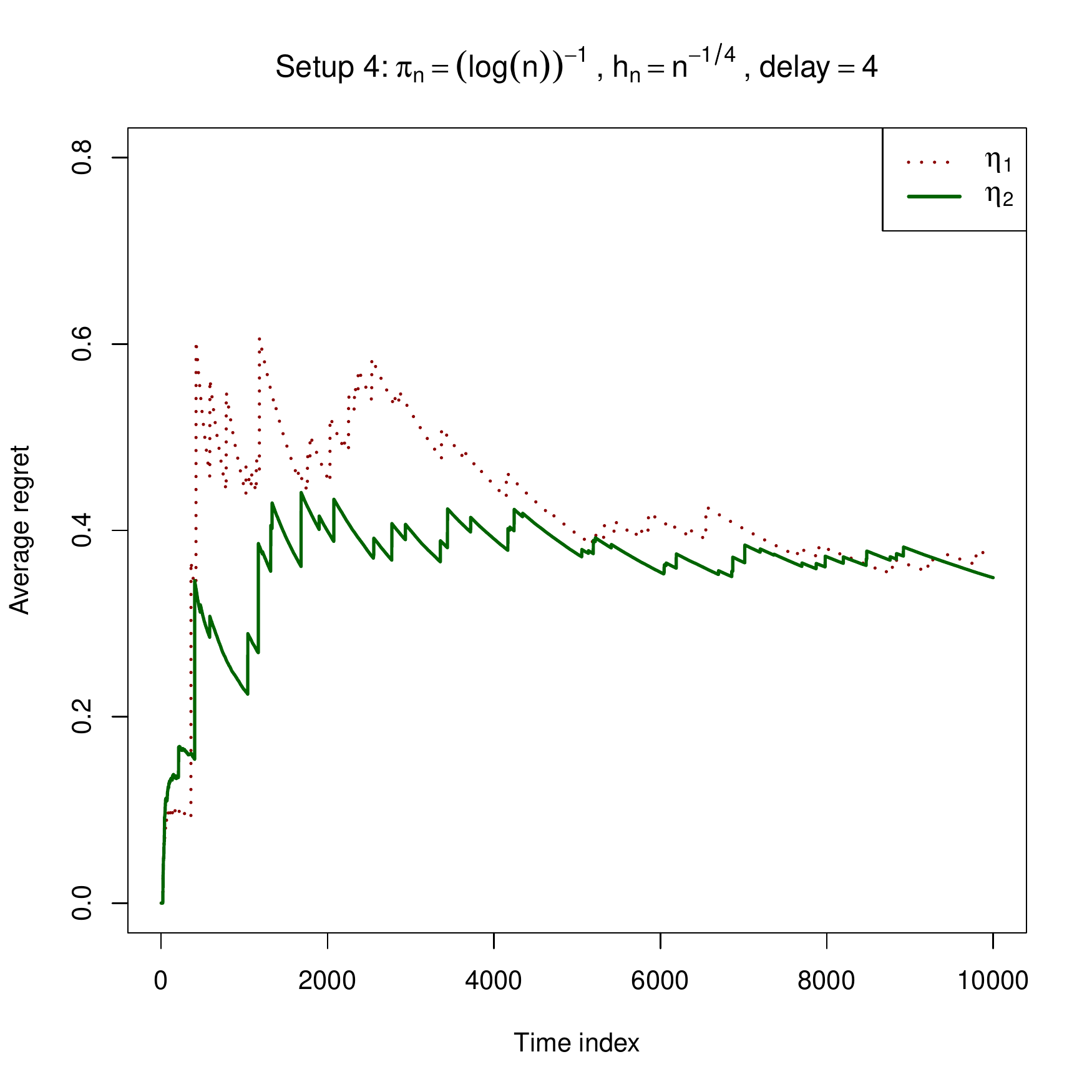}
  \caption{Strategy $\eta_1$ has lower cumulative average regret in Setup 1 and 2 (first two rows) and strategy $\eta_2$ has lower cumulative average regret in Setup 3 and 4 (rows third and fourth).}
 \label{fig: Simulation_result_h2_pi5}
 \end{figure}

\subsection{Simulation plots} \label{Simulation_extra_Appendix}
In this section, we plot the average regret curves for both strategies $\eta_1$ and $\eta_2$ for different hyper-parameter choices. In Figure \ref{fig: Simulation_result_h5_pi6}, we choose $\{h_n\} = (\log{n})^{-1}$ and $\{\pi_n\} = (\log{n})^{-2}$. We still notice the same trend, where $\eta_1$ performs better than strategy $\eta_2$ in Setup 1 and Setup 2, while $\eta_2$ performs better in Setup 3 and Setup 4. Notice that, for Setup 1 and 2, in the case of delay scenario 3, the difference in the average regret is not as noticeable as it is in delay 4. This could be attributed to the fast decaying $\{\pi_n\} = (\log{n})^{-2}$, where whether you update at every time point or only at observed reward time points, there is sharp increase in the amount of exploitation with the amount of data available in Delay 3 scenario unlike the Delay 4 scenario. We also notice that, in Setup 3, with Delay 4, the average regret does not seem to decay by our time horizon and might need a larger horizon to show some decay, which could be because the exploration probability is too fast decaying for both the algorithms to learn efficiently. Figure \ref{fig: Simulation_result_h2_pi2} and Figure \ref{fig: Simulation_result_h2_pi5} correspond to the choices $\{h_n, \pi_n\} = (n^{-1/4}, n^{-1/4}), (n^{-1/4}, (\log{n})^{-1})$ respectively. We see very similar trends as discussed in the paper and for Figure \ref{fig: Simulation_result_h5_pi6} for these two choices as well.

\end{document}